\documentclass[11pt]{amsart}

\usepackage{graphicx}
\usepackage[margin=1in]{geometry}
\usepackage{amsmath}
\usepackage{amssymb}
\usepackage{amsthm}
\usepackage{amsfonts}
\usepackage[export]{adjustbox}
\usepackage{graphicx}
\usepackage[algoruled,lined,longend,linesnumbered]{algorithm2e}
\usepackage{enumerate}
\usepackage[scriptsize,hang,raggedright]{subfigure}
\usepackage[cmyk,usenames,dvipsnames]{xcolor}
\usepackage{mathtools}
\usepackage[titletoc]{appendix}

\usepackage[foot]{amsaddr}

\usepackage{hyperref}
\usepackage{pdfsync}
\usepackage{dsfont}
\usepackage{color}
\usepackage{enumerate}

\usepackage{cite}
\usepackage{bbm}
\usepackage{multirow}
\usepackage[font=footnotesize,labelfont=bf]{caption}
\usepackage{bibentry}
\nobibliography*

\numberwithin{equation}{section}

\newtheorem{theorem}{Theorem}[section]
\newtheorem{proposition}[theorem]{Proposition}
\newtheorem{lemma}[theorem]{Lemma}

\newtheorem{definition}[theorem]{Definition}

\theoremstyle{remark}
\newtheorem{remark}[theorem]{Remark}

\newcommand{\e}{\varepsilon}
\newcommand{\R}{\mathbb{R}}

\newcommand{\ba}{\begin{array}}
\newcommand{\ea}{\end{array}}

\newcommand{\bthm}{\begin{theorem}}
\newcommand{\ethm}{\end{theorem}}
\newcommand{\bprop}{\begin{proposition}}
\newcommand{\eprop}{\end{proposition}}
\newcommand{\blemma}{\begin{lemma}}
\newcommand{\elemma}{\end{lemma}}

\newcommand{\beqn}{\begin{equation}}
\newcommand{\eeqn}{\end{equation}}
\newcommand{\beqns}{\begin{equation*}}
\newcommand{\eeqns}{\end{equation*}}

\newcommand{\supp}{\operatorname{supp}}

\newcommand{\infinity}{\infty}

\renewcommand{\leq}{\leqslant}
\renewcommand{\geq}{\geqslant}

\definecolor{mygreen}{rgb}{0.1,0.75,0.2}

\newcommand{\N}{\mathbb{N}}

\DeclareMathOperator{\Vor}{Vor}

\DeclareMathOperator{\id}{Id}

\graphicspath{{./figs/}}

\newcommand{\Rd}{{\mathord{\mathbb R}^d}}

\newcommand{\conv}{{\rm conv}}
\newcommand{\aff}{{\rm aff}}
\def\P{{\mathcal P}}

\title{ Wasserstein  Archetypal Analysis} 
\thanks{K. Craig acknowledges partial support from NSF DMS 1811012, NSF DMS 2145900, and a Hellman Faculty Fellowship. K. Craig also acknowledges the support from the Simons Center for Theory of Computing, at which part of this work was completed. B. Osting acknowledges partial support from NSF DMS 17-52202. D. Wang acknowledges partial support from NSFC 12101524 and the University Development Fund from The Chinese University of Hong Kong, Shenzhen (UDF01001803). }
\author{Katy Craig}
\address{Department of Mathematics, University of California, Santa Barbara}
\email{kcraig@math.ucsb.edu}
\author{Braxton Osting}
\address{Department of Mathematics, University of Utah, Salt Lake City}
\email{osting@math.utah.edu}
\author{Dong Wang}
\address{School of Science and Engineering \& Guangdong Provincial Key Laboratory of Big Data Computing, The Chinese University of Hong Kong, Shenzhen, Guangdong, China}
\email{wangdong@cuhk.edu.cn}
\author{Yiming Xu}
\address{Corporate Model Risk, Wells Fargo}
\email{yiming.xu@wellsfargo.com}
\date{\today}                                        

\begin{document}

\begin{abstract}
Archetypal analysis is an unsupervised machine learning method that summarizes data using a convex polytope. 
In its original formulation, for fixed $k$, the method finds a convex polytope with $k$ vertices, called archetype points, such that the polytope is contained in the convex hull of the data and the mean squared Euclidean distance between the data and the polytope is minimal. 

In the present work, we consider an alternative formulation of archetypal analysis based on the Wasserstein metric, which we call Wasserstein archetypal analysis (WAA). In one dimension, there exists a unique solution of WAA \cite{cuesta2002shape} and, in two dimensions, we prove existence of a solution, as long as the data distribution is absolutely continuous with respect to Lebesgue measure. We discuss obstacles to extending our result to higher dimensions and general data distributions. We then introduce an appropriate regularization of the problem, via a R\'enyi entropy, which allows us to obtain existence of solutions of the regularized problem for general data distributions, in arbitrary dimensions. We prove a consistency result for the regularized problem, ensuring that if the data are iid samples from a probability measure, then as the number of samples is increased,  a subsequence of the archetype points converges to the archetype points for the limiting data distribution, almost surely. Finally,
we develop and implement a gradient-based computational approach for the two-dimensional problem, based on the semi-discrete formulation of the Wasserstein metric.
Our analysis is supported by detailed computational experiments. 
\end{abstract}
\maketitle



\section{Introduction}
Given a probability measure $\mu \in \P(\Rd)$, \emph{archetypal analysis} (AA) aims to find  the convex polytope $\Omega\subseteq \mathbb R^d$  with $k$ vertices  that best approximates $\mu$.  As originally introduced by Culter and Breiman in 1994 \cite{Cutler_1994}, 
  given data  
$X = \{ x_i\}_{i=1}^{N} \subseteq \R^d$ and   $k \geq d+1$, 
AA   finds   $k$ vertices,  
$  A= \{a_\ell\}_{\ell=1}^{k} \subseteq \R^d$, that belong  to the convex hull of the data, for which the convex hull $\textrm{co}(A)$  explains the most variation  of the dataset. In particular, AA can be framed in terms of the following constrained optimization problem
\begin{align} \label{e:arch}
\min_{A = \{a_\ell\}_{\ell=1}^{k} \subseteq \R^d } \ &  \left\{ \frac{1}{N} \sum_{i = 1}^N d^2(x_i, \textrm{co}(A) )\colon A \subset \textrm{co}(X) \right\} ,
\end{align}
where $d^2(\cdot,\cdot)$ denotes the squared Euclidean distance. 
The archetypes, $A= \{a_\ell\}_{\ell=1}^{k} \subseteq \R^d$, may be interpreted as exemplars of extreme points of the dataset, a mixture of which   explain the general characteristics of the associated distribution; see \cite{Chan_2003, shoval2012evolutionary, M_rup_2012} for applications of AA in astronomy, biology, and many others.

AA is closely related to other unsupervised learning methods, such as   $k$-means, principal component analysis, and nonnegative matrix factorization   \cite{M_rup_2012}. Under bounded support assumptions, the consistency and convergence of AA were recently established in \cite{Yiming2020}, laying the foundation for AA to apply to large-scale datasets \cite{Xu2021}.  
In practice, however, it is often more appropriate to assume that the distribution generating the data has finite moments but unbounded support, in which case AA is not well-posed \cite{Yiming2020}. 
Also, due to the definition of the loss, AA is sensitive to outliers \cite{Eugster_2011}. 
To address both issues, the present work considers a different formulation of the AA problem, based on the Wasserstein metric.

\subsection{Main results}
Let $\P_2(\Rd)$ denote the space of Borel probability measures on $\Rd$ with finite second moment, $M_2(\mu):= \int_\Rd |x|^2 d\mu(x) < +\infty$. Given $\mu \in \P_2(\Rd)$ and a number of vertices $k \geq d+1$, we seek to find the (nondegenerate) convex $k$-gon $\Omega \subseteq \Rd$ that is \emph{closest} to $\mu$, in the sense that the  uniform probability measure on $\Omega$ is  as close as possible  to $\mu$ in the 2-Wasserstein metric:
\begin{align} \label{mainproblem} 
\tag{WAA}   \min_{ \Omega \in S_k} W_2( \mu, 1_\Omega/|\Omega|) , \quad 
 S_k = \{ \Omega \subset \R^d \colon \Omega \text{ is a convex $k$-gon with nonempty interior}  \}  ,
\end{align}
where 
\[      1_\Omega(x) = \begin{cases} 1 &\text{ if } x \in \Omega, \\ 0 &\text{ otherwise.} \end{cases}   \]
Here, we use the term \emph{$k$-gon} to mean a bounded polytope with $k$ vertices. 

Note that we make a mild abuse of notation in the above problem formulation and throughout this manuscript:  if a measure $\nu \in \P_{2}(\Rd)$ is absolutely continuous with respect to $d$-dimensional Lebesgue measure $\mathcal{L}^d$,   we will write $\nu \in \P_{2, ac}(\Rd)$ and use the same symbol to denote both the measure and its density, $d \nu(x) = \nu(x) d\mathcal{L}^d(x)$. In this way, we will use $1_\Omega/|\Omega|$ to denote the uniform probability measure on $\Omega$.

Informally, the 2-Wasserstein metric measures the distance between probability measures in terms of the amount of \emph{effort} it takes to rearrange one to look like the other. More precisely, 
given  $\mu, \nu \in \P_2(\Rd)$,
the $2$-Wasserstein metric is defined by
\begin{align} 
W^2_2(\mu,\nu) = \inf_{X\sim\mu, Y\sim\nu}\mathbb E[|X-Y|^2],\label{l1}
\end{align}
where the expectation is taken with respect to a rich enough underlying probability space and  the infimum is taken over all couplings $(X, Y)$ with marginals $\mu$ and $\nu$. 
If $\mu \in \P_{2, ac}(\R^d)$, then  there is a unique optimal coupling and, furthermore,   the coupling is \emph{deterministic}: there exists a measurable function $T$, which is unique $\mu$-a.e., so that 
\begin{align} \label{existence of OT map} Y = T (X) .
\end{align} 
(See work by Gigli for sharp conditions guaranteeing the existence of deterministic couplings \cite{Gigli}.)
For further background on optimal transport and the 2-Wasserstein metric, we refer the reader to one of the many excellent textbooks on the subject \cite{AGS, Villani, VillaniBig, santambrogio2015optimal, figalli2021invitation,ambrosio2021lectures}.
 
 
Alternative generalizations of the classical AA problem, based on probabilistic interpretations, can be found in \cite{seth2016probabilistic}, where data are assumed to be generated from a parametric model and the corresponding archetypes are found using the maximum likelihood estimation, resulting in a similar formulation as \eqref{e:arch}. 
Nevertheless, the statistical assumptions in such approaches are often hard to verify and may only be appropriate for certain datasets. 
For nonparametric approaches, aside from the Euclidean and 2-Wasserstein metrics, other discrepancy measures such as the Kullback--Leibler (KL) divergence may also be used.
However, a KL divergence would treat outliers of a given mass in the same way regardless of their spatial closeness, whereas a Wasserstein approach places a larger focus on outliers that are farther from the bulk of the dataset. 
More generally, one could also consider $p$-Wasserstein metrics, where larger choices of $p$ would further penalize the distance from the dataset. In this way, the choice of metric or statistical divergence is problem dependent and encodes modeling assumptions about the structure of the dataset. 

 
 When $d=1$, there is a closed form solution for the 2-Wasserstein metric, and it is straightforward to directly compute the unique minimizer of \eqref{mainproblem}, as previously done in related work by Cuesta-Albertos, Matr\'an, and Rodr\'iguez-Rodr\'iguez \cite{cuesta2002shape}. (See below for a discussion of this and other related previous works.) Our first main result is that, when $d=2$,  a solution of \eqref{mainproblem}   exists, provided that $\mu$ is absolutely continuous.  We summarize the results in the one and two dimensional setting in the following theorem.
  \begin{theorem}[Existence of minimizer in one and two dimensions] \label{t:Existence2d}
 Let $d=1$ and $k \geq 2$. Then for all $\mu \in \P_2(\R)$, there exists a unique solution of \eqref{mainproblem}, and this solution may be expressed explicitly as $\Omega = [a,b]$ for
\begin{align}\label{e:1dSol}
a &= 4 C_{0} - 6 C_{1} , \quad  b = 6 C_{1} - 2 C_{0}, \quad  \quad  C_{0} =   \int_{\R} x d\mu(x) , \quad  C_{1} =   \int_{\R} x F(x) d \mu(x),
\end{align}
where $F$ is the cumulative distribution function (CDF) of $\mu$.
 
Let $d= 2$ and $k \geq 3$. Then for all  $\mu \in \P_{2,ac}(\mathbb{R}^2)$,   there exists a solution of \eqref{mainproblem}.
\end{theorem}


In the two dimensional case, our proof of Theorem~\ref{t:Existence2d}   relies on an explicit characterization of the closure of $\{ 1_\Omega/|\Omega| \colon \Omega \in S_k\}$ in the narrow topology; see Proposition \ref{ourlemma}. The key difficulty in proving a minimizer exists arises  when considering the possibility that the minimizing sequence converges to a limit point with lower dimensional support. Note, in particular, that limit points with lower dimensional support are in general \emph{not} uniformly distributed on their support; see Figure \ref{limitpointexample}. For this reason, not all limit points have a valid interpretation in terms of \emph{archetypal analysis}, in which the $k$ vertices or \emph{archetypes} of their support completely describe the measure via their convex combinations. For this reason, it is essential that we rule out the possibility that a limit point of this type achieves the infimum of $W_2(\mu, 1_{\Omega}/|\Omega|)$ over $\Omega \in S_k$. We succeed in doing this when $\mu \in \P_{2,ac}(\R^2)$ by adapting a  perturbation argument of Cuesta-Albertos, et. al., \cite{cuesta2003approximation} to the case of convex polygons. In particular, we construct simple perturbations  around   degenerate limit points   that are both feasible for our constraint set and improve the value of the objective function.  It remains an open question how to extend  this technique to $\mu$ which are no absolutely continuous with respect to $\mathcal{L}^2$; see Remark \ref{muabscts}. Furthermore, our approach to proving the value of the objective function improves strongly leverages the structure of the 2-Wasserstein metric, and a different approach would be needed to extend the result to $p$-Wasserstein metrics for $p \neq 2$.

While this perturbative approach allows us to overcome the difficulty of degenerate limit points in two dimensions, our approach fails in dimensions $d \geq 3$, since in the higher dimensional setting, edges of polygons can cross in the limit, creating  artificial vertices in the  perturbations we consider, so that they no longer belong to the constraint set; see Remark~\ref{difficulty} and Figure~\ref{omeganfig}.
For this reason, it remains an open problem whether minimizers of \eqref{mainproblem} exist for dimensions  $d \geq 3$.

\begin{figure}[t!] 
\includegraphics[width =0.8\textwidth]{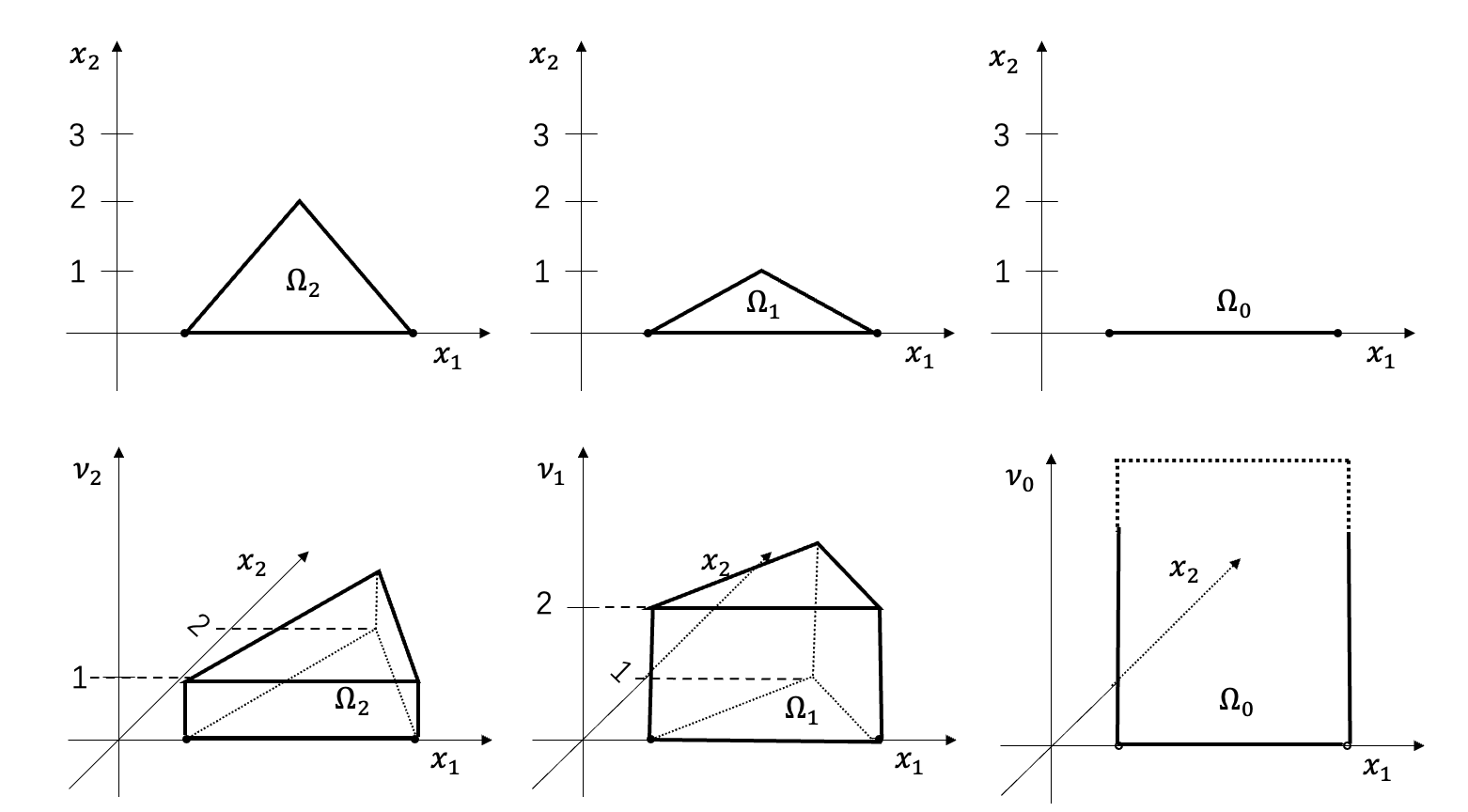}\hspace{0.4 cm} 
\caption{In this figure, we illustrate how limit points of $\{ 1_\Omega/|\Omega| \colon \Omega \in S_k\}$ in the narrow topology are not, in general, given by uniform measures on their support. For example, the top row of this figure illustrates the limit of a sequence of triangles $\Omega_\delta$, where the base of the triangle remains fixed, while the height of the triangle converges $\delta$ converges to zero. The second row of this figure shows the limit of the corresponding probability measures. We see that the limiting measure $\nu_0$ is supported on the one dimensional interval prescribed by the base of the triangle, but its density with respect to $\mathcal{L}^1$ is not constant, but piecewise linear, concentrating more mass in the interior of the interval, where the third vertex converged. In particular, we see that $\lim_{\delta \to 0} 1_{\Omega_\delta}/|\Omega_\delta| = \nu_0 \neq 1_{\Omega_0 }/|\Omega_0|$.  }\label{limitpointexample}
\end{figure}
 
Finally, in terms of uniqueness, we observe that, in dimensions higher than one, the minimizer is clearly not unique: for example, if $\mu$ is the uniform distribution on the unit ball, any rotation of an optimal polytope would also be optimal. Understanding whether uniqueness holds up to such invariances, remains an interesting open question.

As described above, our approach to proving existence of (\ref{mainproblem}) in two dimensions proceeds by adapting a perturbation argument of Cuesta, el. al. \cite{cuesta2003approximation}, which considered a related problem: given $\mu \in \P_{2,ac}(\Rd)$, find the \emph{convex set} $\Omega$ that minimizes $W_2(\mu, 1_{\Omega}/|\Omega|)$. This built on previous work by the same authors, which also examined optimal $W_2$ approximation of measures $\mu \in \P(\R)$ in one dimension by empirical measures, uniform measures on intervals, and ellipsoids \cite{cuesta2002shape}. The motivation of this work was to describe the \emph{shape} and \emph{flatness} of a measure $\mu$. Subsequent work by Belili and Heinich \cite{belili2006approximation} considered $W_2$ approximation of a measure $\mu$ over more general types of measures, including uniform measures on convex sets and uniform measures on sets of the form $\Omega_1 \setminus \Omega_2$ where $\Omega_1, \Omega_2$ are convex. Furthermore, Belili and Heinich suceeded in proving existence of a closest approximation under more general hypotheses on $\mu$: instead of requiring $\mu \in \P_{2, ac}(\Rd)$, they require that the affine hull of the support of $\mu$ is $d$-dimensional. However, an essential hypothesis in the work of Belili and Heinich is that the class of approximating measures satisfy certain symmetry properties \cite[condition (3)]{belili2006approximation}. For example, if $\nu_0$ were the limiting measure shown in  Figure \ref{limitpointexample}, there would have to exist a triangle $\Omega$ so that the first marginal of $1_\Omega/|\Omega|$ is $\nu_0$ and its second marginal is symmetric about the origin. Unfortunately, no such triangle exists. For this reason, while our study of existence of solutions to \eqref{mainproblem} is closely related to the aforementioned works, the fact that \eqref{mainproblem} constrains $\Omega$ to be a $k$-gon requires the development of new techniques.

Motivated by the fact that the key challenge in proving existence of minimizers to (\ref{mainproblem}) arises when the support of the  minimizing sequence collapses to a lower dimensional set, we now introduce a regularized version of the problem that prevents this degeneracy. For $m\geq 1$, the  $m$-R\'enyi entropy is given by 
\[\mathcal{U}_m(\nu) = 
\begin{cases} 
\frac{1}{m-1} \int \nu(x)^m dx &\text{ if } \nu \ll \mathcal{L}^d \text{ and } d \nu (x) = \nu(x) dx , \\ 
+\infty &\text{ otherwise.} 
\end{cases} 
\]
In particular, if $\nu = 1_\Omega / |\Omega|$ for $\Omega \in S_k$, we have
\begin{align} \label{intenshape} \mathcal{U}_m(1_\Omega/|\Omega|) = \begin{cases} \frac{1}{(m-1) |\Omega|^{m-1}} &\text{ if } m>1 , \\
 -\log (|\Omega|) & \text{ if } m =1 . \end{cases}
\end{align}
For fixed  $\varepsilon > 0$, $k \geq d+1$, $m \geq 1$, and $\mu \in \P_2(\Rd)$, we consider the regularized  problem:
\begin{align} 
\label{mainprobv2} \tag{$\text{WAA}_\varepsilon$}
 \min_{\Omega \in S_k}  \quad W_2 \left(\mu, 1_\Omega / |\Omega| \right) + \varepsilon\mathcal{U}_m(1_\Omega/|\Omega|) .
 \end{align}
Due to the fact that the regularization prevents the minimizer from collapsing to a lower dimensional set, we are able to obtain existence of solutions in all dimensions, for general $\mu \in \P_2(\Rd)$.
 
\begin{theorem}[Existence of minimizer for \eqref{mainprobv2}, in all dimensions]
\label{t:Existence}
Fix $\varepsilon > 0$, $k \geq d+1$, $m \geq 1$, and $\mu \in \P_{2}(\Rd)$. 
Then there exists a solution of \eqref{mainprobv2}.
\end{theorem}

The proof of Theorem~\ref{t:Existence} follows via Prokhorov's theorem and the lower semicontinuity of the Wasserstein metric and R\'enyi entropy in the narrow topology.   Uniqueness again fails in dimensions higher than two, due to the rotational invariance of the R\'enyi entropy.

\begin{remark}[Limit as $\varepsilon \to 0$]
If $d=1$ or $2$ and $\mu \in \P_{2, ac}(\Rd)$, similar arguments as in the proof of Theorem \ref{t:Existence2d} can be used to show that, for any $\e_n\to 0$ as $n\to+\infty$, the solutions to \eqref{mainprobv2} with $\e = \e_n$,  form a minimizing sequence to \eqref{mainproblem} and converge (up to a subsequence) to a solution of \eqref{mainproblem}.  
This shows   consistency of the regularized problem with the original problem, as the regularization is removed.
\end{remark}

Our final results consider practical application of Wasserstein archetypal analysis to data. In particular, we seek to understand how solutions of  the regularized problem (\ref{mainprobv2}) behave when a measure $\mu \in \P_2(\Rd)$ is a approximated by a sequence of empirical measure $\mu_n := \frac{1}{n} \sum_{i =1}^n \delta_{X_i}$, \ $X_i \overset{\text{iid}}{\sim} \mu$.
We show that, almost surely and up to a subsequence,  optimizers $\Omega_{n}$ for the empirical measure  $\mu_n$ converge as $n \to \infinity$  to an optimizer $\Omega$ for $\mu$.  We also  provide  a convergence rate in terms of the value of the objective function, based on a classical estimate of Horowitz and Karandikar \cite{horowitz1994mean}. While more sophisticated rates can be obtained using sharper moment estimates or estimates on the dimensionality of the support of $\mu$ \cite{fournier2015rate, weed2019sharp},   we use the classical estimate for simplicity and to preserve the focus on the archetypal analysis problem.

\begin{theorem}[Consistency and convergence for  (\ref{mainprobv2})]
\label{t:Consistency}
Fix $\varepsilon > 0$, $k \geq d+1$, $m \geq 1$, and $\mu \in \P_{2}(\Rd)$. 
Suppose 
\[ \mu_n = \frac{1}{n} \sum_{i=1}^n \delta_{X_i}  , \quad  X_i \overset{\text{iid}}{\sim} \mu. \]
For each $n \in \mathbb{N}$, let $\{\Omega_{n}\}_{n\in\N}$ be a   minimizer  to  \eqref{mainprobv2} for   $\mu_n$. 
Then the following hold:
\begin{enumerate}[(i)]
\item Almost surely, there exists $\Omega \in S_k$ so that, up to a subsequence, 
\begin{align}  1_{\Omega_{n}}/|\Omega_n| \to 1_\Omega/|\Omega| \text{ narrowly} , \label{indicatorsconverge}\end{align}
\label{consisconv}
and $\Omega$ solves \eqref{mainprobv2} for $\mu$; \label{consissol}
\item If $\beta: = \int_{\R^d}|x|^{d+5}d\mu(x)<+\infty$, then \label{ratepart}
\begin{align*}
\mathbb E\left[W_2 \left(\mu, \frac{1}{|\Omega_n|} 1_{\Omega_n} \right) + \e  \mathcal{U}_m(1_{\Omega_n}/|\Omega_n|) -\left(\min_{\Omega \in S_k} W_2 \left(\mu, \frac{1}{|\Omega|} 1_{\Omega} \right) + \e  \mathcal{U}_m(1_\Omega/|\Omega|)\right)\right] \leq \frac{C(\beta, d)}{n^{d+4}},
\end{align*}
where $C(\beta, d)$ is some constant depending on $\beta$ and $d$. 
\end{enumerate}
\end{theorem}

In addition to ensuring consistency of WAA under empirical approximation, the previous theorem is also relevant from the perspective of numerical methods. Our final main result is the development of a numerical method for solving (\ref{mainproblem}) and (\ref{mainprobv2}), based on approximating a given measure $\mu \in \P_2(\Rd)$ by a sequence of finitely supported measures $\mu_n = \sum_{i=1}^n \delta_{x_i} m_i$ and approximating the solution of (\ref{mainproblem}) or (\ref{mainprobv2}) for $\mu_n$. Such an approximation greatly simplifies the problem numerically, since the Wasserstein distance $W_2(\mu_n , 1_{\Omega}/|\Omega|)$ can now be computed using the \emph{semidiscrete} method introduced by M\'erigot  \cite{M_rigot_2011}. Based on this perspective, we introduce an alternating gradient-based method to approximate the optimizer; see Section~\ref{s:CompApproach} and Algorithm~\ref{alg1}.

In  Section~\ref{s:NumExp}, we consider several numerical experiments for approximating solutions of (\ref{mainproblem}) and (\ref{mainprobv2})   in two dimensions. 
First, we consider the case when $\mu$ is the uniform distribution on a disk (Section~\ref{sec:uniform}) and 
a normal distribution (Section~\ref{sec:normal}). In this case, the solution for each $k$   is a regular $k$-gon. 
Next, we study the sensitivity to the parameter $\varepsilon$ in \eqref{mainprobv2} (Section~\ref{s:EpsSensitivity}) and   provide an example that demonstrates the non-convexity of the energy landscape in (\ref{mainproblem}) (Section~\ref{s:NonConvexLandscape}). 
Finally, we consider an example of the early spread of the COVID-19 virus in the U.S. (Section~\ref{s:COVID}). 

We conclude in Section~\ref{s:Disc} with a discussion of our contributions and directions for future work. 
%


\section{Preliminaries}
Given a closed, convex set $S \subseteq \Rd$, let $\pi_S:\Rd \to S$ denote the projection on the set. Given any (Borel) measureable function $T \colon \Rd \to \Rd$ and $\mu \in \P(\Rd)$, the \emph{push-forward of $\mu$ through $T$}, denoted $T \# \mu $, is the probability measure defined by 
\[ \int_\Rd f(T(x)) d \mu(x) = \int_\Rd f(x) d (T \# \mu)(x) , \quad \forall f \text{ bounded and (Borel) measurable. }\]

We now recall several facts about the Wasserstein metric and uniform measures on convex polygons. First, recall that in one dimension, we have an explicit formula for the Wasserstein metric in terms of cumulative distributions functions (CDFs) \cite[Theorem 2.18]{Villani}. In particular, if $\mu, \nu \in \P_2(\R)$ have CDFs $F$ and $G$, then
\begin{align} \label{1dwasserstein} W_2(\mu, \nu) = \left( \int_0^1 |F^{-1}(t) - G^{-1}(t)|^2 d t  \right)^{1/2} ,\end{align}
where the \emph{generalized inverse} of $F$ is given by $F^{-1}(t) = \inf \{ x \in \R \colon F(x) >t \}$.

Next, we recall relevant notions of convergence, beginning with narrow convergence.
\begin{definition}[Narrow convergence of probability measures in $\R^d$]
We say that $\{\mu_n\}_{n=1}^\infty \subseteq \P(\Rd)$ converges narrowly to  $\mu\in\mathcal P(\R^d)$ if 
\begin{align}
&\int_{\R^d}f(x) d\mu_n(x)\to\int_{\R^d}f(x)d\mu(x), \qquad \forall f\in C_b(\R^d),
\end{align}
where $C_b(\R^d)$ denotes the set of bounded continuous functions on $\R^d$. 
\end{definition}
In the probability literature, narrow convergence is also called weak convergence or convergence in distribution. 
Narrow convergence is slightly weaker than convergence in $W_2$, and they are equivalent when the second moments also converge.
\begin{lemma}[{\cite[Remark. 7.1.11]{AGS}}]
Given $\mu_n , \mu \in \P_2(\Rd)$, the following statements are equivalent:
\begin{align} \label{charW2conv}
\lim_{n \to +\infty} W_2(\mu_n, \mu) = 0 \iff \mu_n \to \mu \text{ narrowly and }\lim_{n\to +\infty} M_2( \mu_n)= M_2( \mu).
\end{align}
\end{lemma} 

Recall that, for any measures $\mu, \nu \in \P_2(\Rd)$, if we consider their translations to have mean zero, $\mu' = (\id - \int x d \mu(x))\# \mu$ and $\nu' = (\id - \int x d \nu(x))\# \nu$, then
\begin{align} \label{WLOGmeanzero}
W_2^2(\mu,\nu) =  W_2^2(\mu',\nu') + \left| \int x d \mu(x) - \int x d \nu(x) \right|^2 .
\end{align}
See, for example, \cite[Proposition 2.5]{cuesta2003approximation}.

\begin{remark}[mean of minimizers of \ref{mainproblem} and \eqref{mainprobv2}]
An immediate consequence of equation (\ref{WLOGmeanzero}) is that  any minimizer $\Omega$ of either (\ref{mainproblem}) or the regularized problem \eqref{mainprobv2} must have the same mean as $\mu$, that is, 
\[ \frac{1}{|\Omega|} \int_{\Omega} x dx  \ = \ \int_{\mathbb R^d} x  d \mu(x) . \]
\end{remark}


Next, we recall an elementary lemma for the CDFs of real valued random variables. For the reader's convenience, we include a proof.

\begin{lemma}[{\cite[Equation 2]{cuesta2003approximation} and \cite[Theorem 2.18]{Villani}}]\label{katy's lemma} 
Let $X$ be a real-valued random variable with CDF $F(x)$. 
Assume $\mathbb E[|X|]<+\infty$ and the law of $X$ is not a Dirac mass. 
Then  
\begin{align} \label{keyineqq}
\mathbb E[XF(X)] > \frac12 \mathbb E[X] .
\end{align}
Furthermore, if $X \sim \mu$ for $\mu \in \P_{2,ac}(\R)$, then for any $a<b$, $(b-a)(F(x)+a)$ is the optimal transport map from   $\mu$ to the probability measure $1_{[a,b]}/(b-a)$. 
\end{lemma}

\begin{proof}
First, we show (\ref{keyineqq}). To prove this, note
\begin{align*}
\mathbb E[XF(X)]  &= \int_{\R} xF(x)dF(x) = \frac{1}{2}\int_{\R}xd(F(x)^2) = \frac{1}{2}\mathbb E[\widetilde{X}],
\end{align*}
where $\widetilde{X}$ has CDF $F(x)^2$. 
Since $F(x)\in [0,1]$, $\widetilde{X}$ stochastically dominates $X$, i.e. $\mathbb P(\widetilde{X}\geq t) = 1-F(t)^2\geq 1-F(t)= \mathbb P(X\geq t)$. As a consequence of stochastic dominance \cite[Theorem 1.A.8]{shaked2007stochastic}, $\mathbb E[\widetilde{X}]\geq \mathbb E[X]$, with equality holding if and only if $F(x)^2 = F(x)$ for all $x\in\R$. Thus, equality can only hold if $F(x)$ is either equal to zero or one for all $x \in \R$, which implies the law of $X$ is a Dirac mass. Consequently, we conclude $\mathbb E[XF(X)] = \mathbb E[\widetilde{X}]/2 > \mathbb E[X]/2$.

Now, suppose $\mu \in \P_{2, ac}(\R)$.   \cite[Remark 2.19]{Villani} ensures that $(b-a)(F(x)+a)$ is the optimal transport map from the law of $X$ to $1_{[a,b]}/(b-a)$. 
\end{proof}

Our next lemma concerns optimal transport maps from a measure $\mu \in \P_{2,ac}(\Rd)$ to a measure $\nu \in \P_2(\Rd)$ with $\supp \nu \subseteq \pi_1(\Rd) = \{ (x_1, \dots, x_d) \in \Rd \colon x_i = 0 \quad \forall i =2, \dots d \}$.

\begin{lemma} \label{transportmaponedim}
Given   $\mu \in \P_{2,ac}(\Rd)$ and $\nu \in \P_2(\Rd)$ with $\supp \nu \subseteq \pi_1(\Rd)$, if $T$ is the optimal transport map from $\mu$ to $\nu$, then $T = (T_1, 0, \dots, 0)$, where $T_1 \colon \R \to \R$ is the optimal transport map from $\pi_1 \# \mu$ to $\pi_1 \#\nu$.
\end{lemma}
\begin{proof}
Since $\supp \nu \subseteq \pi^1(\Rd)$, for any $Y = (Y_1, \dots, Y_d)$ with $Y \sim \nu$, we have $Y_i = 0$ almost everywhere, for $i=2, \dots, d$. Thus, by definition of the Wasserstein metric in equation \eqref{l1}, 
\begin{align*} W_2^2(\mu,\nu) 
&= \inf_{X \sim \mu, Y \sim \nu} \mathbb E[(X_1-Y_1)^2]  + \sum_{i=2}^d  \mathbb E[(X_i-Y_i)^2]  = \inf_{X \sim \mu, Y \sim \nu} \mathbb E[(X_1-Y_1)^2]  + \sum_{i=2}^d  \mathbb E[X_i^2]  \\
&= M_2(\pi_{2, \dots, d}\# \mu) + \inf_{X_1 \sim (\pi_1 \# \mu), Y_1 \sim  (\pi_1 \# \nu)} \mathbb E[(X_1-Y_1)^2]   ,
\end{align*}
where the second term in the last expression corresponds with $W^2_2(\pi_1\# \mu, \pi_1 \# \nu)$. Since $\mu \in \P_{2,ac}(\R^d)$, $\pi_1 \# \mu  \in \P_{2,ac}(\R)$, so there exists an optimal transport map $T_1$ from $\pi_1\#\mu$ to $\pi_1 \#\nu$; see equation \eqref{existence of OT map}. Furthermore, the above computation shows that, $\tilde{T} = (T_1, 0 , \dots, 0)$ is optimal from $\mu$ to $\nu$. Uniqueness of optimal transport maps for $\mu \in \P_{2, ac}(\Rd)$ then gives the result. 
\end{proof}
Another useful fact is that, for uniform distributions on convex k-gons in $\R^d$,  uniformly bounded second moments imply uniformly bounded support, as well as convergence, up to a subsequence. This is proved in a slightly different setting in \cite[Lemmas 3.2-3.3]{cuesta2003approximation} and  \cite[Proposition 1]{belili2006approximation}. We recall  the result in the present setting for the reader's convenience.  
\begin{lemma} \label{boundedpolygons}
Fix $d \geq 1$ and $k \geq d+1$. Suppose   $\{\Omega_n\}_{n=1}^{+\infty} \subseteq S_k$ satisfies $\sup_n M_2 \left(1_{\Omega_n}/|\Omega_n| \right) <+\infty$. 
\begin{enumerate}[(i)]
\item   \label{polygonsinball}
There exists $  R >0 \text{ so that } \Omega_n \subseteq B_R(0):=\{x\in\R^d \colon |x|\leq R\}  \ \forall  n \in \mathbb{N}$ .
\item Up to a subsequence, the vertices of $\Omega_n$ converge, and if we let $\Omega$ be their convex hull, then  there exists $\nu \in \P_2(\Rd)$  so that $1_{\Omega_n}/|\Omega_n| \to \nu$ narrowly,  $1_{\Omega_n}(x) \to 1_\Omega(x)$ pointwise $\mathcal{L}^d$-almost everywhere, and $\supp \nu = \Omega$. 
\end{enumerate}
\end{lemma}
\begin{proof}
First, we prove part \eqref{polygonsinball}.  
By H\"older's inequality and the fact that $\mu_n:= 1_{\Omega_n}/|\Omega_n|$ are probability measures,
\[ \int x d \mu_n(x) \leq \left( \int |x|^2 d \mu_n(x) \right)^{1/2} = M_2(\mu_n)^{1/2}, \]
so that the means of $\mu_n$ are   uniformly bounded. Assume, for the sake of contradiction, that item (\ref{polygonsinball}) fails. Then there must exist a line segment $L_n \subseteq \Omega_n$ so that, up to a subsequence, $|L_n| = \mathcal{L}^1(L_n) \to +\infty$. Then, using the fact from  \cite[Lemma 3.2]{cuesta2003approximation} that there exists $\gamma  = \gamma (d) >0$ so that $M_2( \pi_{L_n} \# \mu_n) \geq \gamma |L_n|^2$, we obtain
\[ \int|x|^2 d \mu_n(x) \geq \int |\pi_L(x)|^2 d \mu_n(x) = M_2( \pi_L \# \mu_n)  \geq \gamma |L_n|^2 \to +\infty ,\]
which is a contradiction.

Now, since \eqref{polygonsinball} holds,  by Prokhorov's theorem, the sequence is tight. Thus, up to a subsequence, the sequence converges in the narrow topology to some $\nu \in \P(\Rd)$. By lower-semicontinuity of $M_2$ with respect to narrow convergence, we have $M_2(\nu)< +\infty$, so $\nu \in \P_2(\Rd)$.  By Heine-Borel,  up to a subsequence,  each of the vertices $\{v^i_n\}_{i=1}^k$ of $\Omega_n$ converges to some $\{v^i\}_{i=1}^k$. Let $\Omega$  be the closed, convex hull of these limit points.

In order to prove convergence of the indicator functions $\mathcal{L}^d$-a.e., we first prove the following: \ \\

 \noindent \emph{\textbf{Claim:} Let $\{\Omega_n\}_{n=1}^\infty $ be a sequence of convex $k$-gons with nonempty interior so that their vertices converge, and let $\Omega$ be the convex hull of the limit points  (which may have empty interior). Then if $x$ belongs to the interior of $\Omega$,   there exists $\delta>0$ so that $B_\delta(x) \subseteq \Omega_n$ for $n$ sufficiently large. }

\noindent\textbf{Proof of Claim:} Let $\{ w^j\}_{j=1}^l$ denote the vertices of $\Omega$. By assumption, for all $j =1, \dots, l$, there exists $ w_n^j \in \Omega_n$ so that $\lim_{n \to +\infty} w_n^j = w^j$. Thus, for all $\epsilon >0$, there exists $N$ so that $n>N$ ensures $| w^j -  {w}^j_n |< \epsilon$ for all $j =1, \dots, l$. Since $\Omega_n$ is convex, this implies there exists $\delta>0$ so that $B_\delta(x) \subseteq \Omega_n$ for all $n > N$. \ \\

We now apply this claim to prove convergence of the indicator functions $\mathcal{L}^d$-a.e. In particular, if $x \in \Omega^\circ$, then the previous claim shows $x \in \Omega_n^\circ$ for $n$ sufficiently large. Likewise, if $x \in \Omega^c$, then convergence of the vertices ensures  that $x \in \Omega^c_n$ for $n$ sufficiently large. Combining these gives
\begin{align} \label{firstpartpointwise}
1_{\Omega}(x)  =1 =  \lim_{n \to +\infty} 1_{\Omega_n}(x)  &\qquad \text{ if }x \in \Omega^\circ, \\
1_{\Omega}(x) = 0 = \lim_{n\to+\infty} 1_{\Omega_n}(x) &\qquad \text{ if } x \in \Omega^c .\label{secondpartpointwise}
\end{align}
Since $\mathcal{L}^d(\partial \Omega) = 0$, this  shows that $1_{\Omega_n} \to 1_\Omega$ $\mathcal{L}^d$-a.e.

We now show $\supp \nu = \Omega$. First, we show $\supp \nu \subseteq \Omega$. Suppose $x \in \Omega^c$. Since $\Omega^c$ is open, there exists an open ball so that $x \in B \subset \subset \Omega^c$. Convergence of the vertices ensures   $B \subseteq \Omega^c_n$ for $n$ sufficiently large. Thus, by the fact that $\mu_n \to \nu$ narrowly,
\[ 0 = \liminf_{n \to +\infty} \mu_n(B) \geq \nu(B) . \]
This shows $x \not \in \supp \nu$.

Now, we show $\Omega \subseteq \supp \nu$. Let $\pi$ be a projection onto the affine hull of $\Omega$, $\aff(\Omega)$. If $\dim(\aff(\Omega)) = 0$, then $\Omega$ is a single point, and  the previous implication that $\supp \nu \subseteq \Omega$, combined with  the fact that $\nu \in \P_2(\Rd)$ must have nontrivial support, ensures $\supp \nu = \Omega$. Now, assume $\dim(\aff(\Omega))>0$. Fix $x \in \Omega$ and let $C \subseteq \R^d$ be a closed ball containing $x$. 
It suffices to show that $\nu(C) >0$. Note that the convex $k$-gons $\pi(\Omega_n)$ have nonempty interior with respect to the topology on $\pi(\Rd)$ and   $\pi(\Omega) $ is the convex hull of the limits of their vertices. Furthermore, since $\pi(\Omega)$ is convex, there exists $y \in \pi(\Omega)^\circ$ (interior taken with respect to the topology on $\pi(\Rd)$) and $\epsilon >0$ so that  $B_\epsilon(y) \subseteq \pi(C) \cap \pi(\Omega)$, where $B_\epsilon(y)$ is an open ball in the affine subspace $\pi(\Rd)$. 

By the above claim, there exists $\delta>0$ so that $B_\delta(y) \subseteq \pi(\Omega_n) \cap \pi(C)$ for $n$ sufficiently large.  Then, by \cite[Lemma 5]{belili2006approximation}, there exists $k>0$ so that 
\begin{align} \label{firstlowerbound}
k \leq \frac{ |\Omega_n \cap \pi^{-1}(B_\delta(y))|}{|\Omega_n|} = \mu_n(\pi^{-1}(B_{\delta}(y))) = \pi \# \mu_n(B_\delta(y)) \leq \pi \# \mu_n(\overline{B_\delta(y)}).
\end{align}
Since $\overline{B_\delta(y)}$  is closed  and the continuous mapping theorem ensures $\pi \# \mu_n \to \pi \# \nu$ narrowly, taking the limsup of (\ref{firstlowerbound}) as $n \to +\infty$ gives
\begin{align*}
k \leq \pi \# \nu (\overline{B_\delta(y)}) = \nu ( \pi^{-1}(\overline{B_\delta(y)}))    \leq \nu(C) ,
\end{align*}
where the last inequality follows since $\overline{B_\delta(y)} \subseteq \pi(C)$, so $\pi^{-1} \left(\overline{B_\delta(y)} \right) \cap \supp \nu \subseteq C$. This gives the result.
%
%
\end{proof}

Our next lemma identifies the density of a narrowly convergent sequence in $\P_{2,ac}(\Rd)$, with uniformly bounded support and density.
\begin{lemma}\label{smalllemma1}
Consider $\nu_n\subseteq\mathcal P_{2,ac}(\R^d)$, where $d\nu_n(x) = \nu_n(x) d \mathcal{L}^d(x)$, and suppose 
\begin{enumerate}[(i)]
\item $\nu_n$ has uniformly bounded support, 
\item $\sup_n \| \nu_n \|_{L^\infty(\Rd)} < +\infty$,   
\item $\nu_n \to \nu \in \P_2(\Rd)$ narrowly. 
\end{enumerate}
Then, if  there exists  $f \in L^\infty(\Rd)$ so that $\nu_n(x) \to f(x)$ pointwise, we must have $ d \nu = f(x) d \mathcal{L}^d(x)$.
\end{lemma}

%

\begin{proof}
Fix $\varphi \in C_b(\Rd)$. Since $\nu_n(x)$ is uniformly bounded, with uniformly bounded support, by the dominated convergence theorem.
\begin{align} \label{sufficestoshowweak} \lim_{n \to +\infty} \int \varphi(x) \nu_n(x) d \mathcal{L}^d(x) = \int \varphi(x) f(x) d \mathcal{L}^d(x) .
\end{align}
Furthermore,   narrow convergence of $\nu_n$ to $\nu$  ensures that the left hand side coincides with $\int \varphi(x) d \nu(x)$. Thus, the Riesz–Markov–Kakutani representation theorem implies $d \nu(x) = f(x) dx$. 
\end{proof}

We close with the following technical lemma, focused on the two dimensional case, which characterizes the projection of  measures of the form $1_\Omega/|\Omega|$ onto a one dimensional affine space.

\begin{lemma}\label{smalllemma}
Let $d =2$ and $k \geq d+1$. For every $\Omega\in S_k$ and affine space $V\subseteq \Rd$ with $\dim(V) = 1$, the density of the marginal distribution of $1_{\Omega}/|\Omega|$ along $V$, denoted $\pi_V \#1_{\Omega}/|\Omega|$, with respect to one dimensional Lebesgue measure on $V$, is a   piecewise linear function on $\Omega$, with at most $k$ vertices, that is concave on its support.
\end{lemma}

\begin{proof} 
When $k =3$, the result can be checked by brute force. 

For $k >3$, we observe that any measure of the form $1_\Omega/|\Omega|$ for   $ \Omega \in S_k$ can be written as a conical combination of uniform distributions on triangles,
\[  1_\Omega/|\Omega| = \sum_{i=1}^l 1_{T_i}   /|\Omega| . \]
Thus,  by linearity of the push forward,
\[ \pi_V \#1_{\Omega}/|\Omega| = \sum_{i=1}^l (\pi_V \#1_{T_i})  /|\Omega| .\]

By the $k=3$ case, we know that there exists a piecewise linear function $f_i$, concave on its support, so that $d (\pi_V \#1_{T_i})  = f_i d \mathcal{L}^1$. Since the  space of such functions is a convex  cone, we obtain that $d( \pi_V \#1_{\Omega}/|\Omega|) = f d \mathcal{L}^1$ for $f$ piecewise linear and concave on its support. Finally, the vertices of $f$ (including the endpoints) can only occur at $\pi_V(x)$ for $x$ in the extreme set of $\Omega$, which by definition has cardinality at most $k$. 
\end{proof}

\section{Existence of minimizers}
\label{s:2d}

 In order to prove Theorem~\ref{t:Existence2d} on the existence of a minimizer for \eqref{mainproblem}, we begin with the following lemma, 
characterizing the closure of the constraint set in the narrow topology. We consider three cases, depending on the dimension of the affine hull  of the support of the measure.
\begin{proposition}\label{ourlemma}
Let $d\geq 1$ and $k \geq d+1$, and suppose $\nu \in \P_2(\Rd)$ is the narrow limit of  $1_{\Omega_n}/|\Omega_n| $, for   $\{\Omega_n\}_{n=1}^{+\infty} \subseteq S_k$  with  $\sup_n M_2(1_{\Omega_n}/|\Omega_n|)<+\infty$. Let $\Omega = \supp \nu$.
\begin{enumerate}[(i)]
\item If $\dim(\aff(\Omega)) =0$, then $\nu$ is a Dirac mass at $\supp(\nu)$, i.e., $\nu = \delta_{\supp(\nu)}$. \label{closurei}
\item If $d =2$ and $\dim(\aff(\Omega)) =1$, then the projection of $\nu$ onto $\aff(\Omega)$ is absolutely continuous with respect to the Lebesgue measure on $\aff(\Omega)$ and has a   piecewise linear density with at most $k$ vertices, which is concave on its support. \label{closureii}
\item If $\dim(\aff(\Omega)) = d$, then $\nu = 1_\Omega/|\Omega|$ for $\Omega \in S_k$. \label{closureiii}
\end{enumerate}
\end{proposition}

\begin{proof}[Proof of Proposition \ref{ourlemma}]
First, recall from Lemma \ref{boundedpolygons}, that $\Omega$ is a convex $k$-gon,  $\Omega_n$ and $\Omega$ are uniformly bounded in $\Rd$, and $1_{\Omega_n}$ converges pointwise a.e. to $1_\Omega$.

We now consider part (\ref{closurei}). This result is immediate, since the limit of any narrowly convergent sequence of probability measures must be a probability measure, and the only probability measure concentrated on a point is a Dirac mass at that point.

Next we  show part (\ref{closureii}), under the assumption that $d=2$. 
 Since $\Omega$ is a  bounded, convex k-gon with $\dim(\aff(\Omega)) = 1$, $\Omega$ must be a bounded line segment.  For simplicity of notation, let $V = \aff(\Omega)$, and let $\pi_{V} \# (1_{\Omega_n}/|\Omega_n|)$ denote the marginal distribution of $1_{\Omega_n}/|\Omega_n|$ along $V$. 
By Lemma \ref{smalllemma}, the density of $\pi_{V} \# (1_{\Omega_n}/|\Omega_n|)$ with respect to one dimensional Lebesgue measure on $V$,  which we denote by $f_n$, is a   piecewise linear function with at most $k$ vertices, which is concave on its support. Without loss of generality, suppose $V$ coincides with the $x_1$-axis, so $f_n$ is a function of $x_1$.

We now show that $f_n$ is uniformly bounded for all $n$.  To do this, we begin by showing that, up to a subsequence, $\inf_n \mathcal{L}^1(\pi_V(\Omega_n))>0$.  By Lemma \ref{boundedpolygons},  the vertices  $\{v^i_n\}_{i=1}^k$ of $\Omega_n$ converge to $\{v^i\}_{i=1}^k$, where $\Omega$ is the convex hull of these limit points. By definition of $V$, $v^i = (v_1^i,0)$ and $\pi_V(v^i) = v_1^i$ for all $i=1, \dots, k$.  Furthermore, note that $\pi_V(\Omega_n)$ is a sequence of uniformly bounded intervals on the $x_1$-axis, with nonempty interior with respect to $\R$. Thus, by Lemma \ref{boundedpolygons}, up to a subsequence, the endpoints of the intervals converge, and  we may let $\Omega' \subseteq \mathbb{R}$ denote the convex hull of their limits. By uniqueness of limits and the continuity of $\pi_V$, we have $\Omega' = \pi_V(\Omega)$, so  Lemma \ref{boundedpolygons} ensures $1_{\pi_V(\Omega_n)}(x_1) \to 1_{\pi_V(\Omega)}(x_1)$ pointwise $\mathcal{L}^1$ almost everywhere.
Thus, by the dominated convergence theorem,
\begin{align*}
 \lim_{n \to +\infty } \mathcal{L}^1(\pi_V(\Omega_n)) = \lim_{n \to +\infty} \int 1_{\pi_V(\Omega_n)} d \mathcal{L}^1= \int 1_{ \pi_V(\Omega)}(x_1) d x_1 = \mathcal{L}^1( \pi_V(\Omega)).
 \end{align*} 
 Since $\dim(\aff(\Omega))=1$, $\mathcal{L}^1( \pi_V(\Omega))>0$. Thus, up to another subsequence, we may assume that $\inf_n  \mathcal{L}^1(\pi_V(\Omega_n)) >0$.
 
Since $f_n$ is a concave piecewise linear density, for any $x_1 \in \R$, the area of the triangle with base $\mathcal{L}^1(\pi_V(\Omega_n))$ and height $f_n(x_1)$ is always smaller than $\int f_n d \mathcal{L}^1$. Therefore, for all $n \in \mathbb{N}$ and $x_1 \in \pi_V(\Omega_n)$,
\begin{align*}
\frac{1}{2} \left( \inf_m \mathcal{L}^1(\pi_V(\Omega_m)) \right) f_n(x_1)& \leq \frac12 \mathcal{L}^1(\pi_V(\Omega_n)) f_n(x_1)  \leq\int  f_n  d \mathcal{L}^1 = 1\Rightarrow f_n(x_1)\leq 2 \left( \inf_m \mathcal{L}^1(\pi_V(\Omega_m)) \right)^{-1}.  
\end{align*}
This shows that the densities $f_n$ are uniformly bounded in $n$.

We now seek to apply Lemma \ref{smalllemma1}.   
By the Continuous Mapping Theorem,  
\[ f_n d \mathcal{L}^1 = \pi_{V} \# (1_{\Omega_n}/ |\Omega_n|)\to \pi_{V} \# \nu  \text{ narrowly.} \]
Given that $f_n(x)$ is piecewise linear with at most $k$ vertices,  we may let   $x_{n, 1}\leq \cdots \leq x_{n,k}$ denote the vertices in increasing order. Since the support of $f_n$ is bounded, up to a subsequence, $x_{n,i}$ converges to some $x_i$ for all $i =1, \dots k$, and since $f_n$ is uniformly bounded,  $f_n(x_{n, i})$ converges to some $y_i$ as  for all $i =1, \dots, k$. Let $f$ be the piecewise linear functions that interpolates between $\{(x_i, y_i)\}_{i=1}^k$. Then $f_n \to f$ pointwise so, by   Lemma \ref{smalllemma1}, $d (\pi_V \#\nu) = f d \mathcal{L}^1$.  Finally, since $f_n$ are concave on their support, so is $f$. This completes the proof of part (\ref{closureii}).

It remains to show part (\ref{closureiii}), which holds for general $d \geq 1$. Since $1_{\Omega_n}$ converges pointwise $\mathcal{L}^d$ a.e. to $1_\Omega$ and both have uniformly bounded support,  the dominated convergence theorem ensures
\begin{align*}
 \lim_{n \to +\infty } |\Omega_n| = \lim_{n \to +\infty} \int 1_{\Omega_n} d \mathcal{L}^d= \int 1_{ {\Omega}} d \mathcal{L}^d = | {\Omega}|.
 \end{align*} 
 Therefore, up to a subsequence, we may assume $\inf_{n} |\Omega_n | >0$.  Applying the dominated convergence theorem again, we conclude that for any Borel measurable set $A \subseteq \mathbb{R}^d$,
\begin{align*}
  \lim_{n \to +\infty} \frac{1}{|\Omega_n|} \int_{\Omega_n} 1_A d \mathcal{L}^d   =\lim_{n \to +\infty} \frac{1}{|\Omega_n|}  \int 1_{A }1_{ \Omega_n}d \mathcal{L}^d = \frac{ \int 1_{A }1_{ {\Omega}}}{| {\Omega}|} d \mathcal{L}^d .
\end{align*}
This shows that $1_{\Omega_n}/|\Omega_n| \to 1_{ {\Omega} }/| {\Omega} |$ strongly as measures, hence narrowly. 
By uniqueness of limits, we conclude that $\nu = 1_{ {\Omega}}/|{\Omega}|$. Furthermore, since $\dim(\aff(\Omega))=d$, $\Omega \subseteq \Rd$ must have non-empty interior, so $\Omega \in S_k$.
\end{proof}

We now turn to the proof of  Theorem~\ref{t:Existence2d}. 
\begin{proof}[Proof of Theorem~\ref{t:Existence2d}]
First, suppose $d=1$. Let  $F$ be the CDF of $\mu$, and let $G$ be the CDF of the probability measure $\frac{1}{|b-a| } 1_{[a,b]}$.  Then $G^{-1}(t) := a + (b-a) t$, so, by equation (\ref{1dwasserstein}),
\begin{align*}
W^2_2\left(\mu, \ \frac{1}{|b-a| } 1_{[a,b]} \right) &= \int_0^1 \left( F^{-1}(t) - G^{-1}(t) \right)^2 dt \\
&= \int_0^1 \left( F^{-1}(t) - \left[ a + (b-a) t \right]  \right)^2 dt \\
&=  \frac{1}{6}\begin{pmatrix} a\\ b \end{pmatrix}^t \begin{pmatrix} 2 & 1\\ 1 & 2 \end{pmatrix} \begin{pmatrix} a\\ b \end{pmatrix}
- 2\begin{pmatrix} C_{0} - C_{1}\\ C_{1} \end{pmatrix}^t
\begin{pmatrix} a\\ b \end{pmatrix}
+ C_{2} =: \varphi(a,b), 
\end{align*}
where 
\begin{align*}
C_{0} &= \int_0^1F^{-1}(t) dt = \int_{\Rd} x d\mu(x)  , \\
C_{1} &=\int_0^1t F^{-1}(t) dt = \int_{\Rd} x F(x) d \mu(x) ,  \\
C_{2} &= \int_0^1 \left( F^{-1}(t) \right)^2 dt = \int_{\Rd} x^2 d \mu(x) . 
\end{align*}
The result then follows from the fact that $\varphi(a,b)$ is a strictly convex quadratic function, with unique minimizer given in \eqref{e:1dSol}.  

Now, we turn to the case $d=2$. 
By equation \eqref{WLOGmeanzero}, we may assume without loss of generality that $\mu$ has mean zero and restrict the minimization problem to the set $\{\Omega \in S_k \colon \int_\Omega x d \mathcal{L}^d(x) = 0 \}$.
Consider a minimizing sequence $\Omega_n \in S_k$ with $\int_{\Omega_n} x d \mathcal{L}^d(x) = 0$ so that 
\begin{align} \label{minimizingsequence}
\lim_{n \to +\infty} W_2(\mu, 1_{\Omega_n}/|\Omega_n|)   = \inf_{\Omega \in S_k} W_2(\mu, 1_\Omega/|\Omega|) .
\end{align}
By the triangle inequality,
\[ \sup_n M_2(1_{\Omega_n}/|\Omega_n|) = \sup_n W_2^2(1_{\Omega_n}/|\Omega_n|, \delta_0) \leq \sup_n 2 \left( W_2^2( \mu,\delta_0)+ W_2^2(\mu, 1_{\Omega_n}/|\Omega_n|) \right) < +\infty . \]
Thus, Lemma \ref{boundedpolygons} ensures that, up to a subsequence, there exists $\nu \in \P_2(\Rd)$ and a closed, convex $k$-gon $\Omega$ so that $1_{\Omega_n}/|\Omega_n| \to \nu$ narrowly and $\supp \nu = \Omega$. Since $1_{\Omega_n}/|\Omega_n|$ has mean zero for all $n \in \mathbb{N}$ and uniformly bounded support, $\nu$ also has mean zero. By the lower semicontinuity of the Wasserstein metric with respect to narrow convergence,
\begin{align} \label{nuassmallaspossible}
 W_2(\mu, \nu) \leq \lim_{n \to +\infty} W_2(\mu, 1_{\Omega_n}/|\Omega_n|) = \inf_{\Omega \in S_k} W_2(\mu, 1_\Omega/|\Omega|) . 
 \end{align}
Thus, it suffices to show that $\nu = 1_\Omega/|\Omega|$ for $\Omega \in S_k$ in order to conclude that a minimizer exists.
By Proposition \ref{ourlemma}, if $\dim(\aff(\Omega))=2$, the result holds. Thus, it remains to exclude the possibility that $\dim(\aff(\Omega)) \leq 1$. We accomplish this by showing that, if either $\dim(\aff(\Omega))=0$ or $\dim(\aff(\Omega))=1$, then there exists $\tilde{\Omega} \in S_k$ with $W_2(\mu, 1_{\tilde{\Omega}}/|\tilde{\Omega}|) < W_2(\mu,\nu)$, contradicting equation (\ref{nuassmallaspossible}).

 Let $X = (X_1,X_2)$ be a random variable with distribution $\mu$, and let $F(x_1)$ denote the CDF of $X_1$. Let $X_2|X_1$ denote the conditional distribution of $X_2$ given $X_1$, with CDF $F_{X_1}(x_2)$. By Lemma \ref{katy's lemma}, $F$ is the optimal transport map from $\mu$ to $1_{[0,1]}$,  $F - 1$ is the optimal transport map from $\mu$ to $1_{[-1,0]}$, and, almost everywhere, $F_{X_1}$ is the optimal transport map from the law of $X_2 | X_1$ to $1_{[0,1]}$. Suppose that 
 \begin{align} \label{ghyp}
 &g \colon \R \to [0,+\infty) \text{ is   compactly supported, $\int g = 1$, and, on its support,} \\
 &\text{$g$ is concave and piecewise linear,   with at most $k$ vertices.} \nonumber\end{align}
In particular, we have
\begin{align}\label{fpositivity}
g(x_1) >0 \text{ for }x_1 \in {\rm int}(\supp g) .
\end{align}

For $\alpha \in \R \setminus \{0\}, \beta >0$, consider the following family of convex $k$-gons:
\[ \Omega(\alpha, \beta) := \{ (x_1,x_2) \colon   x_2 \in [0, \alpha g(x_1/\beta) \} \in S_k. \]
Note that $|\Omega(\alpha,\beta)| = \alpha \beta$. Let $T_1$ be the optimal transport map from $\pi_1 \# \mu$ to $ \pi_1 \# 1_\Omega(1,1)$. 
Define the random variables
\begin{align*}
 Z_1&= T_1(X_1) , \\
 Z_2 &=  g(T_1(X_1)) (F_{X_1}( X_2)-C_\alpha),   \quad \text{ for } \quad C_\alpha = \begin{cases} 0 &\text{ if } \alpha >0, \\ 1 &\text{ if } \alpha <0, \end{cases} \\
  Y &= (\beta Z_1,|\alpha| Z_2)  .
 \end{align*}
By construction,   $Y \sim 1_{\Omega(\alpha,\beta)}/|\Omega(\alpha,\beta)|$.
Since $\pi_1 \#  1_{\Omega(1,1)} \in \P_{2,ac}(\R)$ and $T_1(X_1) \in {\rm int}(\pi_1 \#1_{\Omega(1,1)}) = {\rm int}(\supp g)$ almost everywhere, equation (\ref{fpositivity}) and Lemma \ref{katy's lemma} ensure
\begin{align} 
|\alpha| \mathbb{E}[Z_2X_2] &= |\alpha| \mathbb E[g(T_1(X_1)) (F_{X_1}(X_2)-C_\alpha) X_2] \nonumber \\
&= |\alpha| \mathbb{E} [g(T_1(X_1 )) \left( \mathbb{E}[F_{ X_1}(X_2)X_2| X_1] - C_\alpha \mathbb{E}[X_2|X_1]\right) ] , \nonumber \\
&>  |\alpha| \mathbb{E} \left[g(T_1(X_1)) \left(  \frac{1}{2} \mathbb{E}[X_2|X_1]  - C_\alpha \mathbb{E}[X_2|X_1] \right) \right]\nonumber \\
&= \begin{cases}   \ \ \frac{|\alpha|}{2}  \mathbb{E} \left[g(T_1(X_1)) \left(  \mathbb{E}[X_2|X_1]    \right) \right] &\text{ if } \alpha >0 ,  \\
- \frac{|\alpha|}{2} \mathbb{E} \left[g(T_1(X_1)) \left(   \mathbb{E}[X_2|X_1]     \right) \right] &\text{ if } \alpha < 0 . \end{cases}\label{lastcontradictionstep0}
\end{align} 

We now apply this construction to rule out the possibility $\dim(\aff(\Omega)) \leq 1$. First, suppose $\dim(\aff(\Omega))=0$.  By Proposition \ref{ourlemma} and the fact that $\nu$ has mean zero, $\nu = \delta_0$, so $W_2^2(\mu,\nu) = \mathbb{E}[X_1^2+X_2^2]$.
 Define
\[ g(x_1) = \begin{cases} 1+x_1 &\text{ if }x_1 \in [-1,0], \\ 1-x_1 &\text{ if } x_1 \in [0,1]. \end{cases} \]
Then $g$ satisfies (\ref{ghyp}).
Furthermore,
\begin{align}
&W_2^2(\mu,\nu) - W_2^2(\mu, 1_{\Omega_{\alpha, \beta}}/|\Omega(\alpha,\beta)|)  \geq \mathbb{E}[X_1^2 + X_2^2] - \mathbb{E}[(X_1 - \beta Z_1)^2 + (X_2 - |\alpha| Z_2)^2 ] \nonumber \\
&\quad = 2 \beta \mathbb{E}[X_1 Z_1] - \beta^2 \mathbb{E} [Z_1^2] + 2 |\alpha|   \mathbb{E} [ Z_2X_2] - \alpha^2 \mathbb{E}[Z_2^2]  . \label{boundbelow1d}
\end{align}

We now apply inequality (\ref{lastcontradictionstep0}) to bound this strictly from below.
If $\mathbb{E} \left[g(T_1(X_1)) \left(  \mathbb{E}[X_2|X_1]    \right) \right]  \geq0$, then $\mathbb{E}[Z_2X_2]>0$ for all $\alpha >0$. Thus, we may choose $\alpha >0$ and $\beta>0$ sufficiently small so that the right hand side of (\ref{boundbelow1d}) is strictly positive. On the other hand, if $\mathbb{E} \left[g(T_1(X_1)) \left(  \mathbb{E}[X_2|X_1]    \right) \right]  \leq0$, then $\mathbb{E}[Z_2X_2]>0$ for all $\alpha <0$. Thus, we may choose $\alpha<0$ sufficiently close to zero and $\beta>0$ sufficiently small so that the right hand side of (\ref{boundbelow1d}) is strictly positive. In particular, in either case, there exist $\alpha \in \R \setminus \{0\}$ and $\beta >0$ so that 
\begin{align*}
W_2^2(\mu,\nu) > W_2^2(\mu, 1_{\Omega(\alpha,\beta)}/|\Omega(\alpha,\beta)|) , 
\end{align*}
which contradicts (\ref{nuassmallaspossible}). This shows that $\dim(\aff(\Omega))=0$ is impossible.

It remains to exclude the possibility that $\dim(\aff(\Omega))=1$. We proceed by contradiction, assuming  $\dim(\aff(\Omega))=1$. Without loss of generality, we may rotate $\mu$ and $\nu$ so that $\aff(\Omega)$ coincides with the $x_1$-axis. By  Proposition \ref{ourlemma}, $d(\pi_1 \# \nu) = g d \mathcal{L}^1$, for $g$ satisfying (\ref{ghyp}) above. Note that, in this case,  $\pi_1 \# (1_{\Omega(1,1)}/|\Omega(1,1)|) = \pi_1 \#\nu$.
By Lemma \ref{transportmaponedim}, the optimal coupling between $\mu$ and $\nu$ is of the form $(X, T(X))$, where $T(x_1,x_2) = (T_1(x_1), 0)$ and $T_1$ is the optimal transport map from $\pi_1 \# \mu$ to $\pi_1 \# \nu$. Therefore,
\begin{align} \label{W2smaller}
&W^2_2\left(\mu, \nu\right)-W^2_2\left(\mu,  1_{\Omega(\alpha,1)} / |\Omega(\alpha,1)| \right)\nonumber\\
&\quad \geq \mathbb E[(X_1-T_1(X_1))^2+X_2^2]-\mathbb E[(X_1-Z_1)^2+(X_2- |\alpha| Z_2)^2]\nonumber\\
&\quad = \mathbb E[(X_1-T_1(X_1))^2+X_2^2]-\mathbb E[(X_1 -T_1(X_1))^2+(X_2 - |\alpha| Z_2)^2]\nonumber\\
&\quad =  2 |\alpha|  \mathbb E[X_2Z_2] - \alpha^2  \mathbb E[Z_2^2].
\end{align}

As before, we   apply inequality (\ref{lastcontradictionstep0}) to bound this   from below.
If $\mathbb{E} \left[g(T_1(X_1)) \left(  \mathbb{E}[X_2|X_1]    \right) \right]  \geq0$, then $\mathbb{E}[X_2Z_2]>0$ for all $\alpha >0$. Thus, we may choose $\alpha >0$ sufficiently small so that the right hand side of (\ref{boundbelow1d}) is strictly positive. On the other hand, if $\mathbb{E} \left[g(T_1(X_1)) \left(  \mathbb{E}[X_2|X_1]    \right) \right]  \leq0$, then $\mathbb{E}[X_2Z_2]>0$ for all $\alpha <0$. Thus, we may choose $\alpha<0$ sufficiently close to zero   so that the right hand side of (\ref{boundbelow1d}) is strictly positive. In particular, in either case, there exists $\alpha \in \R \setminus \{0\}$ 
\begin{align*}
W_2^2(\mu,\nu) > W_2^2(\mu, 1_{\Omega(\alpha,\beta)}/|\Omega(\alpha,\beta)|) , 
\end{align*}
which contradicts (\ref{nuassmallaspossible}). This shows that $\dim(\aff(\Omega))=1$ is impossible.  
\end{proof}

\begin{remark}[regularity of $\mu$ when $d=2$] \label{muabscts}
 While in the one dimensional case, our proof of existence of minimizers to (\ref{mainproblem}) holds for all $\mu \in \P_2(\R)$, in the two dimensional case, we require $\mu$ to be absolutely continuous with respect to $\mathcal{L}^2$. This is used in our application of Lemma \ref{katy's lemma}. In particular, we must have the following: after an arbitrary rotation of the coordinate plane, if $X = (X_1,X_2)$ is a random variable with law $\mu$, then, on a set of positive measure, the conditional distribution of $X_2$ given $X_1$, $X_2|X_1$, is not a single Dirac mass. Note that this fails to be true when $\mu$ is an empirical measure. The assumption that $\mu \in \P_{2, ac}(\R^2)$ is also convenient, since it ensures $\pi_1\#\mu$ and the conditional distribution $d \mu_{x_1}(x_2)$ are absolutely continuous with respect to one dimensional Lebesgue measure, hence that optimal transport maps exist from these measures to any other measure in $\P_2(\Rd)$.
 
 It remains an open question how to extend our result to more general $\mu \in \P_2(\Rd)$, particularly $\mu$ with lower dimensional support, such as an empirical measure. While related work due to Belili and Heinich \cite{belili2006approximation} on approximating measures by general convex sets $\Omega$ succeeded in weakening the condition on $\mu$ to merely require  $\dim(\aff(\supp \mu))=d$, their approach strongly uses symmetry arguments, which fail in our setting.
\end{remark}

\begin{remark}[existence of minimizers for $d \geq 3$] \label{difficulty}
There are two key gaps that prevent obtaining existence of minimizers to (\ref{mainproblem}) when $d\geq 3$ by a similar approach. The first is an analogue of Proposition \ref{ourlemma}, characterizing how limits of minimizing sequences could degenerate.
This becomes more difficult in higher dimensions, as a degeneracy can occur downwards by more than one dimension: for example, a three dimensional polytope can collapse to a line segment.

The second, more significant, gap is that, even given an appropriate analogue of Proposition \ref{ourlemma}, the perturbation argument used in the proof of Theorem ~\ref{t:Existence2d} would fail. 
For example, consider the following sequence $\Omega_n \in S_k$ for  $d=3$ with $k = 4$, illustrated in Figure \ref{omeganfig},
\begin{align*}
&\Omega_n = \conv \left( \left\{(-1, 0, 0), (0,1,0), (0,-1,0), \left(\cos \left(\frac{\pi}{2n} \right), 0, \sin \left(\frac{\pi}{2n}  \right) \right) \right\} \right), \quad n\in\N.
\end{align*}  
While the polygons $\Omega_n$ converge to the square $\Omega = \conv \{(-1, 0, 0), (0,1,0), (0,-1,0), (1,0,0) \}$ in the $xy$-plane, the measures $1_{\Omega_n}/|\Omega_n|$  narrowly converge to the measure $\nu = f d \mathcal{L}^2$, where $f$ is the piecewise linear density,  supported on $\Omega$, that interpolates between $(-1, 0, 0)$, $(0,1,0)$, $(0,-1,0)$, $(1,0,0)$, and $(0, 0, 3/2)$. 
In other words, $f$ is a tent-like function with a peak in the interior. 
In this case, one cannot construct a competitive element in $S_k$, $k=4$, following the same path as in the proof of Theorem ~\ref{t:Existence2d}. In particular, perturbing $\nu$ by scaling the $z$-direction according to $f$, since would result in an element in $S_5$, instead of $S_4$.

While there do exist elements in $S_4$ whose marginal distributions in the $xy$-plane coincide with $\nu$, e.g. the polytopes with vertices $(-1, 0, 0)$, $(0,1,0)$, $(0,-1,0)$, $(1,0,h)$ for any $h\neq 0$, this example illustrates how a different approach from that used in Theorem \ref{t:Existence2d} would be needed to deal with the more intricate structure of the narrow closure of the constraint set found in higher dimensions.
\end{remark}

\begin{figure}[t!] 
\includegraphics[width =0.3\textwidth]{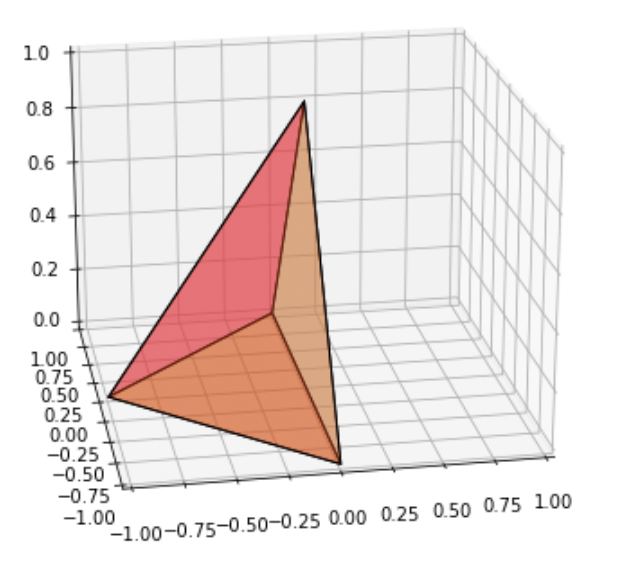}\hspace{0.4 cm}
\includegraphics[width =0.3\textwidth]{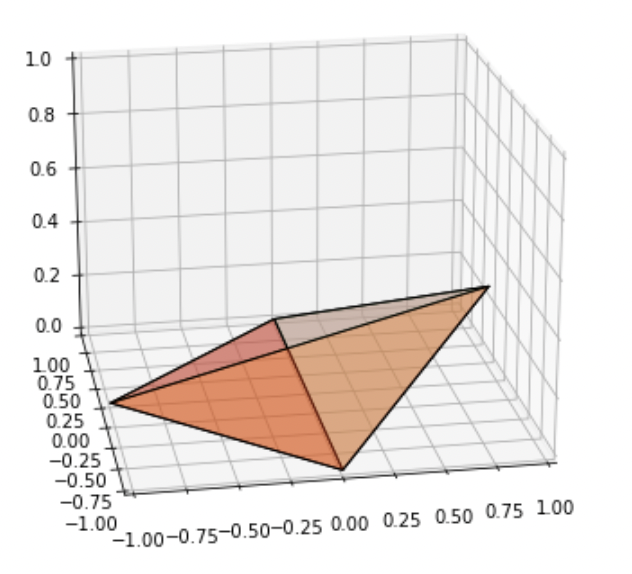}\hspace{0.4 cm}
\includegraphics[width =0.3\textwidth]{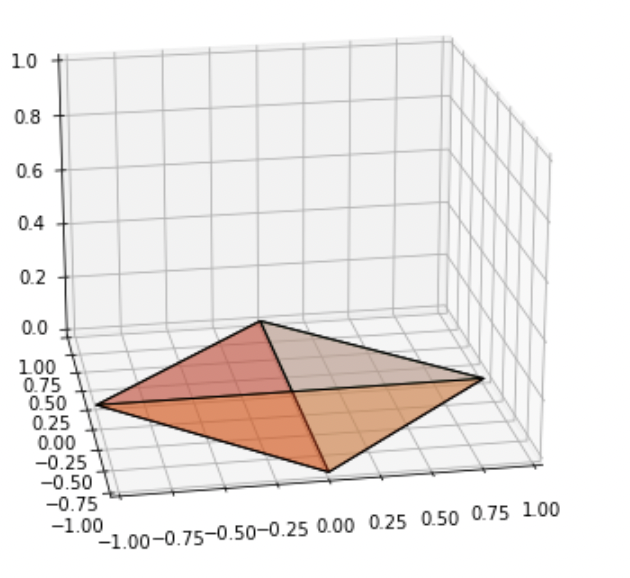}
\caption{Visualization of $\Omega_n$ as $n\to+\infty$.} \label{omeganfig}
\end{figure}


In spite of the difficulty of obtaining existence of minimizers to (\ref{mainproblem}) when $d \geq 3$, by introducing an arbitrarily small regularization term, we are able to obtain existence in all dimensions.  We begin with a lemma, providing compactness of the constraint set for the regularized problem \eqref{mainprobv2}.
 
 \begin{lemma} \label{Umcpt}
Consider $\mu_n \in \P_2(\Rd)$ with $\sup_n M_2(\mu_n) < +\infty$. Fix $\varepsilon >0$, $k \geq d+1$, and $m \geq 1$. Suppose $\Omega_n \in S_k$ satisfy 
 \begin{align*}
 \sup_n W_2(\mu_n, 1_{\Omega_n}/|\Omega_n|) + \varepsilon \mathcal{U}_m(1_{\Omega_n}/|\Omega_n|) < +\infty.
 \end{align*}
 Then there exists $\Omega \in S_k$ so that, up to a subsequence, $1_{\Omega_n}/|\Omega_n| \to 1_\Omega/|\Omega|$ narrowly.
 \end{lemma}
 \begin{proof}
 When $m >1$, $\mathcal{U}_m(1_\Omega/|\Omega|) \geq 0$, and when $m =1$, a Carleman-type estimate \cite[Lemma 4.1]{carrillo2016convergence} shows that, for all $\delta >0$,  $\mathcal{U}_1(1_\Omega/|\Omega|) \geq -(2 \pi/\delta)^{d/2} - \delta M_2(1_{\Omega}/|\Omega|)$.  Thus, along the minimizing sequence, taking $\delta = 1/(4\varepsilon)$ and applying the triangle inequality gives
\begin{align*} \sup_n   M_2(1_{\Omega_n}/|\Omega_n|) &= \sup_n 2 W_2^2(1_{\Omega_n}/|\Omega_n|, \delta_0) -  M_2(1_{\Omega_n}/|\Omega_n|) \\
&\leq \sup_n 4\left( W_2^2( \mu_n,\delta_0)+ W_2^2(\mu_n, 1_{\Omega_n}/|\Omega_n|) \right) -   M_2(1_{\Omega_n}/|\Omega_n|) \\
&\leq \sup_n  4 M_2(\mu_n) + 4 W_2^2(\mu_n, 1_{\Omega_n}/|\Omega_n|)  +  4 \varepsilon \mathcal{U}_m(1_{\Omega_n}/|\Omega_n|) + 4\varepsilon (8\varepsilon \pi)^{d/2} < +\infty . \end{align*}
Thus, Lemma \ref{boundedpolygons} ensures that, up to a subsequence, there exists $\nu \in \P_2(\Rd)$ and a closed, convex $k$-gon $\Omega$ so that $1_{\Omega_n}/|\Omega_n| \to \nu$ narrowly, $1_{\Omega_n}(x) \to 1_\Omega(x)$ pointwise almost everywhere, and $\supp \nu = \Omega$. Furthermore, there exists $R>0$ so that $\Omega_n \subseteq B_R(0)$ for all $n \in \mathbb{N}$. Since $\e>0$, $\mathcal{U}^m(1_{\Omega_n}/|\Omega_n|)$ is   bounded above uniformly along a minimizing sequence, so by equation (\ref{intenshape}), $|\Omega_n|$ is bounded below. Hence, by the dominated convergence theorem, 
\begin{align}
|\Omega| = \int_\Rd 1_\Omega = \lim_{n \to +\infty} \int_\Rd 1_{\Omega_n} = \lim_{n \to +\infty} |\Omega_n| >0. \label{torefer}
\end{align}
Therefore $\dim(\aff(\Omega))= d$, so by Proposition \ref{ourlemma}, we have $\nu = 1_\Omega/|\Omega|$.
 \end{proof}

Now, we turn to the proof of Theorem~\ref{t:Existence}, which ensures existence of solutions of the regulaized problem \eqref{mainprobv2}, for all $\mu \in \P_2(\Rd)$, $d \geq 1$.

\begin{proof}[Proof of Theorem~\ref{t:Existence}]
Our proof begins similarly to the proof of Theorem~\ref{t:Existence2d}.
By equation \eqref{WLOGmeanzero} and the fact that $\mathcal{U}_m$ is invariant under translations, we may assume without loss of generality that $\mu$ has mean zero and restrict the minimization problem to the set $\{\Omega \in S_k \colon \int_\Omega x d \mathcal{L}^d(x) = 0 \}$.
Consider a minimizing sequence $\Omega_n \in S_k$ with $\int_{\Omega_n} x d \mathcal{L}^d(x) = 0$ so that 
\begin{align} \label{minimizingsequence}
\lim_{n \to +\infty} W_2(\mu, 1_{\Omega_n}/|\Omega_n|) + \varepsilon\mathcal{U}_m(1_{\Omega_n}/|\Omega_n|) = \inf_{\Omega \in S_k} W_2(\mu, 1_\Omega/|\Omega|) + \varepsilon\mathcal{U}_m(1_{\Omega}/|\Omega|) .
\end{align}
By Lemma \ref{Umcpt}, with $\mu_n \equiv \mu$, there exists $\Omega \in S_k$ so that, up to a subsequence $1_{\Omega_n}/|\Omega_n| \to 1_\Omega / |\Omega|$ narrowly. 
 By lower semicontinuity of the Wasserstein metric  and $\mathcal{U}_m$ in the narrow topology, we have 
\begin{align} \label{nuachievesinf}
\inf_{\Omega \in S_k} W_2(\mu,1_\Omega/|\Omega|) + \varepsilon \mathcal{U}_m(1_\Omega/|\Omega|) &= \liminf_{n \to +\infty} W_2(\mu,1_{\Omega_n}/|\Omega_n|) +  \varepsilon \mathcal{U}_m(1_{\Omega_n}/|\Omega_n|) \\
& \geq W_2(\mu,1_{\Omega}/|\Omega|) + \mathcal{U}_m(1_{\Omega}/|\Omega|).  \nonumber
\end{align} 
Therefore, $\Omega$ is a minimizer of (\ref{mainprobv2}).
\end{proof}

\section{Consistency} 
We now turn to the proof of consistency for the regularized Wasserstein archetypal analysis problem, as the measure $\mu$ is approximated by a sequence of empirical measures.

\label{s:Proofs}

\begin{proof}[Proof of Theorem~\ref{t:Consistency}]
First, note that by definition of $\Omega_n$ as the solution of (\ref{mainprobv2}), for any $\tilde{\Omega} \in S_k$, we have
\begin{align} \label{competitor}
W_2( \mu_n, 1_{\Omega_n}/|\Omega_n|) +\varepsilon \mathcal{U}_m(1_{\Omega_n}/|\Omega_n|) & \leq  W_2( \mu_n, 1_{\tilde{\Omega}}/|\tilde{\Omega}|) +\varepsilon \mathcal{U}_m(1_{\tilde{\Omega}}/|\tilde{\Omega}|) .
\end{align}

We now turn to part (\ref{consisconv}). Since $\mu \in \P_2(\Rd)$, the law of large numbers ensures  
 $M_2(\mu_n) \to M_2(\mu)$ and $\mu_n \xrightarrow{narrowly} \mu$, almost surely. Thus, equation \eqref{charW2conv} ensures $W_2(\mu_n,\mu) \to 0$ almost surely.  Thus, Lemma \ref{Umcpt}  ensures that, almost surely, there is a subsequence such that $1_{\Omega_n}/|\Omega_n| \to 1_\Omega/|\Omega|$ narrowly, where $\Omega \in S_k$.
 
To see $\Omega$ is optimal, note that, by inequality (\ref{competitor}) and lower semicontinuity of the Wasserstein metric and $\mathcal{U}_m$ with respect to narrow convergence, almost surely, for any $\tilde{\Omega} \in S_k$,
\begin{align*}
W_2( \mu, 1_{\tilde{\Omega}}/|\tilde{\Omega}|) +\varepsilon \mathcal{U}_m(1_{\tilde{\Omega}}/|\tilde{\Omega}|)  &= \liminf_{n \to +\infty} W_2( \mu_n, 1_{\tilde{\Omega}}/|\tilde{\Omega}|) +\varepsilon \mathcal{U}_m(1_{\tilde{\Omega}}/|\tilde{\Omega}|)\\
&\geq \liminf_{n \to +\infty} W_2(\mu_n, 1_{\Omega_n}/|\Omega_n|) + \varepsilon \mathcal{U}_m(1_{\Omega_n}/|\Omega_n|) \\
&\geq W_2(\mu, 1_{\Omega}/|\Omega|) + \varepsilon \mathcal{U}_m(1_{\Omega}/|\Omega|) ,
\end{align*} 
so $\Omega$ solves (\ref{mainprobv2}).

Finally, to get the convergence rate in part (\ref{ratepart}), note that for all $n \in \mathbb{N}$, 
\begin{align*}
W_2 \left(\mu,   1_{\Omega_n} /|\Omega_n| \right) + \e \mathcal{U}_m( 1_{\Omega_n} /|\Omega_n| )&\leq W_2 \left(\mu_n,  1_{\Omega_n} /|\Omega_n|  \right) +\e \mathcal{U}_m( 1_{\Omega_n} /|\Omega_n| ) + W_2(\mu_n,\mu)\\
&\leq W_2 \left(\mu_n, 1_{\Omega} /|\Omega| \right) +\e \mathcal{U}_m(1_{\Omega} /|\Omega| ) + W_2(\mu_n,\mu)\\
&\leq W_2 \left(\mu, 1_{\Omega} /|\Omega|  \right) +\e \mathcal{U}_m(1_{\Omega} /|\Omega| ) + 2W_2(\mu_n,\mu),
\end{align*}
which rearranges to 
\begin{align*}
\mathbb E\left[W_2 \left(\mu,   1_{\Omega_n} /|\Omega_n| \right) + \e \mathcal{U}_m( 1_{\Omega_n} /|\Omega_n| )- (W_2 \left(\mu, 1_{\Omega} /|\Omega|  \right) +\e \mathcal{U}_m(1_{\Omega} /|\Omega| ))\right]&\leq 2\mathbb E[W_2(\mu_n,\mu)]\\
&\leq 2(\mathbb E[W^2_2(\mu_n,\mu)])^{1/2}\\
&\leq \frac{C(\beta, d)}{n^{d+4}},
\end{align*}
where the last step follows from \cite[Theorem 1]{horowitz1994mean}, where $C(\beta, d)$ is some constant depending on $\beta$ and $d$, completing the proof. 

\end{proof}

\section{Computational approach in two dimensions}
\label{s:CompApproach}
Here, we develop a computational approach for solving (\ref{mainproblem}) and (\ref{mainprobv2}) in dimension $d=2$, when $\mu$ is a sum of Dirac masses, based on the semidiscrete approach introduced by M\'erigot  \cite{M_rigot_2011}; see also \cite[Section 6.4.2]{santambrogio2015optimal}.
This approach is based on the  dual formulation of the 2-Wasserstein metric, which we now recall; see also \cite[Theorem 6.1.1, Theorem 6.1.4]{AGS}.

\subsection{Semidiscrete formulation of (\ref{mainproblem}) and (\ref{mainprobv2})}
Suppose that $\mu \in \P(X)$ and $\nu \in \P(Y)$, where $X, Y$ are compact subsets of $\Rd$.\footnote{The compactness assumptions on $X$ and $Y$ can be removed, but then the maximum in the dual problem becomes a supremum. Also, $X$, $Y$ compact is sufficient for our purposes, since our discrete measure $\mu$ and uniform distribution on a polygon $\nu$ will always be contained in  compact sets.} Then
\begin{align*}
   W_2(\mu,\nu)
&  = \max_{ {\varphi}  \in C(X)  } \Phi_\nu({\varphi})
\end{align*}
for
\begin{align}
\Phi_{\nu}({\varphi})&:=  \int_X  \varphi(x) d \mu(x) + \int_Y \varphi^*(y) d \nu(y) , \label{originalphi} \\
  \varphi^*(y) &:= \inf_{x \in X} |x-y|^2 - {\varphi}(x) . \label{varphic}
\end{align}
In the special case that $\mu$ is a sum of Dirac masses, $\mu   =  \sum_{i =1}^n m_i \delta_{x_i}  $, we may suppose $X = \{ x_i \}_{i =1}^n$ and identify  $C(X) \simeq \R^n$, abbreviate ${\varphi}_i := {\varphi}(x_i) $, for any $\varphi \in C(X) \simeq \R^n$. Given any such  $\varphi \in \R^n$, we may define the weighted Voronoi tessellation corresponding to the points $\{x_i\}_{i =1}^n$ by
\[ \Vor^\varphi_\mu(x_i) = \{ y \in \Rd \colon |x_i - y|^2 - \varphi_i \leq |x_j - y|^2 - \varphi_j   \text{ for all } j = 1, \dots, n\}  .\] 
Using this partition of the domain $\Rd$, we may rewrite $\Phi_\nu$   as 
\begin{align*}
\Phi_{\nu}({\varphi})& =  \sum_{i =1}^n  \left( \varphi_i m_i + \int_{\Vor^{{\varphi}}_\mu(x_i)}  \varphi^*(y) d \nu(y) \right), \\
& =  \sum_{i =1}^n  \left( \varphi_i m_i + \int_{\Vor^{{\varphi}_i}_\mu(x_i)}   |x_i - y|^2 - {\varphi}_i d \nu(y) \right) .
\end{align*}
In this way, we may equivalently  reformulate \eqref{mainproblem} and  \eqref{mainprobv2} by writing, for $\e \geq 0$,
\begin{align}
&\min_{ \Omega \in S_k  }   W_2^2(\mu,1_\Omega/|\Omega|) + \varepsilon \mathcal{U}_m(1_{\Omega}/|\Omega|) \nonumber \\
&\quad = \min_{  \Omega \in S_k     }  \max_{\varphi \in \R^n} \Phi_{1_\Omega / |\Omega|}(\varphi) + \e   \frac{1}{(m-1) |\Omega|^{m-1}}  \nonumber \\
\label{goalnum}  
&\quad = \min_{\Omega \in S_k} \max_{\varphi \in \R^n}  \sum_{i =1}^n \left(  \varphi_i m_i + \frac{1}{|\Omega|}\int_{\Vor^\varphi_\mu(x_i) \cap \Omega} (|y-x_i|^2-\varphi_i ) d y+ \e \frac{1}{(m-1) |\Omega|^{m-1}} \right),  
\end{align}
where in the second equality, we apply equation (\ref{intenshape}) for $\mathcal{U}_m(1_{\Omega}/|\Omega|)$, with the convention that $|\Omega|^0/0 = -\log(|\Omega|)$.

\subsection{Optimization approach}
We now discuss our approach for finding a polygon $\Omega \in S_k$ and dual vector $\varphi \in \R^n$ that solves equation \eqref{goalnum}, via an alternating gradient descent/ascent method in $\Omega$ and $\varphi$.  

First, we consider the gradient ascent step in $\varphi$. Since $\varphi$ only appears in the first two terms of our objective functional, it suffices to compute the variation of $\Phi_{1_\Omega/|\Omega|}(\varphi)$ with respect to $\varphi$, as in M\'erigot's original work  \cite{M_rigot_2011}. For the reader's convenience, we recall the form of this variation, following the presentation of Santambrogio \cite[section 6.4.2]{santambrogio2015optimal}.
 
\begin{proposition} 
\label{p:DerivPhi}
Suppose  $\mu  =  \sum_{i =1}^n \delta_{x_i} m_i $ and $\nu \ll \mathcal{L}^d$. Then, for all $l =1, \dots, n$,
\begin{align} \label{gradformula}
 \frac{\partial \Phi_\nu}{\partial \varphi_l}  = m_l- \nu (\Vor^\varphi_\mu(x_l)) .
\end{align}
\end{proposition}
\begin{proof}
By definition of $\Phi_\nu$ and $\varphi^*$ in equations \eqref{originalphi}-\eqref{varphic}, we may rewrite this as
\begin{align} \label{ddwPhi}
\frac{d}{d \varphi_l} \Phi_\nu(\varphi) &= \frac{d}{d \varphi_l} \sum_{i =1}^n \varphi_i m_i  + \int \left( \inf_{j } |x_j-y|^2 - \varphi_j \right) d \nu(y) \\
&= m_l +  \frac{d}{d \varphi_l} \int \left( \inf_{j } |x_j-y|^2 - \varphi_j \right) d \nu(y) .
\end{align}
To compute the remaining derivative, we follow Santambrogio \cite[p243]{santambrogio2015optimal}. We proceed by decomposing the domain of integration into the following three regions:
\begin{align*}
S_1 &:= \{ y \colon |x_l - y|^2 - \varphi_l < |x_j - y |^2 - \varphi_j  \ \text{ for all } j \neq l \} \\
S_2 &:= \{ y \colon  |x_l - y|^2 - \varphi_l > |x_j - y |^2 - \varphi_j \ \text{ for some } j \neq l \} \\
S_3 &:= \{ y \colon |x_l - y|^2 - \varphi_l = |x_j - y|^2 - \varphi_j = \inf_{j'} |x_{j'} - y|^2 - \varphi_{j'} \ \text{ for some } j \neq l \}. 
\end{align*}
These regions may be interpreted as the region for which $l$ is the unique index that attains the infimum in $\inf_{j } |x_j-y|^2 - \varphi_j$, the region for which $l$ does not attain the infimum, and the region for which $l$ attains the infimum, but not uniquely. Then,
\begin{align*}
\frac{d}{d \varphi_l} \int_{S_1} \left( \inf_{j } |x_j-y|^2 - \varphi_j \right) d \nu(y)  &= \frac{d}{d \varphi_l} \int_{S_1}  \left( |x_l - y|^2 - \varphi_l \right) d \nu(y) = - \nu(S_1) = - \nu (\Vor^\varphi_\mu(x_l)) \\
\frac{d}{d \varphi_l} \int_{S_2} \left( \inf_{j } |x_j-y|^2 - \varphi_j \right) d \nu(y) & = \frac{d}{d \varphi_l} \int_{S_2}  \left( |x_j - y|^2 - \varphi_j \right) d \nu(y) = 0 \\
\frac{d}{d \varphi_l} \int_{S_3}  \left( \inf_{j  } |x_j-y|^2 - \varphi_j \right) d \nu(y) & = \frac{d}{d \varphi_l} 0 = 0 \
\end{align*}
where the last equality follows since $S_3$ has zero Lebesgue measure and $\nu \ll \mathcal{L}^d$. Combining these estimates with equation \eqref{ddwPhi} above gives \eqref{gradformula}. 
\end{proof}

Now, we turn to the gradient descent step in $\Omega$, which we will perform by perturbing the vertices of $\Omega$. 
 In particular, in the two dimensional case, we may assume that
$\Omega$ is a convex $k$-gon with vertices $\{ a_\ell \}_{\ell =1}^k$, ordered in the clockwise direction. We denote the edge connecting vertices $a_\ell$ and $a_{\ell + 1}$ by $E_\ell = [a_\ell, a_{\ell + 1}]$, where we use circular indexing, and we denote the outward normal vector to $E_\ell$ by  $n_\ell$. 
The following theorem will allow us to compute the gradient of the objective function in \eqref{goalnum} with respect to the vertex $a_\ell$, $\ell =1,\dots, k$. 

\begin{proposition} 
\label{p:DerivOmega} Let $g$ be a continuous function on $\R^2$ and consider the shape function $J(\Omega) = \int_\Omega g(x) \ dx $. 
Then the gradient of $J$ with respect to the vertex $a_\ell$, $\ell =1 , \dots, k $, is given by
\begin{align}
\label{eq:derivative}
\nabla_{a_\ell} J = L_{\ell} n_{\ell - 1} + R_{\ell} n_\ell,  
\end{align}
where 
$$
L_{\ell} = \int_{E_{\ell-1}} g(x) F_{\ell-1,\ell}(x) \ dx  
\qquad \textrm{and} \qquad 
R_{\ell} = \int_{E_\ell} g(x) F_{\ell+1, \ell}(x)  \ dx
$$
are constants computed on the intervals to the left and right of vertex $a_\ell$ and $ F_{\ell-1,\ell}$ is the affine function defined in \eqref{e:F}. 
\end{proposition}
\begin{proof}
In general, if a domain $\Omega$ is perturbed in the direction of a vector field, $v$, the shape derivative of $J(\Omega)$ is given by 
$$
\delta J(\Omega) \cdot v = \int_{\partial \Omega} g(x)  \langle v,  n\rangle \ dx,
$$
where $n = n(x)$ denotes the outward normal vector at $x \in \partial \Omega$. 
Let $F_{j,k}\colon \mathbb R^2 \to \mathbb R$ be the affine function 
\begin{equation} \label{e:F}
F_{j,k}(x) = \frac{\langle x - a_j, a_k - a_j \rangle}{ \|a_k - a_j\|^2 }, 
\end{equation}
which satisfies 
$F_{j,k}(a_j) = 0$ and 
$F_{j,k}(a_k) = 1$. 
When we perturb the vertex $a_\ell \mapsto a_\ell + \alpha$ for $\alpha \in \mathbb R^2$, 
this perturbs the two adjacent edges 
$E_{\ell-1}$ and $E_\ell$. 
In particular, the velocity field, $v$, on these edges has the form 
$$
v(x) = 
\begin{cases}
F_{\ell-1,\ell}(x) \alpha, 
& x \in E_{\ell-1} \\ 
F_{\ell+1,\ell}(x) \alpha, 
& x \in  E_\ell. 
\end{cases}
$$
We compute 
$$
\langle \nabla_{a_\ell} J, \alpha \rangle
=  \int_{E_{\ell-1}} g(x) F_{\ell-1,\ell}(x) \langle n_{\ell-1}, \alpha\rangle \ dx   
\ + \ 
 \int_{E_\ell} g(x) F_{\ell+1, \ell}(x) \langle n_\ell, \alpha\rangle  \ dx. 
$$
from which \eqref{eq:derivative} follows.
\end{proof}

Using the above result, we may now compute the gradient descent step in (\ref{goalnum}) by rewriting the second two terms in the objective function as
\begin{align} \label{secondpartobj}
\sum_{i=1}^n \frac{1}{|\Omega|}\int_{\Vor^\varphi_\mu(x_i) \cap \Omega} (|y-x_i|^2-\varphi_i ) d y+ \e \frac{1}{(m-1) |\Omega|^{m-1}} = \frac{1}{|\Omega|} \int_\Omega f_\varphi(y) dy + \e \frac{1}{(m-1) |\Omega|^{m-1}} ,
\end{align}
for
\begin{align*}
  f_\varphi(y) =  \sum_{i =1}^n 1_{{\Vor_\mu^{\varphi}(x_i)}}(y) \left(|y-x_i|^2-\varphi_i \right) ,
\end{align*}
and taking the gradient of (\ref{secondpartobj}) with respect to each vertex by applying Proposition \ref{p:DerivOmega}, the formula for $|\Omega|$ in terms of its vertices,  and the quotient rule.

In summary, our method aims to optimize the saddle point problem (\ref{goalnum}) by alternating gradient ascent steps according to (\ref{gradformula}) and gradient descent steps of (\ref{secondpartobj}), for each vertex of $\Omega$; see Algorithm \ref{alg1}. While our algorithm performs well in practice (see section \ref{s:NumExp}) we leave convergence analysis to future work. In particular, as illustrated in Figure \ref{f:Landscape}, we generically expect the energy landscape to have multiple stationary points, so that further hypotheses would be needed to ensure our method converges to a global optimizer.
\begin{algorithm}[t!]
\DontPrintSemicolon
 \KwIn{A distribution $\mu = \sum_{i=1}^n \delta_{x_i}m_i$, 
 an initial dual variable $\varphi^0$, a number of vertices $k \geq 3$, 
 an initial polygon  $\Omega^0 = \conv \left(\{a^0_\ell\}_{\ell \in [n]} \right)$, regularization parameters $\e \geq 0$ and $m \geq 1$,
 step sizes $\tau_1,\tau_2>0$, and tolerances $\delta_1, \delta_2>0$. }
 \KwOut{ A polygon $\Omega^n$ that is a candidate approximate solution for \eqref{mainproblem} or  \eqref{mainprobv2}.}
 \For{$s \in \mathbb N$}{
{\bf 1. Gradient ascent in $\varphi$.} 
Set $\bar \varphi^{s,1} = \varphi^s $\;
 \For{$r\in \mathbb N$}{
Find the power diagram at $\bar \varphi^{s,r}$, compute  $ |\Vor^\varphi_\mu(x_i) \cap \Omega|/|\Omega|$ and compute
$$\bar \varphi^{s,r+1}_i = \bar \varphi^{s,r}_i+ \tau_1(m_i -  |\Vor^\varphi_\mu(x_i) \cap \Omega|/|\Omega|)  $$\;
If $\|\bar \varphi^{s,r+1} - \bar \varphi^{s,r}\|_1 \leq \delta_1$, STOP \;  
}
Set $\varphi^{s+1} = \bar \varphi^{s,r}$\;
{\bf 2. Gradient descent in $\Omega$.} Update the vertices $\{a^s_\ell\}_{\ell = 1}^k$ of $\Omega$ by 
$$a_{\ell}^{s+1} = a_{\ell}^s - \tau_2 \nabla_{a_\ell}\left( \frac{1}{|\Omega|} \int_\Omega f_{\varphi^{s+1}}(y) dy +  \e \frac{1}{(m-1) |\Omega|^{m-1}} \right).$$
If $\|\alpha^{s+1}-\alpha^s\|_\infty <\delta_2$, STOP\;  
 }
\caption{An alternating method to approximately solve \eqref{mainproblem} and \eqref{mainprobv2}. } 
\label{alg1}
\end{algorithm}

\section{Numerical experiments} 
\label{s:NumExp}

In this section, we illustrate the performance of the algorithm  via several computational experiments.  In all numerical experiments, we set $\tau_1 = 0.5$, $\tau_2 = 0.1$, $\delta_1 = 10^{-3}$, and $\delta_2 = 10^{-5}$. We choose the initial dual variable $\varphi^0$ to be a zero vector with dimension $n$ in all experiments.

\subsection{Example 1: Uniform distribution in a disk} \label{sec:uniform}
We first consider the case where the measure $\mu$ is the uniform probability measure on the unit disk.
The measure $\mu$ is approximated as follows. We first generate $10, \ 000$ data points, where  each data point $x_i = (x^1_i, x^2_i)$ is randomly generated by 
\[\begin{cases} x^1 = \sqrt{r} \cos (2\pi\theta) , \\
x^2 = \sqrt{r} \sin (2\pi\theta) ,
\end{cases}
\]
with $r$ and $\theta$   drawn from uniform distributions on $[0,1]$. We generate a uniform $15\times 15$ discretization of the square $[x_{min},x_{max}]\times[y_{min},y_{max}]$ where $x_{min/max},y_{min/max}$ are the smallest or largest value in $x,y$, respectively. We then set the center of the $i$-th cell as $x_i$ and the ratio of random points located in this cell as $m_i$ to have the approximate measure $ \mu_n = \sum_{i=1}^n m_i \delta_{x_i}$ by dropping those cells where $m_i=0$. In this experiment, $n = 172$.
We seek to solve (\ref{mainproblem}), with no regularization ($\e = 0$), and $k =3$. 

\begin{figure}[ht!]
\includegraphics[width =0.18\textwidth,clip, trim = 2cm 0cm 3cm 0cm]{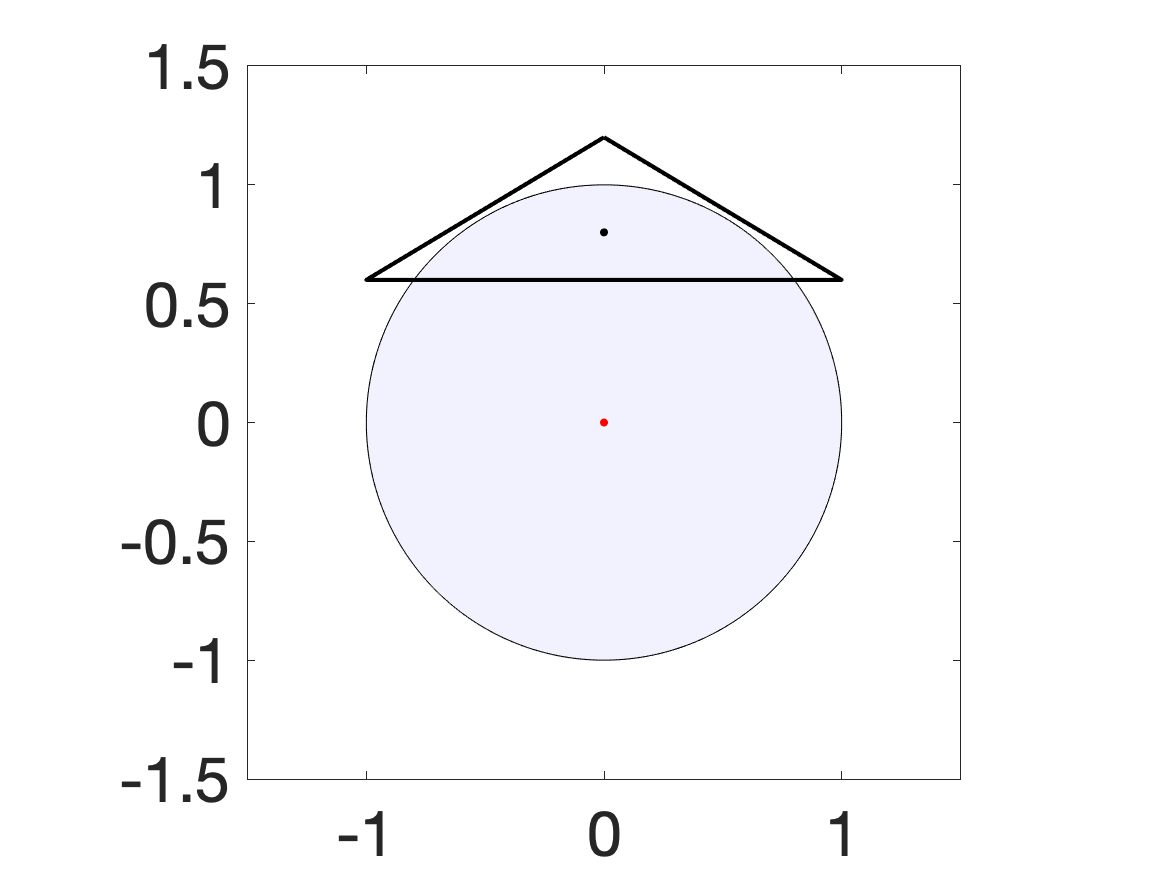}
\includegraphics[width =0.18\textwidth,clip, trim = 2cm 0cm 3cm 0cm]{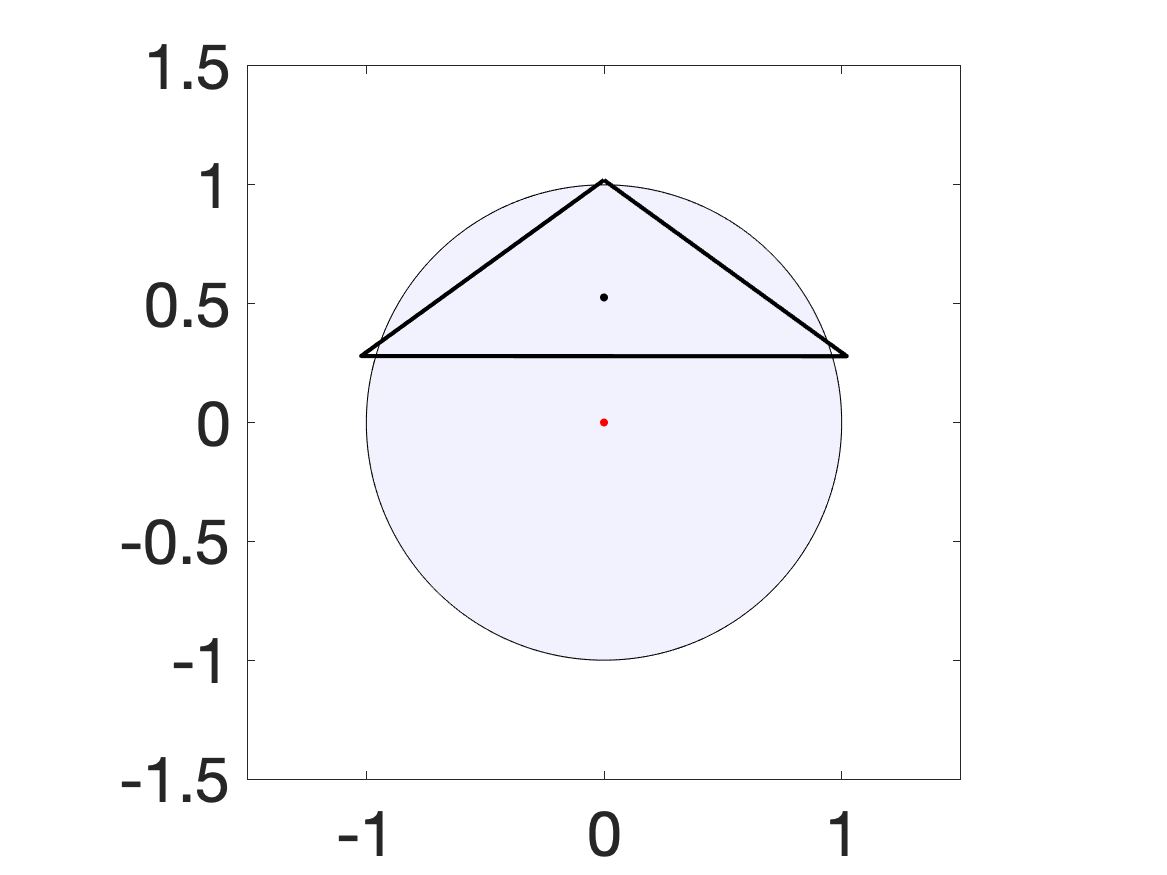}
\includegraphics[width =0.18\textwidth,clip, trim = 2cm 0cm 3cm 0cm]{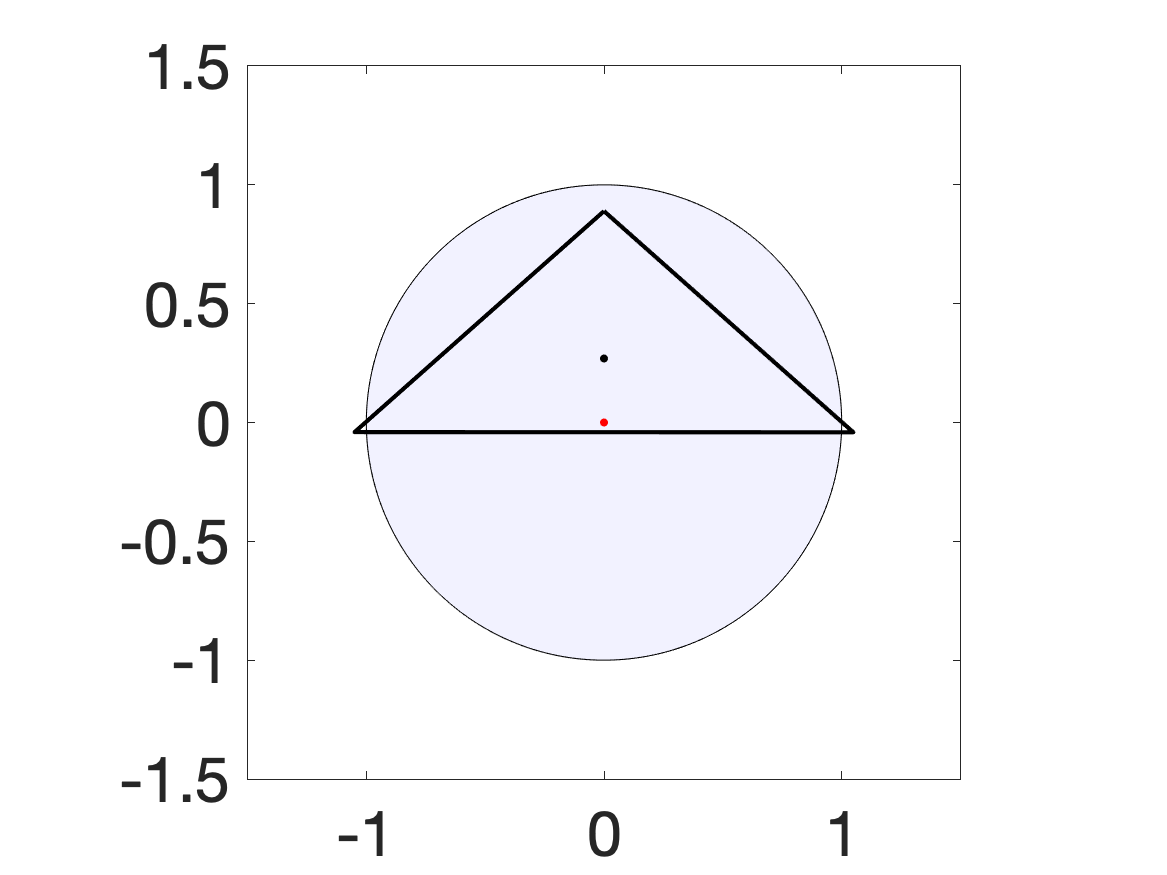}
\includegraphics[width =0.18\textwidth,clip, trim = 2cm 0cm 3cm 0cm]{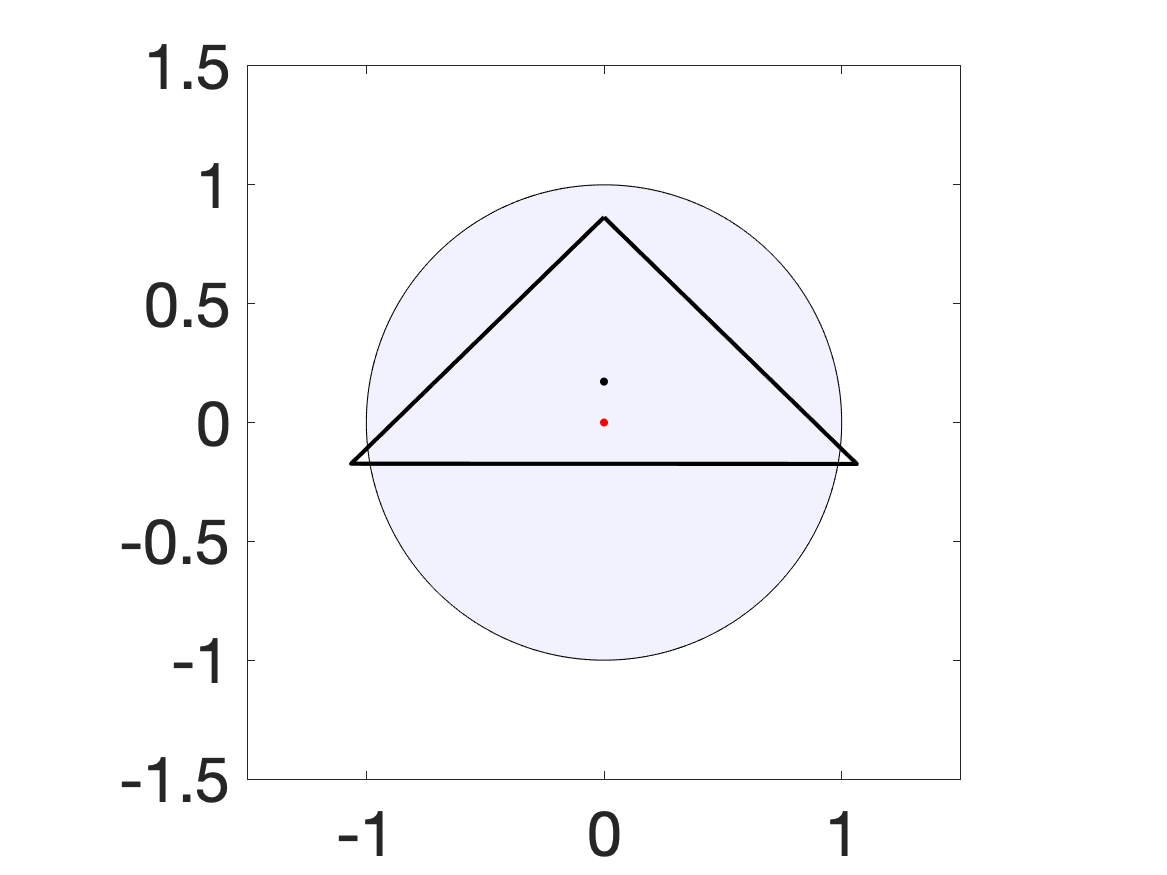}
\includegraphics[width =0.18\textwidth,clip, trim = 2cm 0cm 3cm 0cm]{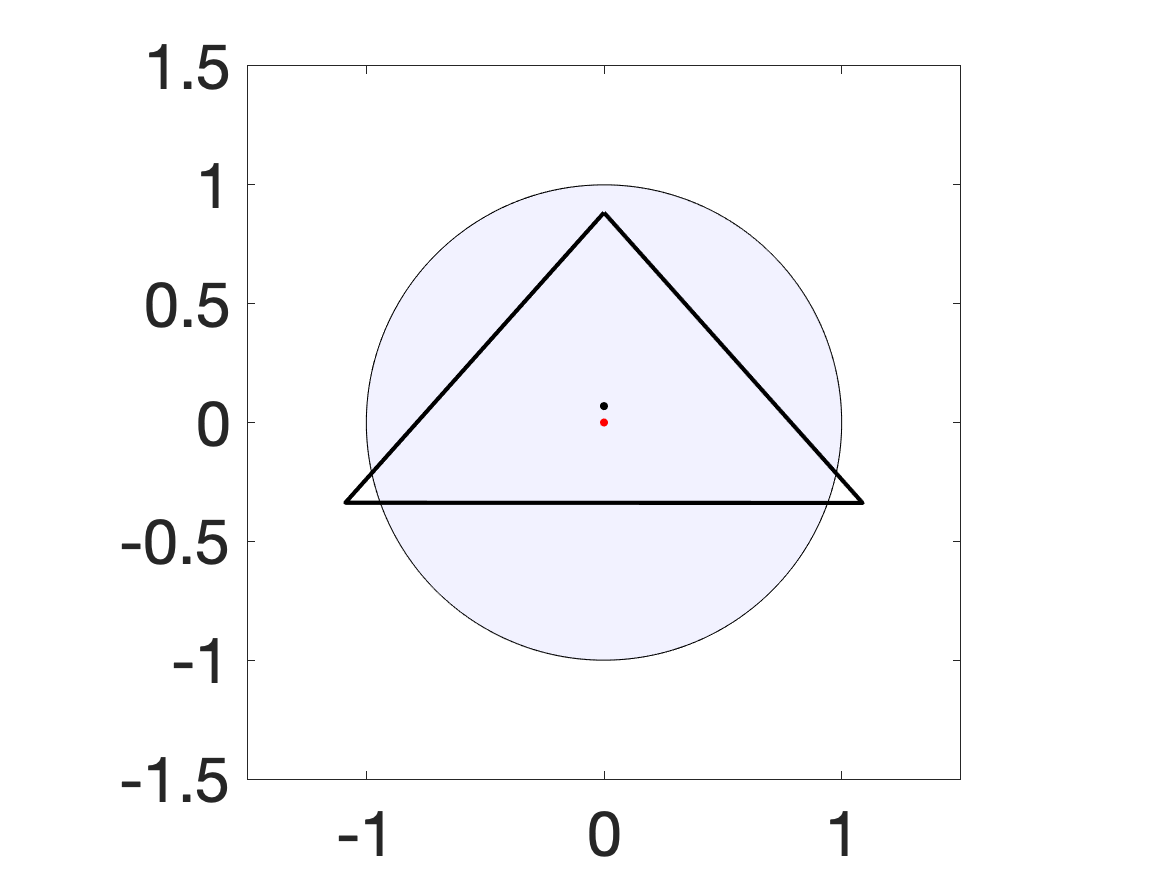}
\includegraphics[width =0.18\textwidth,clip, trim = 2cm 0cm 3cm 0cm]{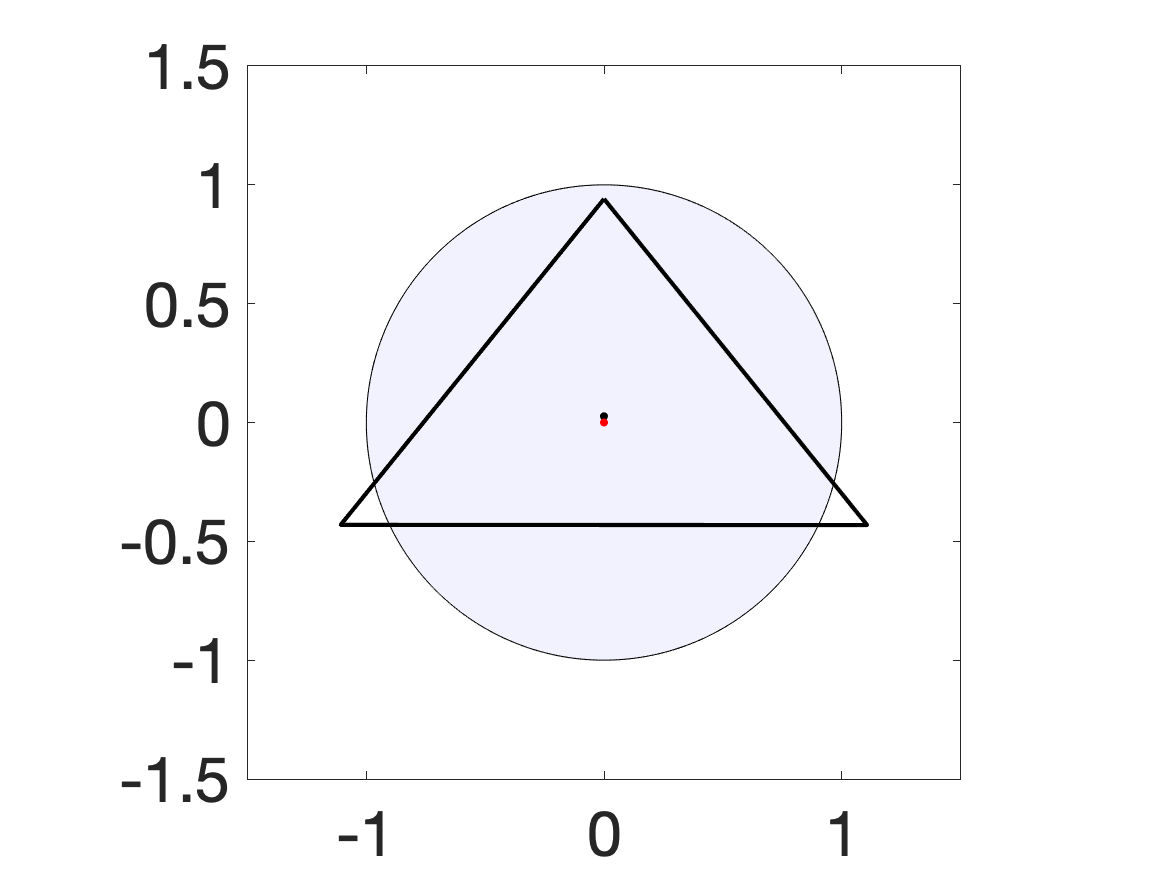}
\includegraphics[width =0.18\textwidth,clip, trim = 2cm 0cm 3cm 0cm]{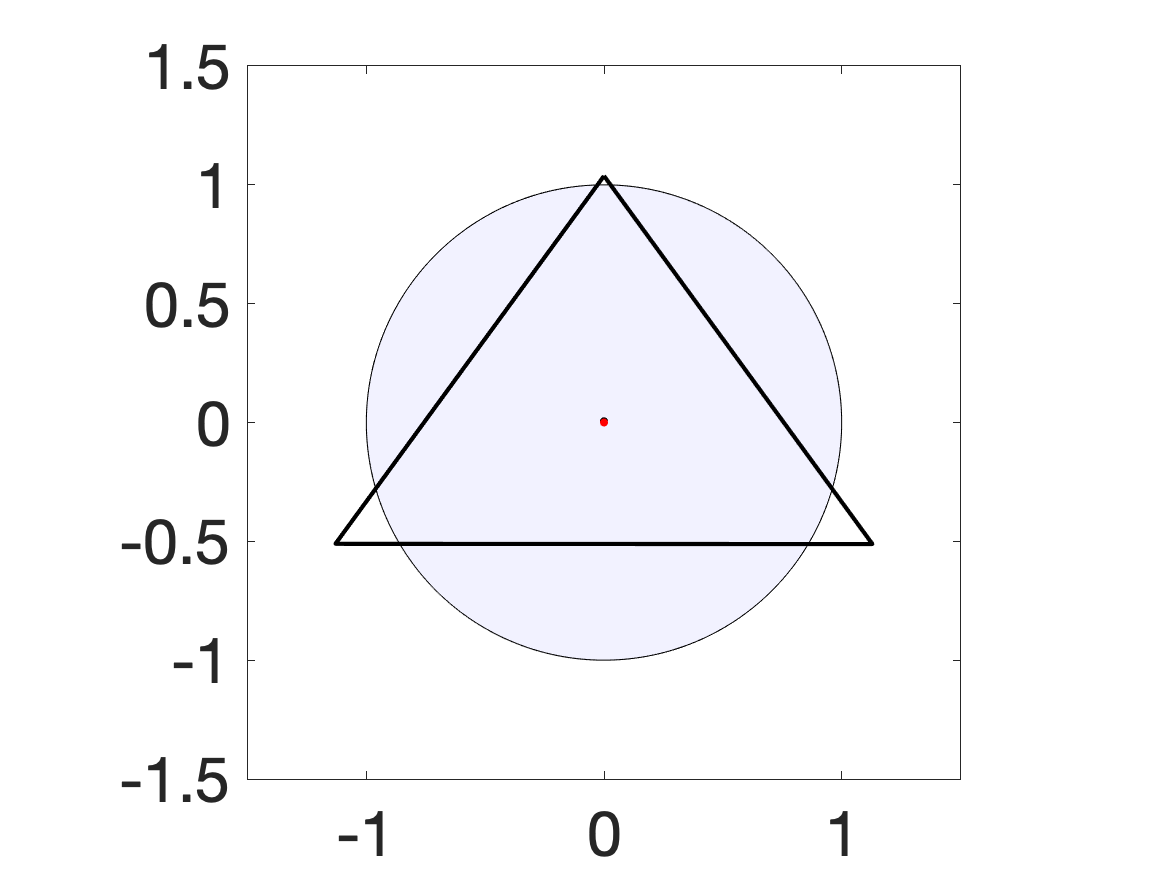}
\includegraphics[width =0.18\textwidth,clip, trim = 2cm 0cm 3cm 0cm]{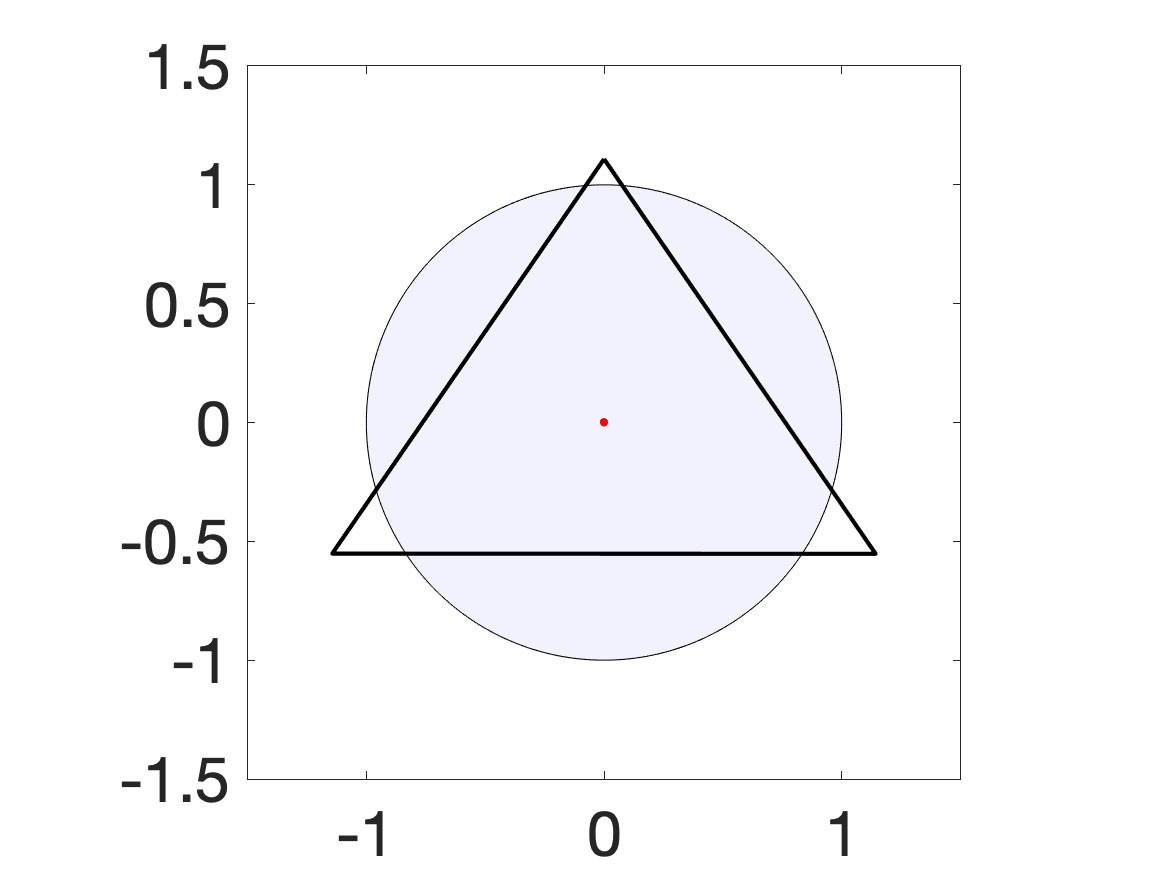}
\includegraphics[width =0.18\textwidth,clip, trim = 2cm 0cm 3cm 0cm]{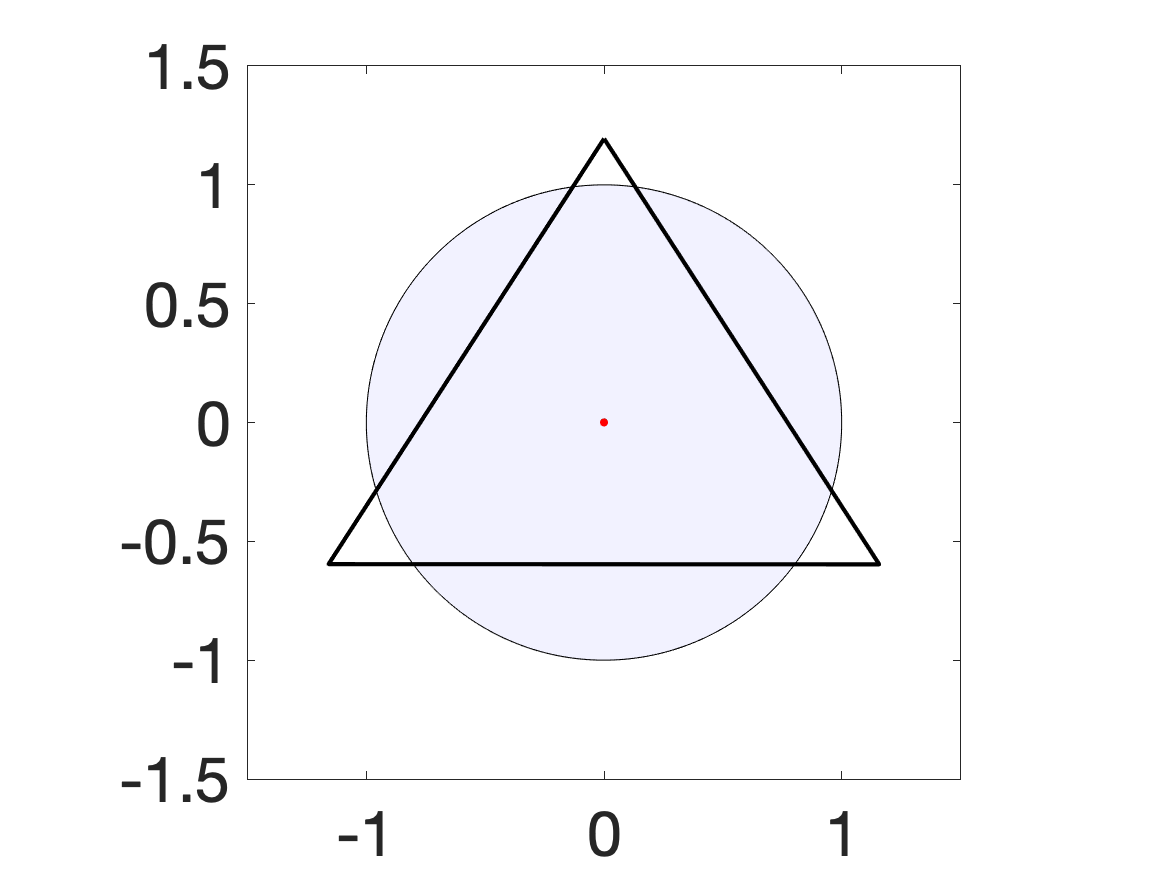}
\includegraphics[width =0.18\textwidth,clip, trim = 2cm 0cm 3cm 0cm]{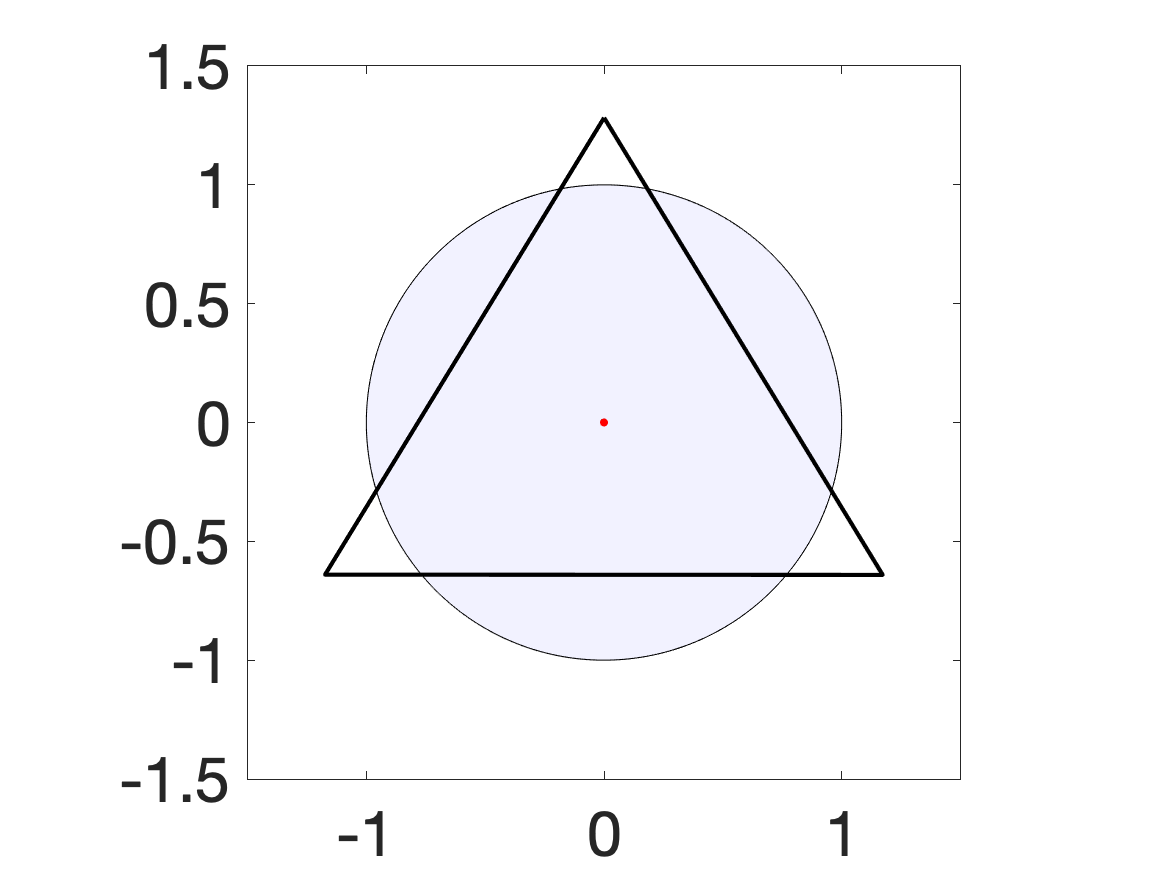}
\caption{Snapshots of the evolution of the triangle computed by Algorithm~\ref{alg1} until convergence. See Subsection~\ref{sec:uniform}. } \label{fig:1}
\end{figure}

Figure~\ref{fig:1} illustrates the  evolution of the triangle computed using Algorithm~\ref{alg1}. We observe that the triangle evolves gradually to a regular triangle centered at the origin.  Note that all three vertices lie outside of the disk. This is in contrast to  the classical archetypal analysis problem \eqref{e:arch}, where the solution is an inscribed triangle, as depicted in Figure~\ref{fig:aaresults}. For more detail see \cite[Proposition 3.1]{Yiming2020}. 

\begin{figure}[ht!]
\includegraphics[width = 0.4\textwidth,clip, trim = 2cm 0cm 3cm 0cm]{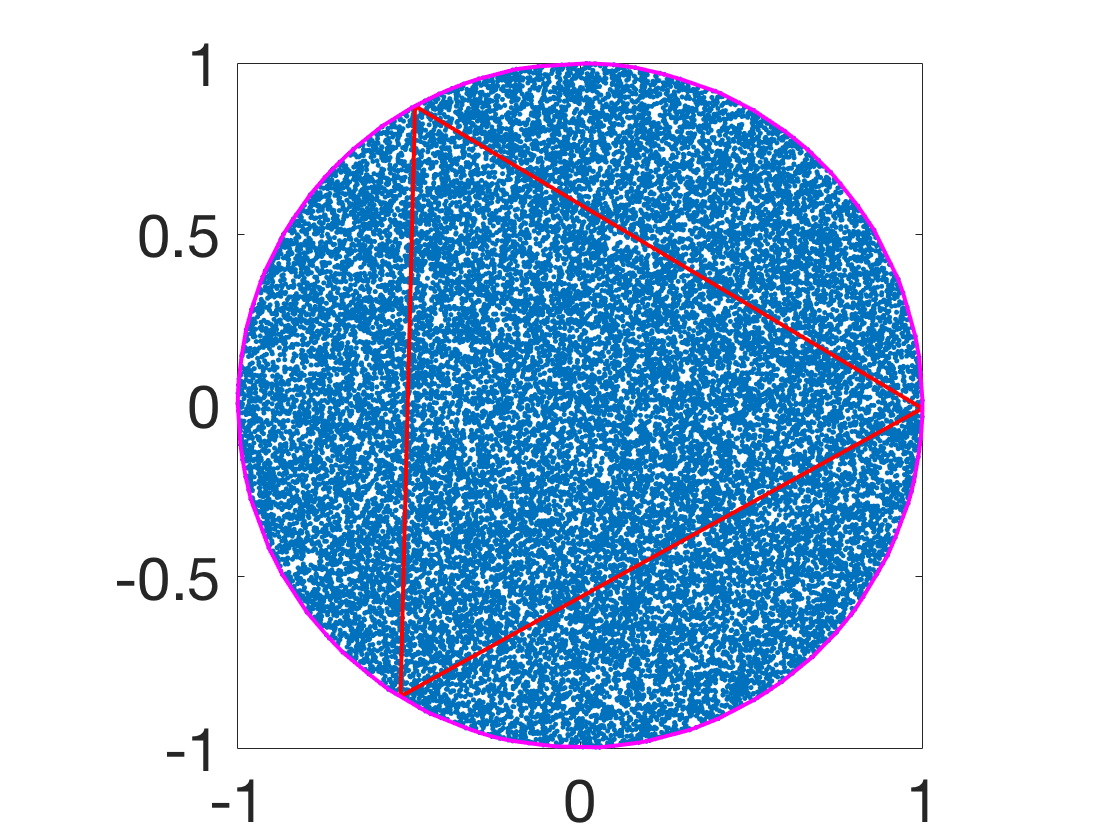}
\caption{Figure 3 of \cite{Yiming2020}: A minimizer for 30, 033 random data points. The red curve is the boundary of the convex hull of the archetype points and the magenta curve is the boundary of the convex hull of random data points. See Subsection~\ref{sec:uniform}.} \label{fig:aaresults}
\end{figure}

\begin{figure}[ht!]
\includegraphics[width = 0.22\textwidth,clip, trim = 2cm 0cm 3cm 0cm]{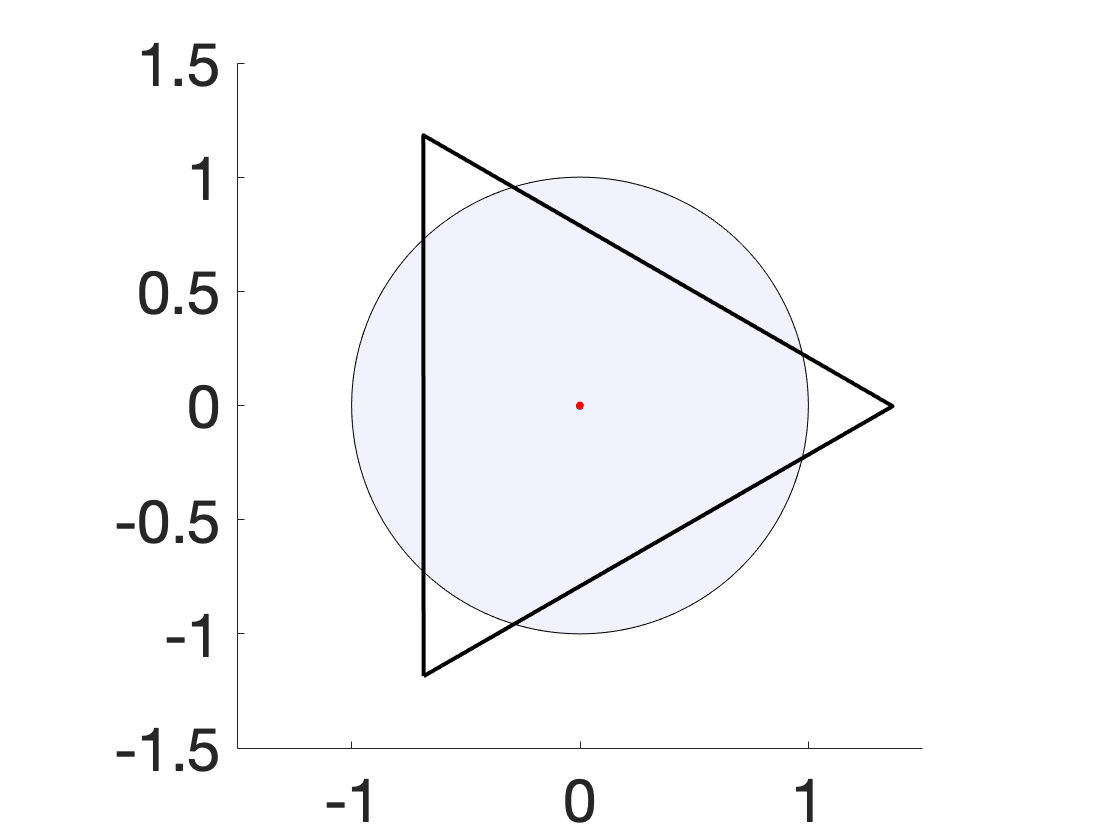}
\includegraphics[width = 0.22\textwidth,clip, trim = 2cm 0cm 3cm 0cm]{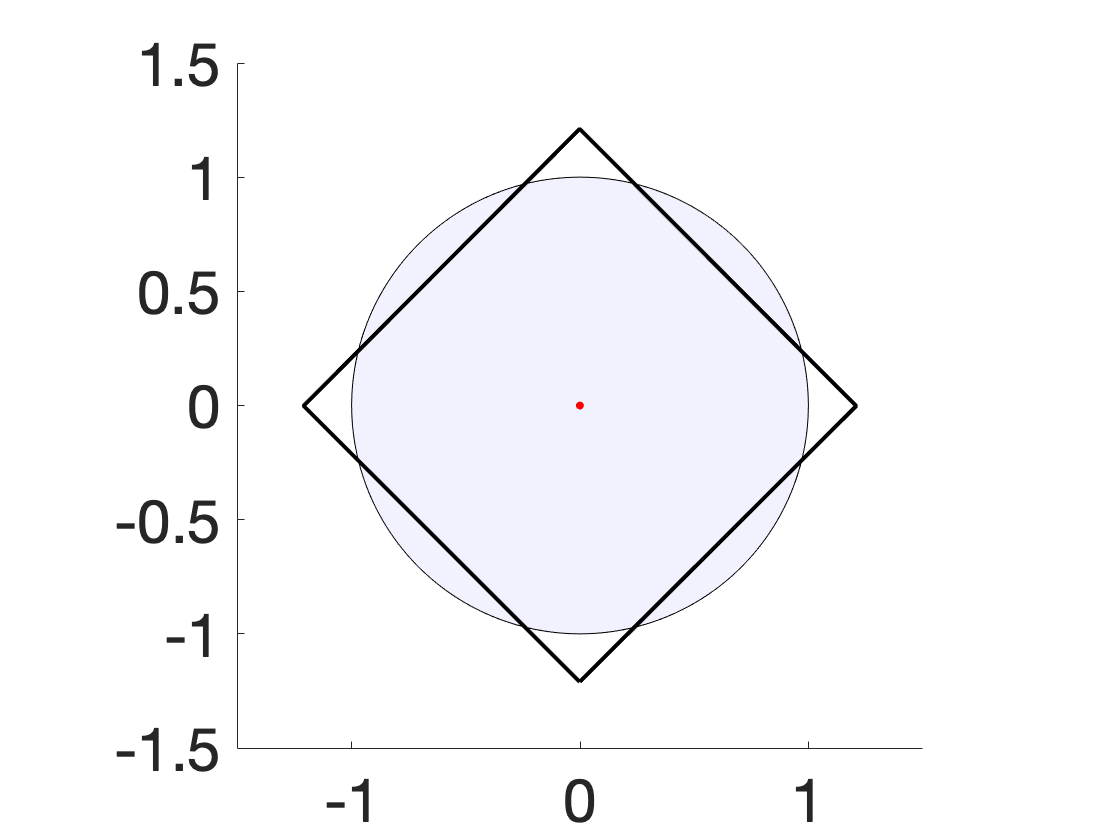}
\includegraphics[width = 0.22\textwidth,clip, trim = 2cm 0cm 3cm 0cm]{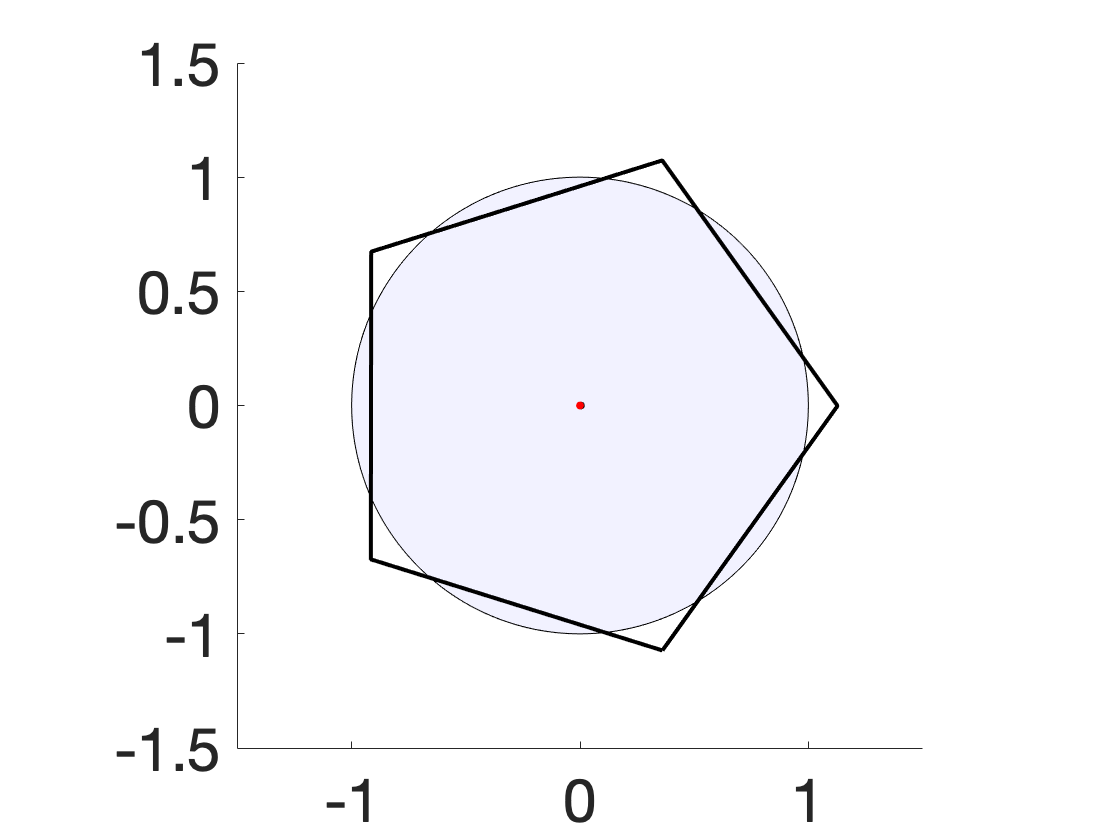}
\includegraphics[width = 0.22\textwidth,clip, trim = 2cm 0cm 3cm 0cm]{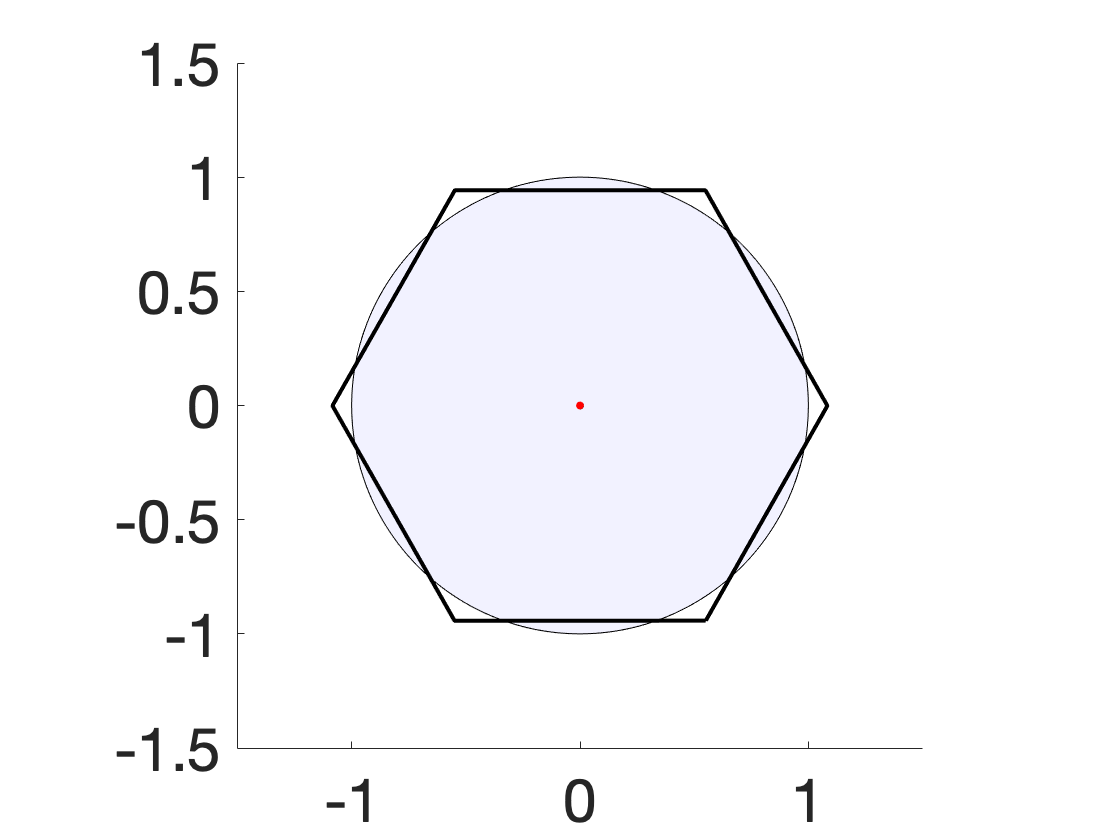} \\
\includegraphics[width = 0.22\textwidth,clip, trim = 2cm 0cm 3cm 0cm]{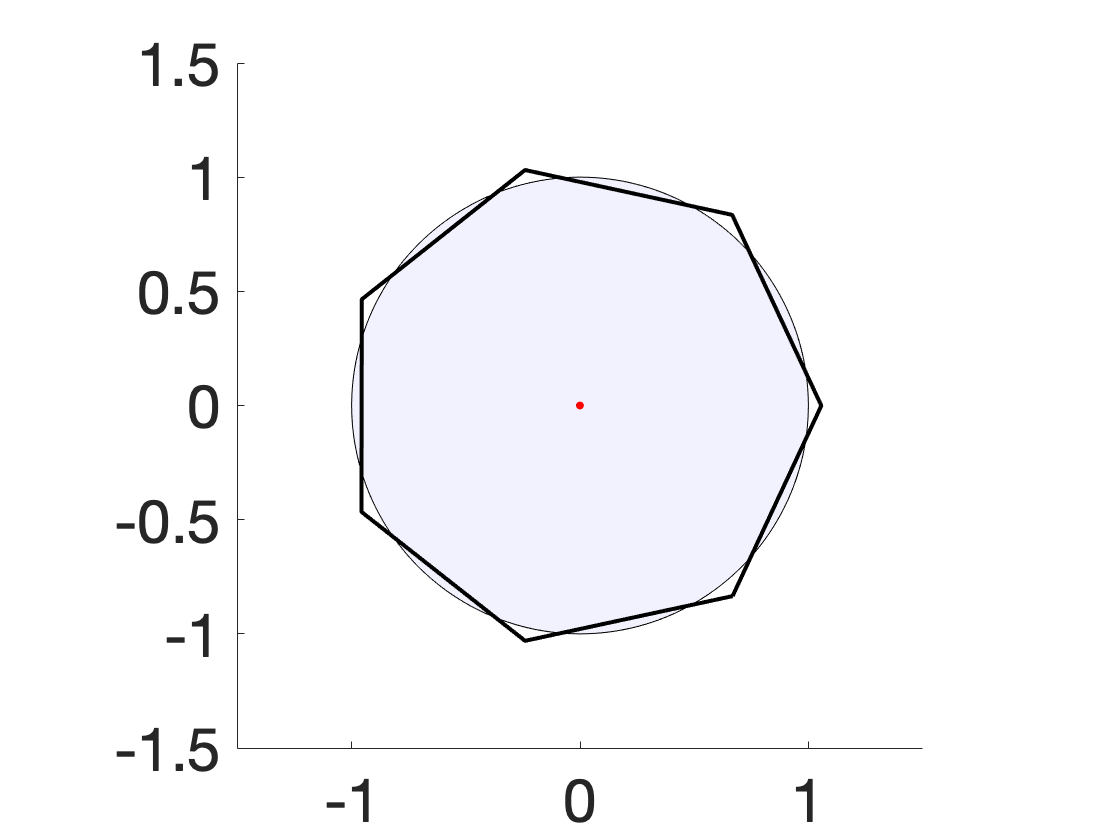}
\includegraphics[width = 0.22\textwidth,clip, trim = 2cm 0cm 3cm 0cm]{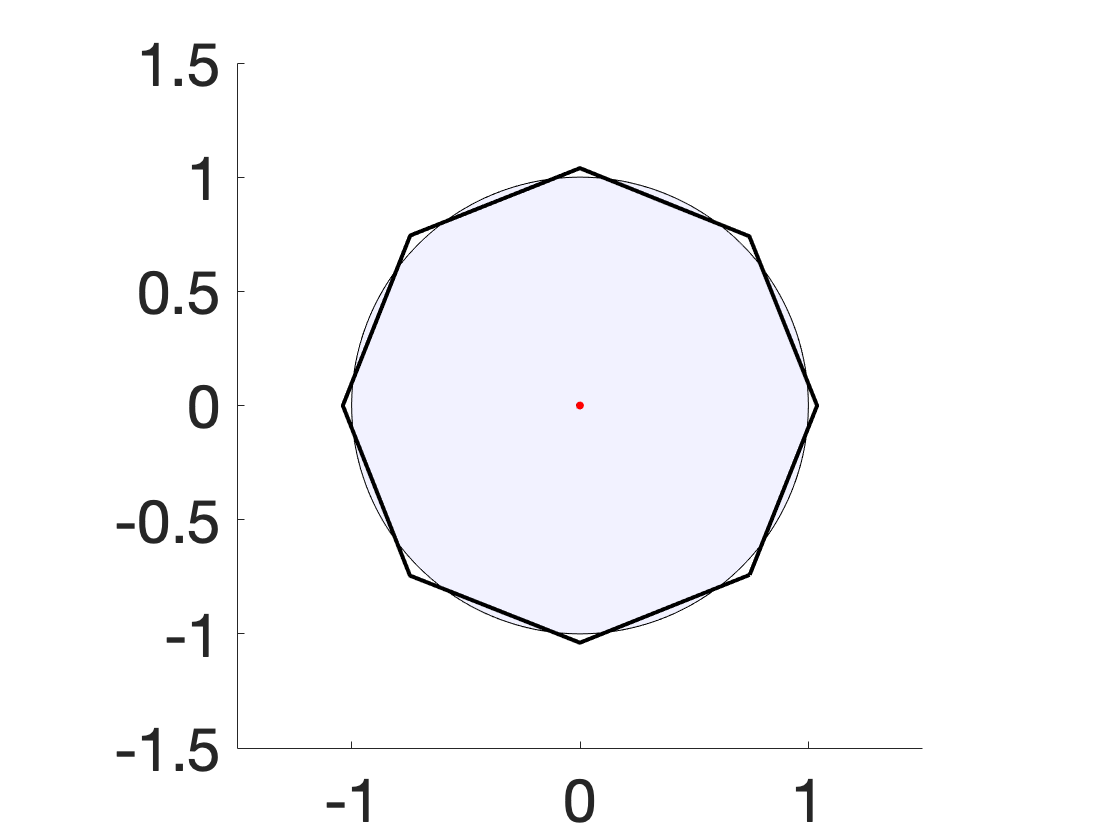}
\includegraphics[width = 0.22\textwidth,clip, trim = 2cm 0cm 3cm 0cm]{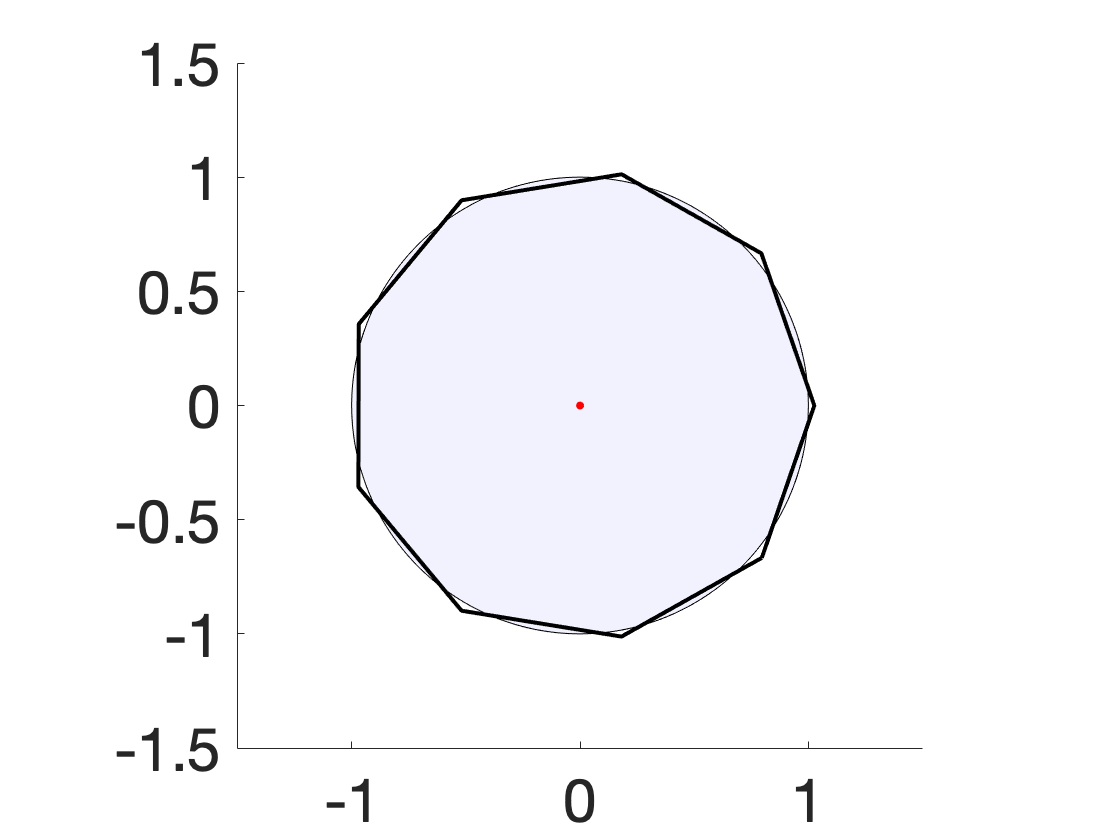}
\includegraphics[width = 0.22\textwidth,clip, trim = 2cm 0cm 3cm 0cm]{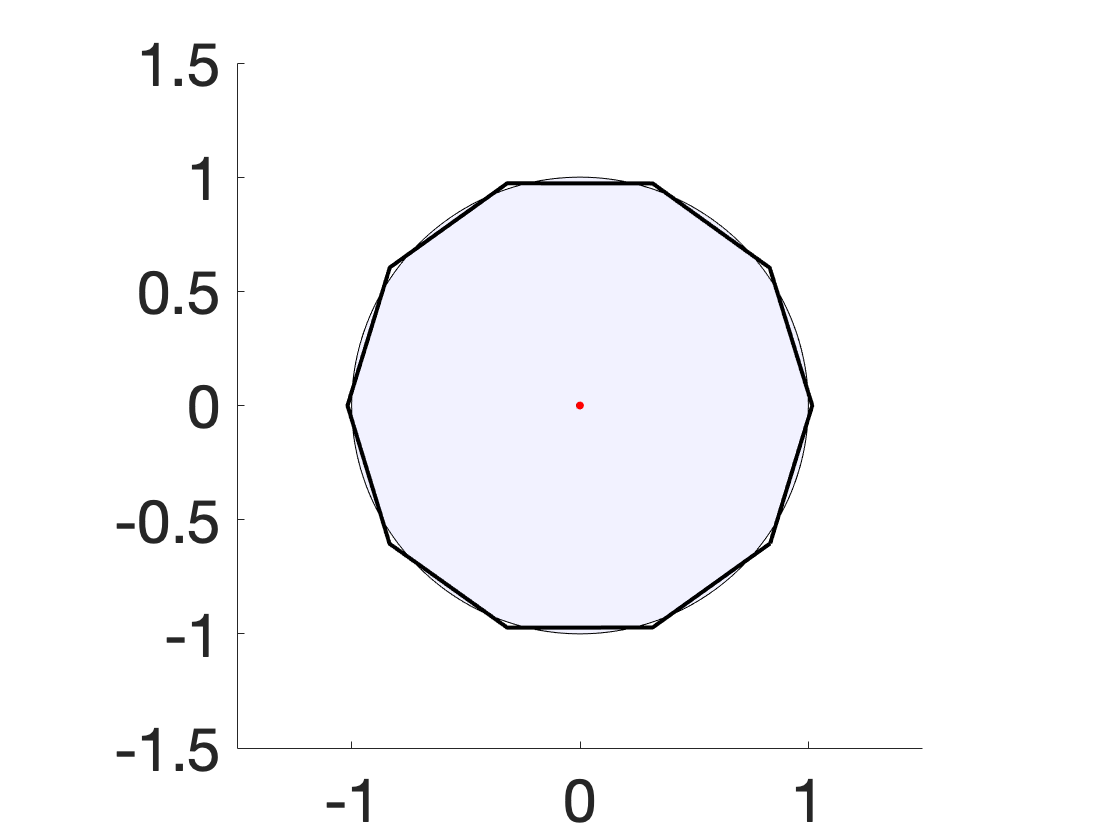}\\
\ \\
\includegraphics[width = 0.8\textwidth,clip,trim = 0cm 0cm 0cm 0cm]{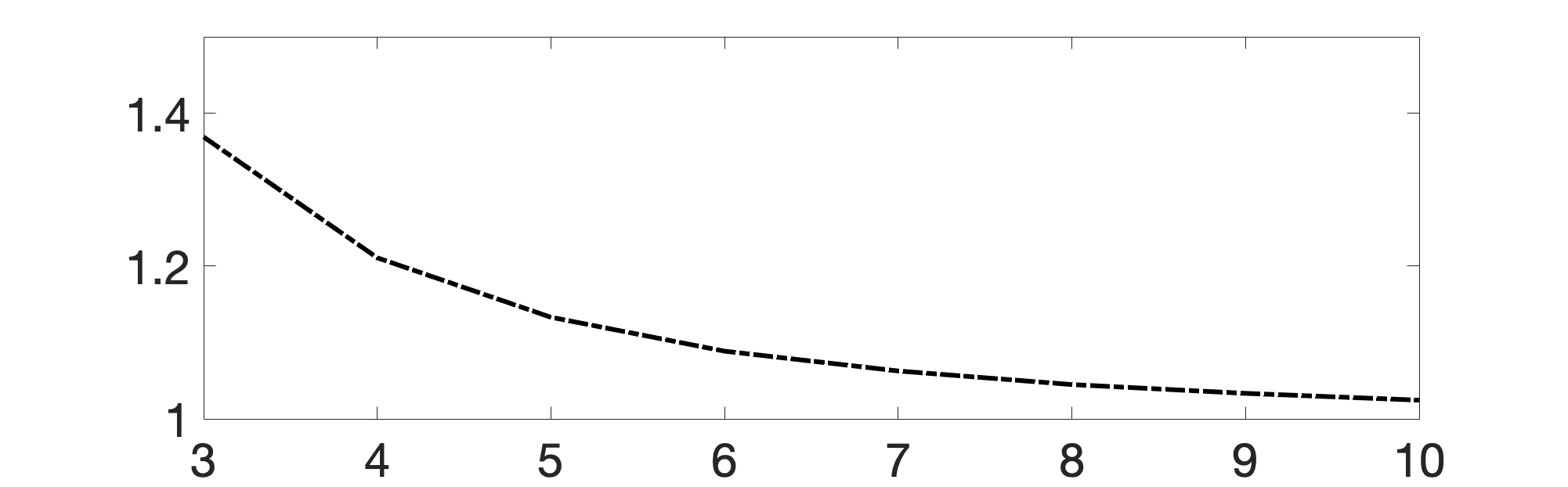}
\caption{Solutions obtained from Algorithm~\ref{alg1} for different numbers of vertices, $k = 3,4,  \dots, 10$, where $\mu$ is the uniform probability measure on the unit disk. {\bf Bottom:} The change of radius of the optimal $k$-gon as the number of vertices increases. See Subsection~\ref{sec:uniform}.} \label{fig:2}
\end{figure}

In Figure~\ref{fig:2}, we display the results of  (\ref{mainproblem}) for $k = 3, 4, \ldots, 10$. In all experiments, the optimal solutions appear to be regular polygons. It is known that the solution for the classical archetypal analysis is a regular $k$-gon \cite[Proposition 3.1]{Yiming2020}  and we conjecture that this holds true for the solution to (\ref{mainproblem}) as well. As $k$ increases, the $k-$gon becomes closer to a disk. In Figure~\ref{fig:2}, we also plot the change of the radius of the polygons as $k$ increases, and the asymptotic behavior is consistent with the fact that the polygon converges to a disk as $k \rightarrow \infty$.

\subsection{Example 2: Normal distribution} \label{sec:normal}
In this experiment, we study the behavior of the solution of (\ref{mainproblem}) when $\mu$ is a normal distribution $\mathcal N\left(0,  I\right)$, where  $I$ is the identity matrix.  The approximate measure $\mu_n = \sum_{i=1}^n m_i \delta_{x_i}$ is generated in a similar way as those described in Section~\ref{sec:uniform}, except that the empirical measure is directly generated by a 2-dimensional normal distribution with $n = 92$.
In Figure~\ref{fig:normal1}, we plot the solution for the cases when $k=3, 4$. We observe that the solution is in the interior of the convex hull of data points randomly generated from the normal distribution. 
For the case $k=4$, we observe that the solution is a square with a side length $\approx 3.3855$, which approximately coincides with the length of the optimal interval ($\approx 3.3862$) for the analogous 1 dimensional problem; see equation  \eqref{e:1dSol}.

\begin{figure}
\includegraphics[width = 0.45\textwidth,clip, trim = 4cm 0cm 4cm 0cm]{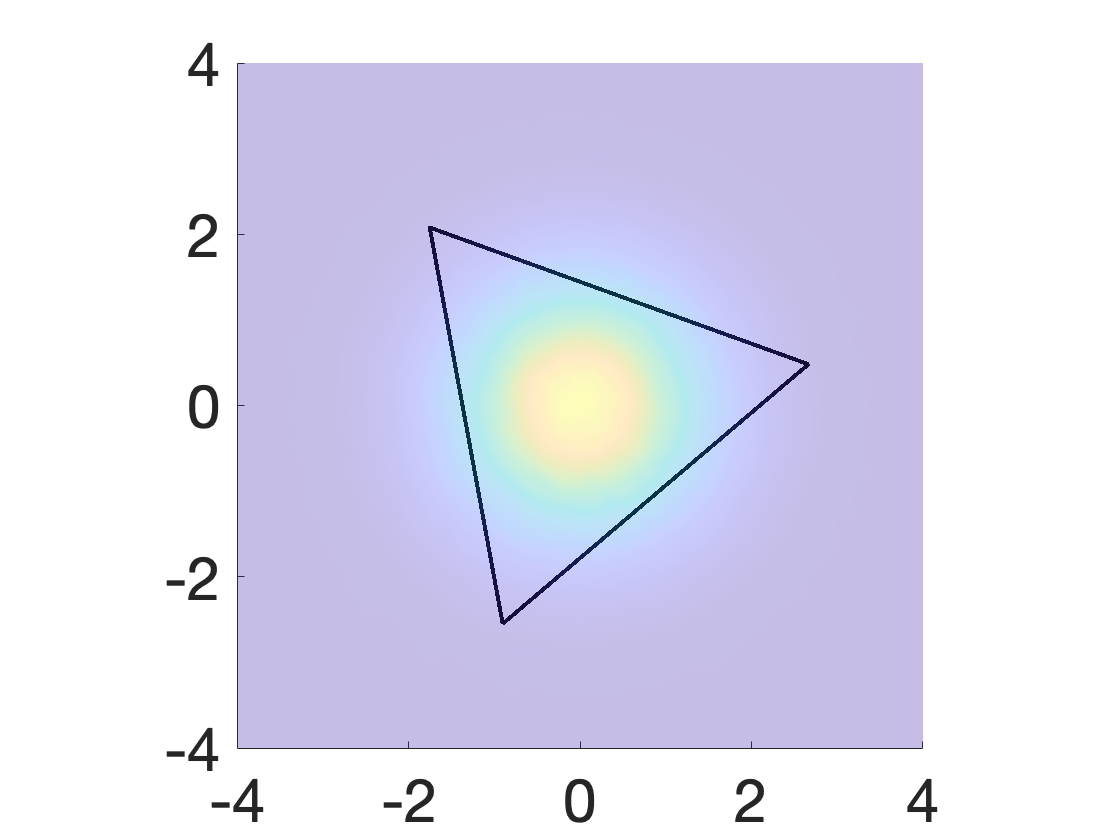} \ \ \ \
\includegraphics[width = 0.45\textwidth,clip, trim = 4cm 0cm 4cm 0cm]{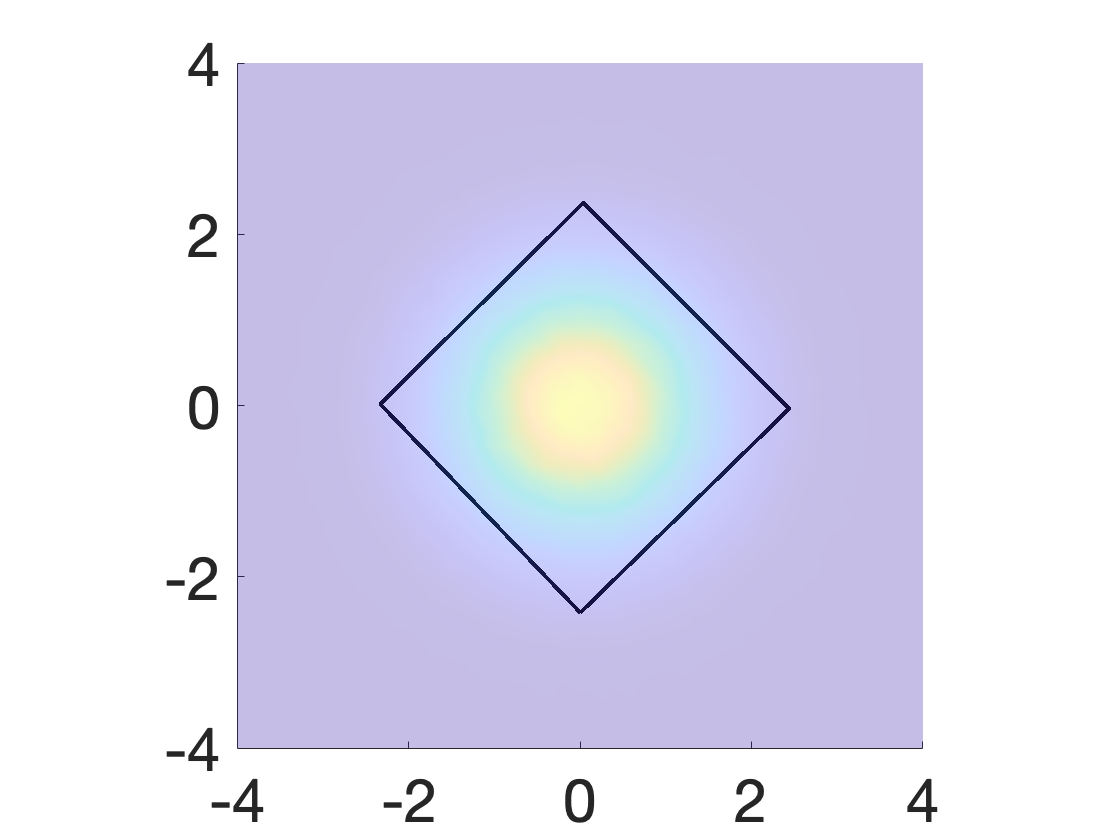}
\caption{Solutions for $k = 3$ and $k=4$, when $\mu$ is a normal distribution $\mathcal N\left(0, I\right)$.} \label{fig:normal1}
\end{figure}

\subsection{Example 3: Sensitivity to $\varepsilon$ in \eqref{mainprobv2}} 
\label{s:EpsSensitivity} 
In this experiment, we consider the solution of the problem  \eqref{mainprobv2} with 
\[ \mu = \mathcal N\left(0,  \Sigma\right) , \quad \Sigma = \begin{bmatrix} 5 & 0 \\ 0 &1 \end{bmatrix}\ , \]
 approximating $\mu$ as in Example 2 with $n =90$.
In Figure~\ref{fig:normal2}, we study how the solution changes as the value of $\varepsilon$ increases from $0$ to $50$. In all cases, we set $m=2$. We observe that as $\varepsilon$ increases, the quadrilateral becomes larger and  closer to a square, that is,  the ratio between the long side and the short side decreases, as shown in Figure~\ref{fig:normal3}. This is consistent with the motivation on introducing the regularization term, as it rewards larger areas.

\begin{figure}
\includegraphics[width = 0.9\textwidth,clip, trim = 4cm 0cm 4cm 0cm]{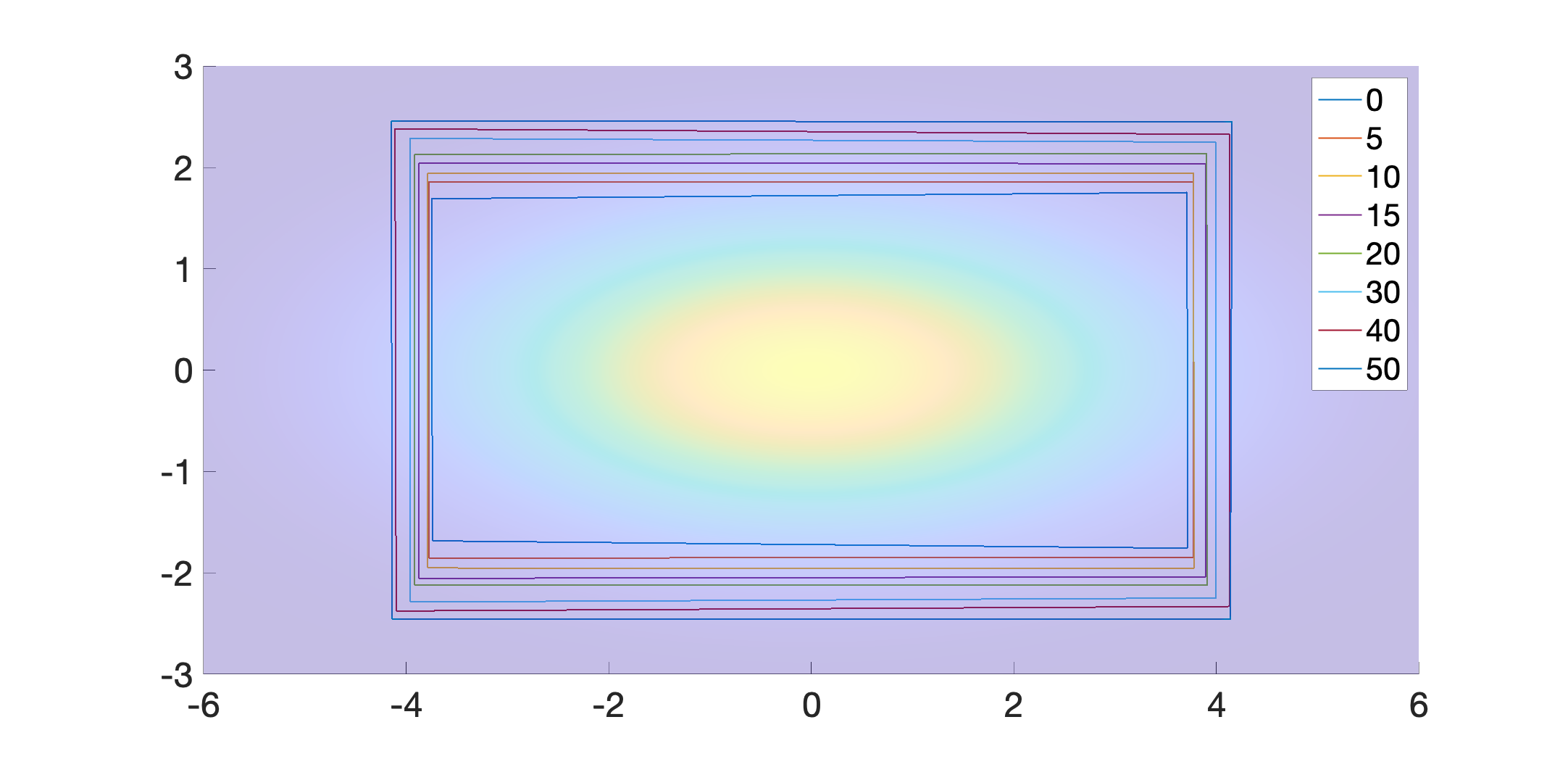}
\caption{Solutions for $k=4$ with target distribution $\mathcal{N} \left(0,  \begin{bmatrix} 5 & 0 \\ 0 &1 \end{bmatrix}\right)$ and different choices of $\varepsilon \in \{0, 5, 10, 15, 20, 30, 40, 50 \}$. } 
\label{fig:normal2}
\end{figure}

\begin{figure}
\includegraphics[width = 0.8\textwidth,clip, trim = 0cm 0cm 0cm 0cm]{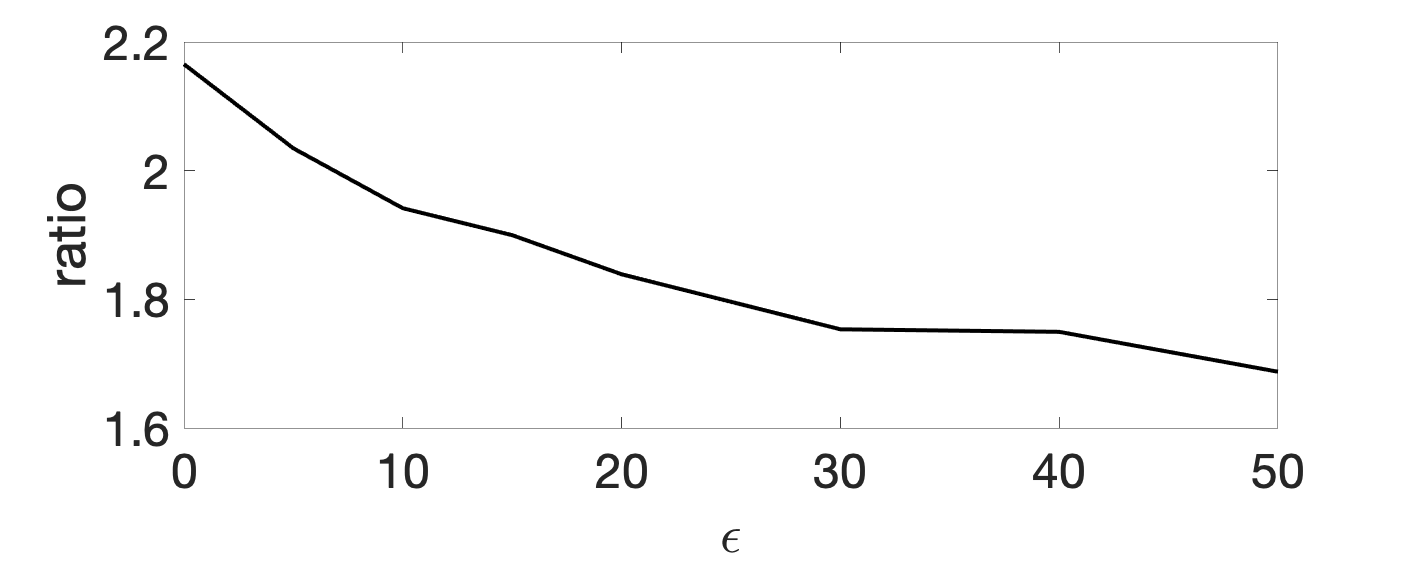}
\caption{Change of ratio between the longest side and shortest side of the obtained quadrilateral with $\varepsilon$.} \label{fig:normal3}
\end{figure}

\subsection{Example 4: Non-convexity of the energy landscape} 
\label{s:NonConvexLandscape} 
We now consider an example that provides numerical evidence suggesting    \eqref{mainproblem} is a non-convex optimization problem. Let $\mu = 1_T / |T| \in \P(\R^2)$, where $T \subseteq \R^2$ is the triangle with base length two, height one, and center of mass zero, where the base and height are chosen to be parallel with the coordinate axes: see black triangles on the left hand side of Figure \ref{f:Landscape}. In this case, \eqref{mainproblem}, with $k=3$, has an obvious unique global minimizer: $\Omega = T$. We seek to compute the values of $W_2(\mu, 1_\Omega/|\Omega|)$ for many different triangles $\Omega$ to investigate nonconvexity of the energy landscape.

In order to compute the 2-Wasserstein distance, we being by approximating $\mu$ and $1_{\Omega}/|\Omega|$ by $n~=~3,792$ Dirac masses, arranged on a uniform grid. We then apply the  \texttt{emd} function from the Python Optimal Transport library \cite{flamary2021pot} to compute the 2-Wasserstein distance between the approximations. Our approximation by Dirac masses introduces numerical error on the order of $\sim 0.01$ in the distance computation.

On the left hand side of Figure  \ref{f:Landscape}, we compute  $W_2(\mu, 1_{\Omega}/|\Omega|)$ as   the triangle $\Omega$ varies according to two parameters: $p1$ controls the height (different values of $p1$ are shown in each row) and $p2$ controls the width of the base (different values of $p2$ are shown in each column). The triangles $\Omega$ are constrained to always have center of mass $0$, since the optimal choice of $\Omega$ will   have center of mass coinciding with   $\mu$. Each cell in the grid shows the varying triangle $\Omega$ in blue, the target triangle $T$ in black, and the value of the 2-Wasserstein distance between them. We see that the minimum distance is obtained when $\Omega = T$, shown in the second row, third column. Due to the error of our numerical approximation of the 2-Wasserstein distance, the distance between them is shown as $0.01$. 

The right hand side of Figure  \ref{f:Landscape} depicts the value of $W_2(\mu , 1_{\Omega}/|\Omega|)$, as the triangles $\Omega$ vary in the same manner, visualized as a contour plot. We observe the global minima when $\Omega = T$ in the top of the figure. The bottom of the figure shows a potential local minimum, though due to the accuracy limits of our numerical computation, it could also be a flat area of the energy landscape. Either way, it is clear from the contour plot that the energy landscape it nonconvex. For this reason, it is an important direction for future work to understand under what conditions our gradient-based minimization algorithm is guaranteed to converge to the global minimum, as well as to develop non-convex minimization methods for the general setting.

%

\begin{figure}
\includegraphics[width = 0.61\textwidth,clip, trim = .5cm 2.2cm 2.5cm .5cm]{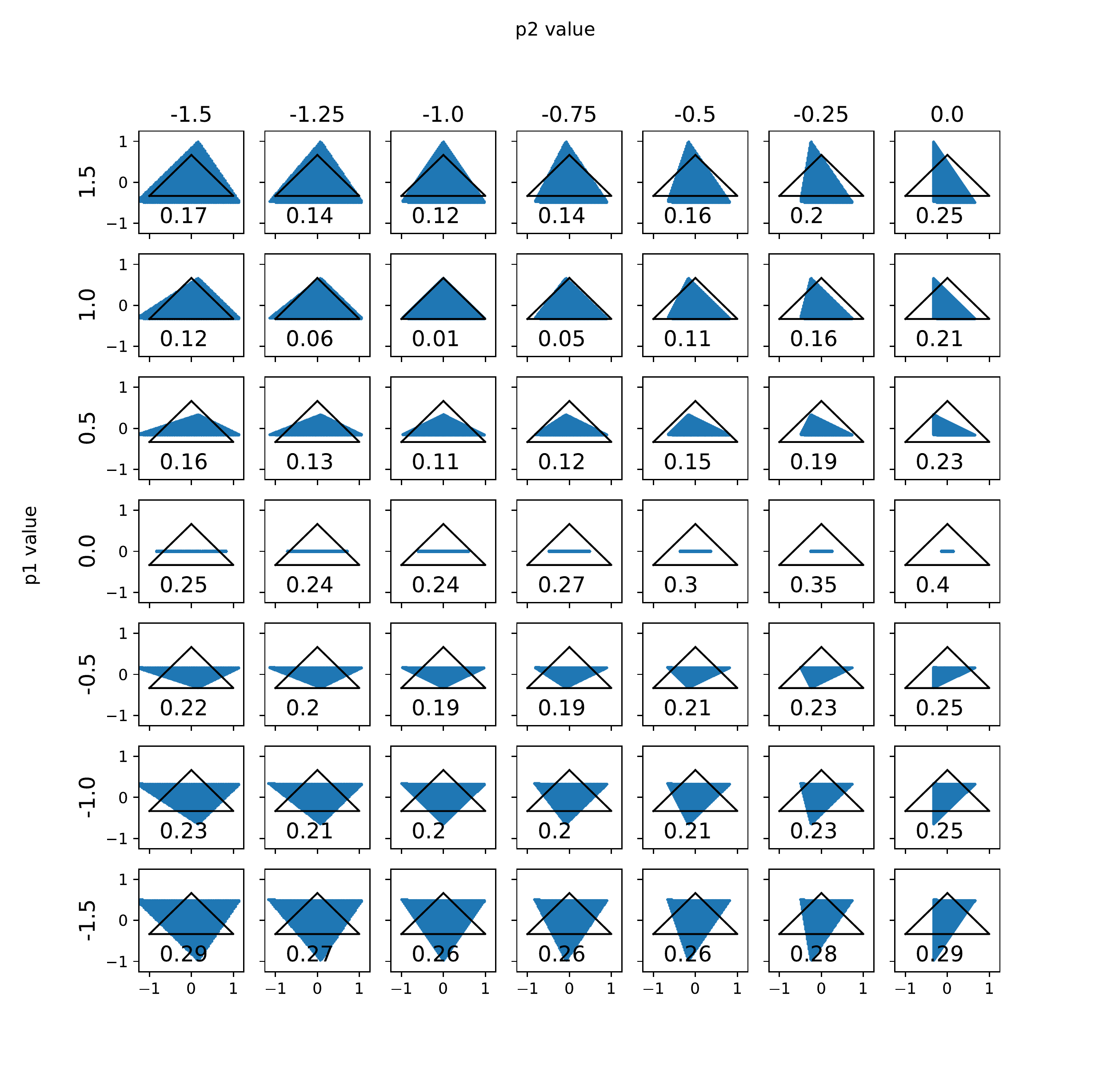}
\includegraphics[width = 0.35\textwidth,clip, trim = .4cm .4cm .4cm .4cm]{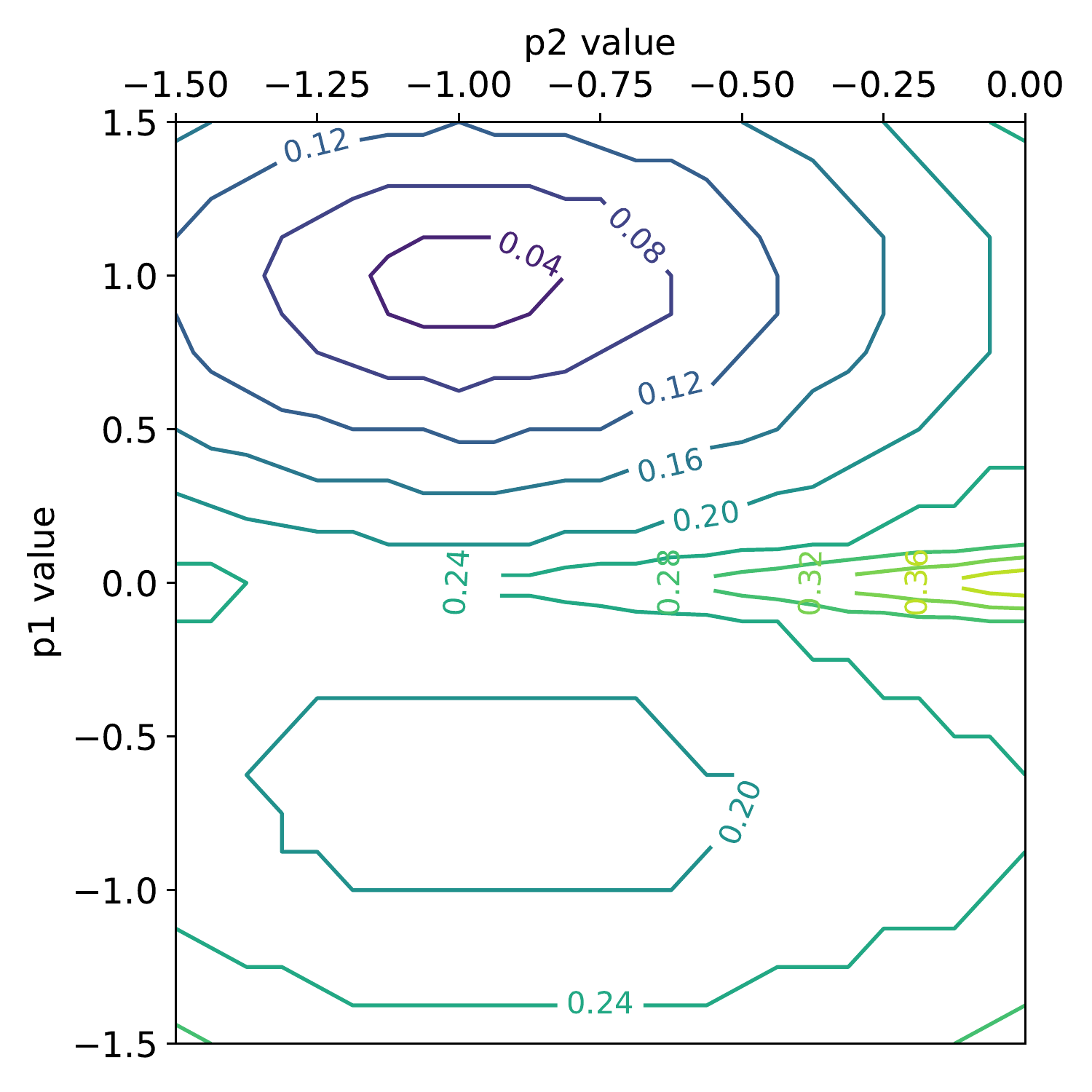}
\caption{We numerically investigate the energy landscape of \eqref{mainproblem}, when $k=3$ and $\mu = 1_T/|T|$, when $T \subseteq \R^2$ is given by the black triangle shown in each cell. On the left, the figure shows how $W_2(\mu, 1_\Omega/|\Omega|)$ changes as the form of the triangles $\Omega$ is varied according to parameters $p1$ (height) and $p2$ (width). The blue triangle in each cell shows the shape of $\Omega$ for a given choice of $(p1, p2)$, and the number in each cell shows the value of $W_2(\mu, 1_\Omega/|\Omega|)$. On the right, the figure shows contour plots for the value of $W_2(\mu, 1_\Omega/|\Omega|)$ as $\Omega$ varies, depicting a nonconvex energy landscape. } \label{f:Landscape}
\end{figure}

\subsection{Example 5: Early spread of the Covid-19 virus in the U.S} 
\label{s:COVID} 
In this section, we apply (\ref{mainproblem}) to explore the early-stage evolution of the Covid-19 virus in the U.S. 
The dataset used for analysis is freely available at \url{https://covidtracking.com/data/api}.
In total, there are $51$ data points (50 states + D.C.), each corresponding to a time series of the average positivity rates (PRs), where
\begin{align*}
\text{PR} = \frac{\text{total \# of positive cases by the day}}{\text{total \# of tests by the day}},
\end{align*}
and the average is computed using the centered $5$-day moving average scheme.   
The time range is chosen between April 20 and September 20, 2020, the relatively early stage of the   pandemic.
Visualization of the dataset can be found in Figure \ref{kuy}.

In this example, the dimension of the time series data is $154$, which is much higher than the number of data points. 
For convenience, we first use PCA to reduce the dimension of the data before applying WAA, with the explained variances by the first five principal components (PCs) plotted in Figure \ref{kuy}.
In this example, the first two PCs combined account for about $99\%$ variation of the dataset.
As a result, we use the first two PCs to obtain a reduced representation for the dataset.
A quantitative analysis of such an approximation procedure in the context of AA can be found in \cite{Xu2021}.
\begin{figure}
\includegraphics[width = 0.46\textwidth,clip, trim = 0cm 0cm 0cm 0cm]{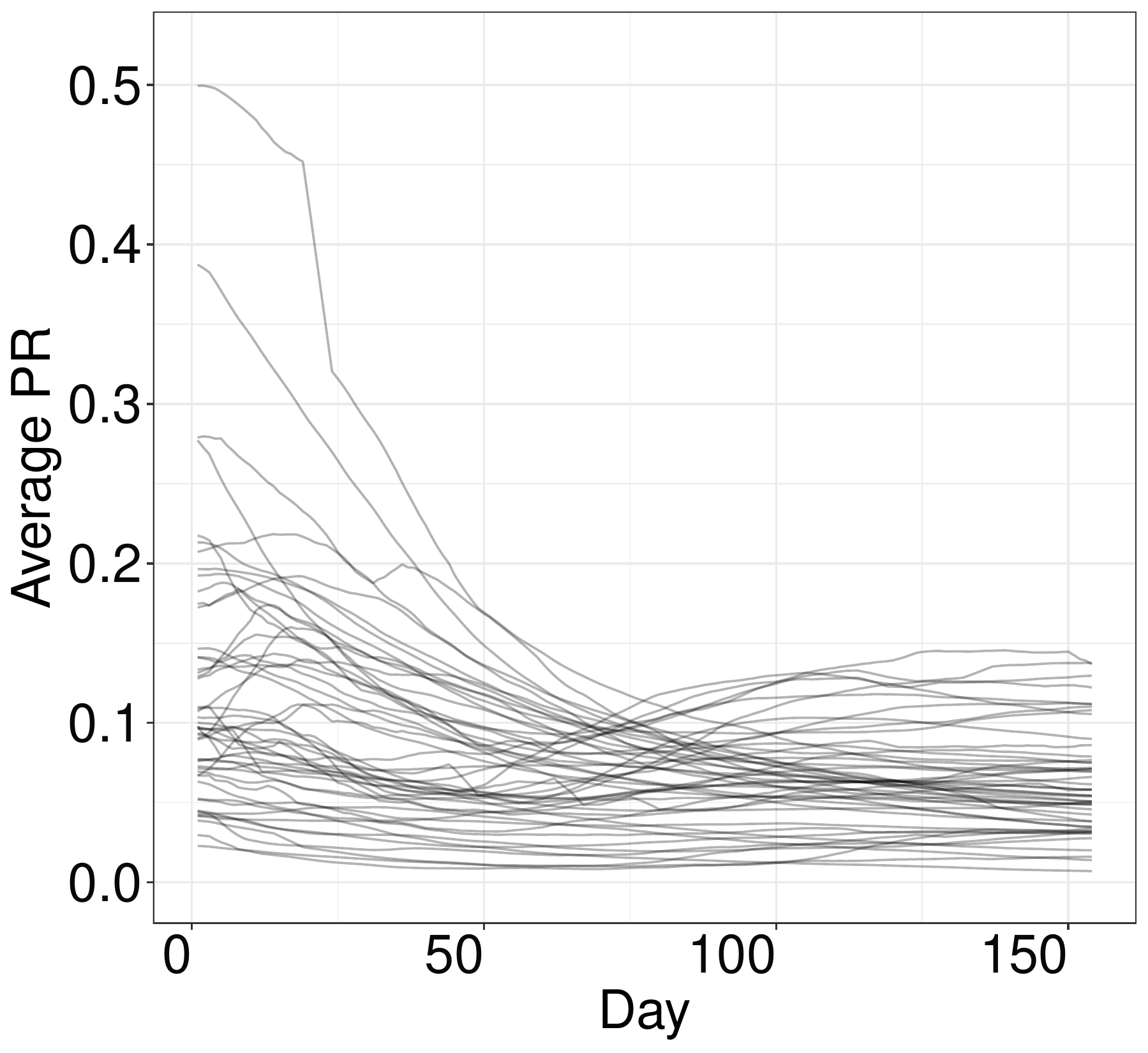}
\includegraphics[width = 0.46\textwidth,clip, trim = 0cm 0cm 0cm 0cm]{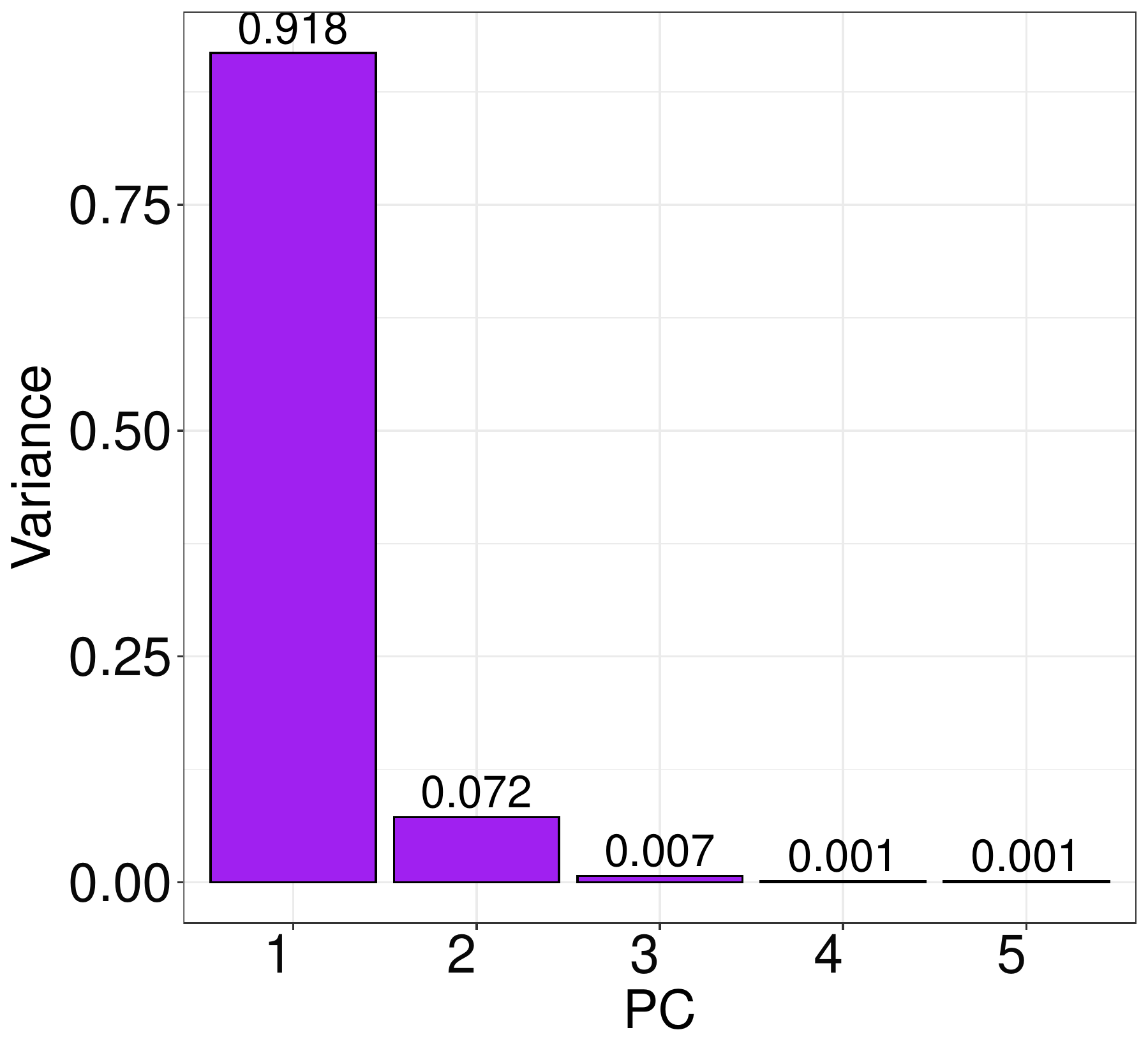}
\caption{Visualization of average PRs of the 50 states plus D.C. from April 20 to September 20, 2022 (left) and the variances explained by the first 5 PCs (right).}\label{kuy}
\end{figure}

In Figure  \ref{comp}, we apply WAA to the dataset, choosing $k = 3$, using the empirical elbow rule.  We also compare WAA to  both classical  AA and  $k$-means, applied to the same dataset.
It can be seen that both WAA and AA yield similar results, with the former demonstrating a more robust performance to the outliers. 
In both cases, the archetypes correspond to exemplars of the different evolutionary patterns. 
The upper right archetype is close to the northeastern states like New York (NY) and New Jersey (NJ), which is   where the first outbreak in the U.S. took place. 
The bottom archetype is close to many southern states, which corresponds to the second outbreak.
The upper left archetype is surrounded by states that have a low population density and experienced a relatively slow positive rate curve at the beginning stage of the pandemic. 
In contrast, the $k$-means centers are more difficult to interpret.  
\begin{figure}
\includegraphics[width = 0.46\textwidth,clip, trim = 0cm 0cm 0cm 0cm]{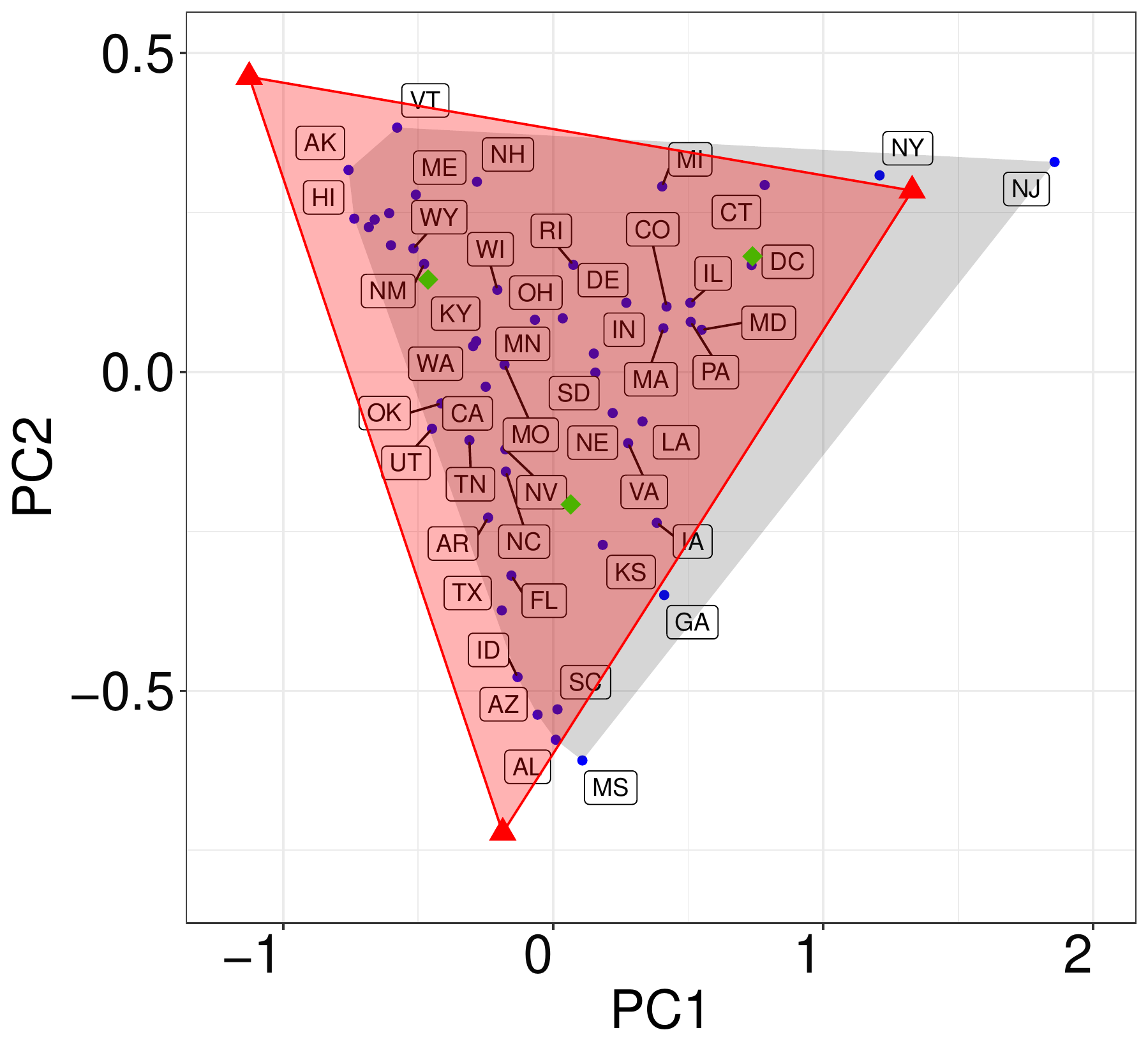}
\includegraphics[width = 0.46\textwidth,clip, trim = 0cm 0cm 0cm 0cm]{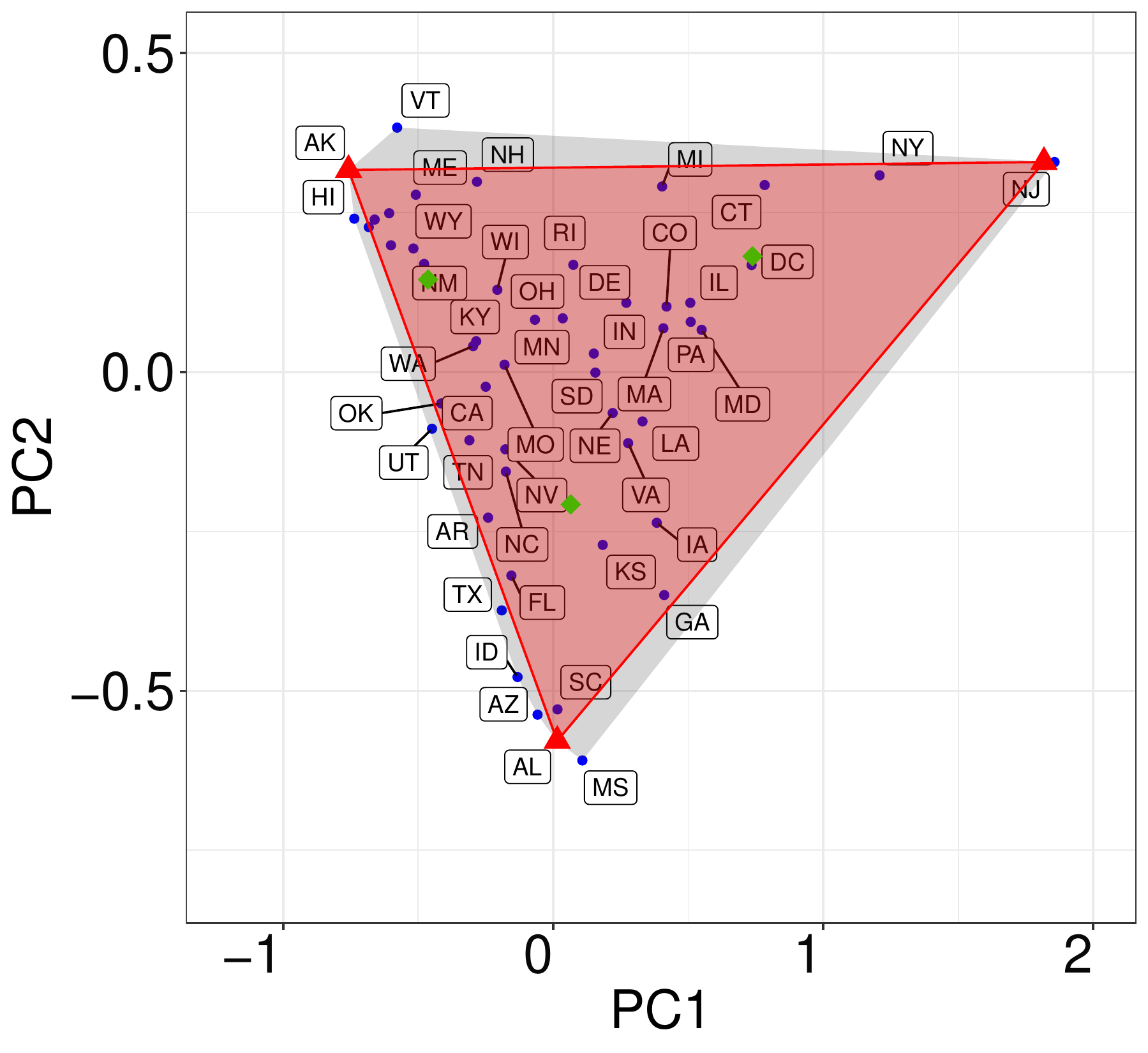}
\caption{Three archetypes (red triangles) given by WAA (left) and the original AA (right) applied to the reduced dataset. Green squares correspond to the $k$-means centers. }\label{comp}
\end{figure}

\section{Discussion}
\label{s:Disc}

In this paper, we considered a Wasserstein metric-based Archetypal Analysis, where a 
probability measure $\mu \in \P(\Rd)$ is approximated by a distribution that is uniformly supported on a polytope with a fixed number of vertices. 
We established that, in one dimension, there is a unique minimizer  of (\ref{mainproblem}), and in two dimensions, we showed that a minimizer exists, under the additional assumption that $\mu$ is absolutely continuous with respect to Lebesgue measure (Theorem \ref{t:Existence2d}). By adding an appropriate regularization in terms a R\'enyi entropy, we were able to prove that a solution of the regularized problem (\ref{mainprobv2}) exists for all $\varepsilon >0$,   $d \geq 1$ and $\mu \in \P_2(\Rd)$ (Theorem~\ref{t:Existence}). Finally, we  proved a consistency and convergence result for the (\ref{mainprobv2}) problem (Theorem~\ref{t:Consistency}).  
In Section~\ref{s:CompApproach}, we introduced a computational approach for the (\ref{mainproblem}) and (\ref{mainprobv2}) problems using the semi-discrete formulation of the Wasserstein metric; see Algorithm~\ref{alg1}. 
We concluded by implementing the algorithm and conducting several numerical experiments in Section~\ref{s:NumExp} to support our analytical findings.

There are many interesting future directions for this work.
\begin{itemize}
\item Extend the analysis for $\varepsilon =0$ to higher dimensions (see remark~\ref{difficulty}) and $\mu$ with lower dimensional support (see remark \ref{muabscts}).
\item Consider formulations of the archetypal analysis problem for more general optimal transport metrics, e.g. $p$-Wasserstein, $p \neq 2$.
\item Analyze uniqueness of solutions up to global invariances.
\item For $\mu$ the uniform distribution on the unit disk (Section~\ref{sec:uniform}) and a standard normal distribution (Section~\ref{sec:normal}), establish that the optimal solution is given by a regular $k$-gon. 
\item Develop computational methods for higher dimensional problems. Here one could use the back and forth method \cite{jacobs2020fast} or the entropic approximations of the 2-Wasserstein distance \cite{cuturi2013sinkhorn}.
\item Further study the differences between the original archetypal analysis problem and WAA. 
\item Analysis of sufficient conditions on the data distribution and the initialization of  our numerical method to ensure   convergence to an optimizer. 
\end{itemize}

\subsection*{Acknowledgements} 
K. Craig would like to thank Jun Kitagawa and Dejan Slep{\v{c}}ev for helpful discussions on semi-discrete optimal transport and the challenges of WAA in higher dimensions. 

\bibliographystyle{abbrv}
\bibliography{KatysBib_013119,refs}
\end{document}